\documentclass{article} 

\usepackage{xcolor}  
\usepackage[breaklinks=true,colorlinks]{hyperref}
\hypersetup{
  linkcolor={red!50!black},
  citecolor={blue!50!black},
}   
\usepackage{url}

\usepackage[utf8]{inputenc} 
\usepackage[truedimen,margin=30mm]{geometry} 

\usepackage[T1]{fontenc}    
\usepackage{url}            
\usepackage{booktabs}       
\usepackage{amsfonts}       
\usepackage{nicefrac}       
\usepackage{microtype}      

\usepackage[pdftex]{graphicx}
\usepackage{natbib} 
\usepackage{algorithm} 
\usepackage{algorithmic} 
\usepackage{multirow}
\usepackage{amsmath}
\usepackage{amssymb}
\usepackage{amsthm}
\usepackage{here}
\usepackage{wrapfig}
\usepackage{centernot}
\usepackage{enumitem}
\usepackage{comment}
\usepackage{subfig}
\usepackage{bbm}

\allowdisplaybreaks[1]

\newtheorem{theorem}{Theorem}
\newtheorem{lemma}{Lemma}
\newtheorem{proposition}{Proposition}
\newtheorem{definition}{Definition}
\newtheorem{corollary}{Corollary}
\newtheorem{assumption}{Assumption}
\newtheorem{example}{Example}

\usepackage{color}
\usepackage{CJKutf8}
\usepackage{ifthen}
\newcommand{\COMM}[2]{{
\begin{CJK}{UTF8}{ipxm}
\ifthenelse{\equal{#1}{AN}}{\color{red}}{
\ifthenelse{\equal{#1}{DW}}{\color{blue}}{
\ifthenelse{\equal{#1}{TS}}{\color{green}}}}
[#1: #2]
\end{CJK}
}}

\newcommand{\defeq}{\overset{\mathrm{def}}{=}}

\def\pd<#1>{\left\langle #1 \right\rangle}
\def\floor[#1]{\left\lfloor #1 \right\rfloor}
\def\ceil[#1]{\left\lceil #1 \right\rceil}

\newcommand{\rd}{\mathrm{d}}

\newcommand{\bone}{\mathbbm{1}}
\newcommand{\bE}{\mathbb{E}}

\newcommand{\bP}{\mathbb{P}}
\newcommand{\bR}{\mathbb{R}}

\newcommand{\cC}{\mathcal{C}}
\newcommand{\cD}{\mathcal{D}}
\newcommand{\cF}{\mathcal{F}}

\newcommand{\cL}{\mathcal{L}}
\newcommand{\cN}{\mathcal{N}}
\newcommand{\cP}{\mathcal{P}}
\newcommand{\cX}{\mathcal{X}}
\newcommand{\cY}{\mathcal{Y}}
\newcommand{\cZ}{\mathcal{Z}}

\newcommand{\KL}{\mathrm{KL}}
\newcommand{\vx}{\mathbf{x}}
\newcommand{\rademacher}{\hat{\Re}}

\title{Particle Dual Averaging: Optimization of Mean Field\\ 
\!\!\!Neural Network with Global Convergence Rate Analysis\!\!\!}

\author{Atsushi Nitanda$^{1\dag}$, Denny Wu$^{2\ddag}$, Taiji Suzuki$^{3\star}$
\vspace{2mm}\\
\normalsize{\textit{$^1$Kyushu Institute of Technology and RIKEN Center for Advanced Intelligence Project}} \\
\normalsize{\textit{$^2$University of Toronto and Vector Institute for Artificial Intelligence}} \\
\normalsize{\textit{$^3$The University of Tokyo and RIKEN Center for Advanced Intelligence Project}} \\
\small{Email: $^\dag$nitanda@ai.kyutech.ac.jp, $^\ddag$dennywu@cs.toronto.edu, $^\star$taiji@mist.i.u-tokyo.ac.jp}} 

\date{}

\begin{document}

\maketitle

\begin{abstract}
We propose the {\it particle dual averaging} (PDA) method, which generalizes the dual averaging method in convex optimization to the optimization over probability distributions with quantitative runtime guarantee. The algorithm consists of an inner loop and outer loop: the inner loop utilizes the Langevin algorithm to approximately solve for a stationary distribution, which is then optimized in the outer loop. The method can thus be interpreted as an extension of the Langevin algorithm to naturally handle \textit{nonlinear} functional on the probability space. An important application of the proposed method is the optimization of neural network in the \textit{mean field} regime, which is theoretically attractive due to the presence of nonlinear feature learning, but quantitative convergence rate can be challenging to obtain. By adapting finite-dimensional convex optimization theory into the space of measures, we analyze PDA in regularized empirical / expected risk minimization, and establish \textit{quantitative} global convergence in learning two-layer mean field neural networks under more general settings. Our theoretical results are supported by numerical simulations on neural networks with reasonable size. 
\end{abstract}

\section{Introduction}
Gradient-based optimization can achieve vanishing training error on neural networks, despite the apparent non-convex landscape. Among various works that explains the global convergence, one common ingredient is to utilize overparameterization to translate the training dynamics into function spaces, and then exploit the convexity of the loss function with respect to the function. Such endeavors usually consider models in one of the two categories: the \textit{mean field} regime or the \textit{kernel} regime. 

On one hand, analysis in the kernel (lazy) regime
connects gradient descent on wide neural network to kernel regression with respect to the neural tangent kernel \citep{jacot2018neural}, which leads to global convergence at linear rate \citep{du2018gradient,allen2019convergence,zou2020gradient}. However, key to the analysis is the \textit{linearization} of the training dynamics, which requires appropriate scaling of the model such that distance traveled by the parameters vanishes \citep{chizat2018note}. Such regime thus fails to explain the \textit{feature learning} of neural networks \citep{yang2020feature}, which is believed to be an important advantage of deep learning; indeed, it has been shown that deep learning can outperform kernel models due to this adaptivity \citep{suzuki2018adaptivity,ghorbani2019limitations}. 

In contrast, the mean field regime describes the gradient descent dynamics as Wasserstein gradient flow in the probability space \citep{nitanda2017stochastic,mei2018mean,chizat2018global}, which captures the potentially \textit{nonlinear} evolution of parameters travelling beyond the kernel regime. While the mean field limit is appealing due to the presence of ``feature learning'', its characterization is more challenging and quantitative analysis is largely lacking. Recent works established convergence rate in continuous time under modified dynamics \citep{rotskoff2019global}, strong assumptions on the target function \citep{javanmard2019analysis}, or regularized objective \citep{hu2019mean}, but such result can be fragile in the discrete-time or finite-particle setting --- in fact, the discretization error often scales exponentially with the time horizon or dimensionality, which limits the applicability of the theory.
Hence, an important research problem that we aim to address is

\vspace{-1mm}
\begin{center}
{\it Can we develop optimization algorithms for neural networks in the mean field regime with more accurate quantitative guarantees the kernel regime enjoys?}
\end{center}
\vspace{-1mm}

We address this question by introducing the {\it particle dual averaging} (PDA) method, which globally optimizes an entropic regularized nonlinear functional. For two-layer mean field network which is an important application, we establish polynomial runtime guarantee for the \textit{discrete-time} algorithm; to our knowledge this is the first quantitative global convergence result under similar settings.
 
\subsection{Contributions}
We propose the PDA algorithm, which draws inspiration from the dual averaging method originally developed for finite-dimensional convex optimization \citep{nesterov2005smooth,nesterov2009primal,xiao2009dual}.
We iteratively optimize a probability distribution in the form of a Boltzmann distribution, samples from which can be obtained from the Langevin algorithm (see Figure~\ref{fig:PDA_illustration}). 
The resulting algorithm has comparable per-iteration cost as gradient descent and can be efficiently implemented. 

For optimizing two-layer neural network in the mean-field regime, we establish quantitative global convergence rate of PDA in minimizing an KL-regularized objective: the algorithm requires $\tilde{O}(\epsilon^{-3})$ steps and $\tilde{O}(\epsilon^{-2})$ particles to reach an $\epsilon$-accurate solution, where $\tilde{O}$ hides logarithmic factors. 
Importantly, our analysis does not couple the learning dynamics with certain continuous time limit, but directly handles the discrete update. This leads to a simpler analysis that covers more general settings. We also derive the generalization bound on the solution obtained by the algorithm.  
From the viewpoint of the optimization, PDA is an extension of Langevin algorithm to handle entropic-regularized nonlinear functionals on the probability space. 
Hence we believe our proposed method can also be applied to other distribution optimization problems beyond the training of neural networks.  

\begin{figure}[t]  
\vspace{-4mm} 
\centering
\begin{minipage}[t]{0.6\linewidth} 
{\includegraphics[width=0.94\linewidth]{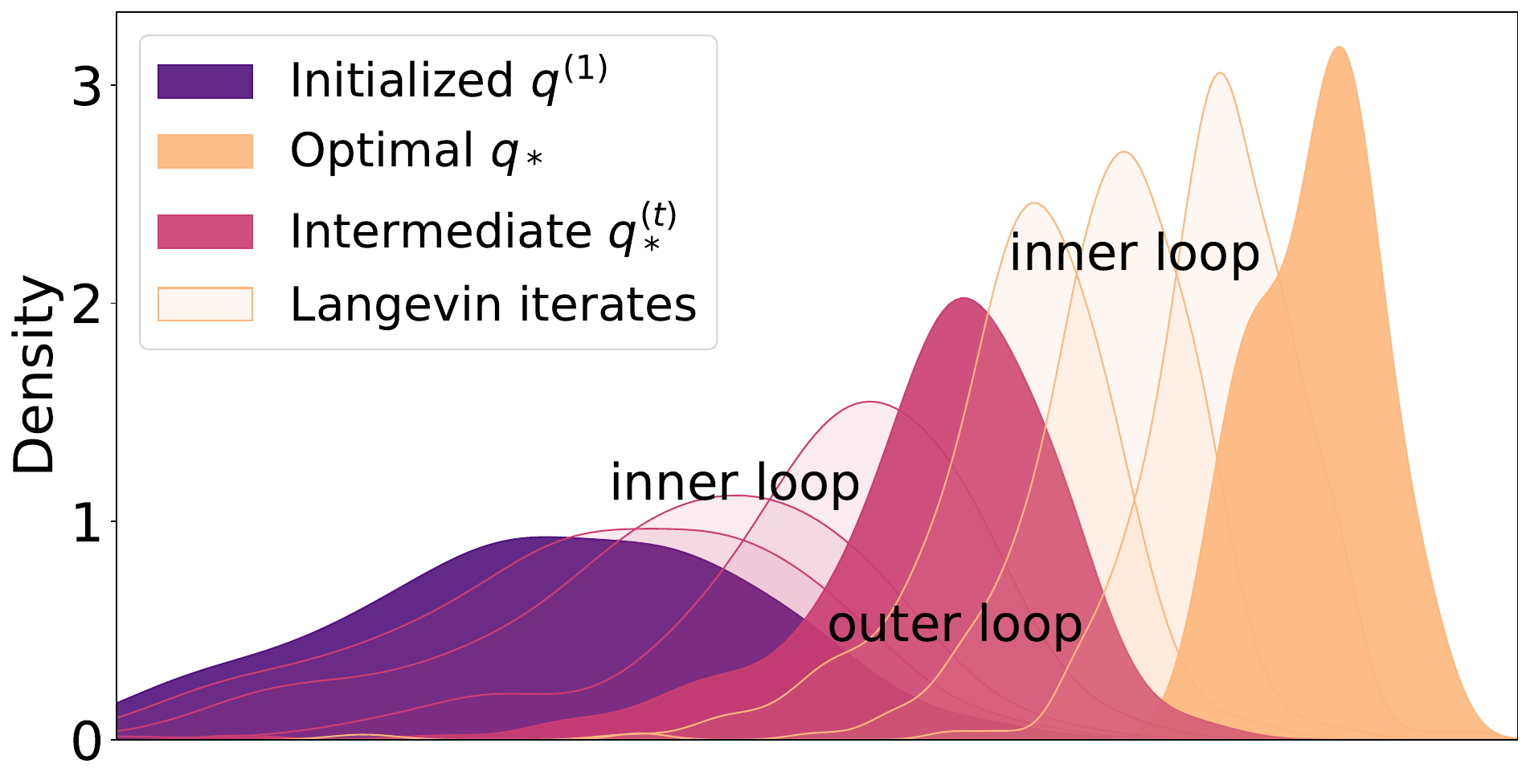}} 
\end{minipage}
\begin{minipage}[b]{0.38\linewidth} 
\captionof{figure}{\small 1D visualization of parameter distribution of mean field two-layer neural network (tanh) optimized by PDA. The \textit{inner loop} uses the Langevin algorithm to solve an approximate stationary distribution $q_*^{(t)}$, which is then optimized in the \textit{outer loop} towards the true target $q_*$.} 
\label{fig:PDA_illustration}
\end{minipage}
\vspace{-3mm}
\end{figure} 

\subsection{Related Literature}
\paragraph{Mean field limit of two-layer NNs.}
The key observation for the mean field analysis is that when the number of neurons becomes large, the evolution of parameters is well-described by a nonlinear partial differential equation (PDE), which can be viewed as solving an infinite-dimensional \textit{convex} problem \citep{bengio2006convex,bach2017breaking}. Global convergence can be derived by studying the limiting PDE \citep{mei2018mean,chizat2018global,rotskoff2018trainability,sirignano2020mean}, yet quantitative convergence rate generally requires additional assumptions.    

\citet{javanmard2019analysis} analyzed a particular RBF network and established linear convergence (up to certain error\footnote{Note that such error yields sublinear rate with respect to arbitrarily small accuracy $\epsilon$.\vspace{-2mm}}) for strongly concave target functions.  \citet{rotskoff2019global} provided a sublinear rate in continuous time for a modified gradient flow.
In the regularized setting, \citet{chizat2019sparse} obtained local linear convergence under certain non-degeneracy assumption on the objective. 
\citet{wei2019regularization} also proved polynomial rate for a perturbed dynamics under weak $\ell_2$ regularization.

Our setting is most related to \citet{hu2019mean}, who studied the minimization of a nonlinear functional with KL regularization on the probability space, and showed linear convergence (in continuous time) of a particle dynamics named {\it mean field Langevin dynamics} when the regularization is sufficiently strong. 
\citet{chen2020generalized} also considered optimizing a KL-regularized objective in the infinite-width and continuous-time limit, and derived NTK-like convergence guarantee under certain parameter scaling.  
Compared to these prior works, we directly handle the discrete time update in the mean-field regime, and our analysis covers a wider range of regularization parameters and loss functions.  

\vspace{-1.5mm}
\paragraph{Langevin algorithm.}
Langevin dynamics can be viewed as optimization in the space of probability measures \citep{jordan199618,jordan1998variational}; this perspective has been explored in \citet{wibisono2018sampling,durmus2019analysis}. 
It is known that the continuous-time Langevin diffusion converges exponentially fast to target distributions satisfying certain growth conditions \citep{roberts1996exponential,mattingly2002ergodicity}. 
The discretized \textit{Langevin algorithm} has a sublinear convergence rate that depends on the numerical scheme \citep{li2019stochastic} and has been studied under various metrics \citep{dalalyan2014theoretical,durmus2017nonasymptotic,cheng2017convergence}. 

The Langevin algorithm can also optimize certain non-convex objectives \citep{raginsky2017non,xu2018global,erdogdu2018global}, in which one finite-dimensional ``particle'' can attain approximate global convergence due to concentration of Boltzmann distribution around the true minimizer. 
However, such result often depends on the spectral gap that grows exponentially in dimensionality, which renders the analysis ineffective for neural net optimization in the high-dimensional \textit{parameter space}.     

Very recently, convergence of Hamiltonian Monte Carlo in learning certain mean field models has been analyzed in \citet{bou2020convergence,bou2021mixing}. 
Compared to these concurrent results, our formulation covers a more general class of potentials, and in the context of two-layer neural network, we provide optimization guarantees for a wider range of loss functions. 

\vspace{-1mm}
\subsection{Notations}
\vspace{-1mm}
Let $\bR_{+}$ denote the set of non-negative real numbers and $\|\cdot\|_2$ the Euclidean norm.
Given a density function $q: \bR^p \rightarrow \bR_{+}$, we denote the expectation with respect to $q(\theta) \rd \theta$ by $\bE_q[\cdot]$.
For a function $f: \bR^p \rightarrow \bR$, we define $\bE_q[f] = \int f(\theta) q(\theta)\rd\theta$ when $f$ is integrable.
$\KL$ is the Kullback-Leibler divergence: $\KL(q\|q^{\prime}) \defeq \int q(\theta) \log \left( \frac{q(\theta)}{q^{\prime}(\theta)}\right)\rd\theta$.
Let $\cP_2$ denote the set of positive densities $q$ on $\bR^p$ such that the second-order moment $\bE_q[\|\theta\|_2^2] < \infty$ and entropy $-\infty < -\bE_q[\log(q)] < +\infty$ are well defined.
$\cN(0,I_p)$ is the Gaussian distribution on $\bR^p$ with mean $0$ and covariance matrix $I_p$.

\section{Problem Setting}\label{sec:problem}
We consider the problem of risk minimization with neural networks in the mean field regime. 
For simplicity, we focus on supervised learning. 
We here formalize the problem setting and models. 
Let $\cX \subset \bR^d$ and $\cY \subset \bR$ be the input and output spaces, respectively.
For given input data $x \in \cX$, we predict a corresponding output $y=h(x) \in \cY$ through a hypothesis function $h: \cX \rightarrow \cY$. 

\subsection{Neural Network and Mean Field Limit}\label{subsec:nn_mf}
We adopt a neural network in the mean field regime as a hypothesis function.
Let $\Omega = \bR^p$ be a parameter space and $h_\theta: \cX \rightarrow \cY$ $(\theta \in \Omega)$ be a bounded function which will be a component of a neural network.
We sometimes denote $h(\theta,x) = h_\theta(x)$.
Let $q(\theta)\rd\theta$ be a probability distribution on the parameter space $\Omega$ and $\Theta =\{\theta_r\}_{r=1}^M$ be the set of parameters $\theta_r$ sampled from $q(\theta)\rd\theta$.
A hypothesis is defined as an ensemble of $h_{\theta_r}$ as follows:
\vspace{-0.5mm}
\begin{equation}\label{eq:nn} 
h_\Theta(x) \defeq \frac{1}{M}\sum_{r=1}^M h_{\theta_r}(x). 
\end{equation}
\vspace{-3mm}

A typical example in the literature of the above formulation is a two-layer neural network.
\begin{example}[Two-layer Network]\label{eg:2nn}
Let $a_r \in \bR$ and $b_r \in \bR^d$ $(r \in \{1,2,\ldots,M\})$ be parameters for output and input layers, respectively.
We set $\theta_r = (a_r, b_r)$ and $\Theta = \{\theta_r\}_{r=1}^M$.
Denote $h_{\theta_r}(x) \defeq \sigma_2( a_r \sigma_1(b_r^\top x))$ $(x \in \cX)$, 
where $\sigma_1$ and $\sigma_2$ are smooth activation functions.
Then the hypothesis $h_{\Theta}$ is a two-layer neural network composed of neurons $h_{\theta_r}$: $h_{\Theta}(x) = \frac{1}{M} \sum_{r=1}^M \sigma_2( a_r \sigma_1 (b_r^\top x)).$
\end{example}
\paragraph{Remark.} The purpose of $\sigma_2$ in the last layer is to ensure the boundedness of output (e.g., see Assumption 2 in \cite{mei2018mean}); this nonlinearity can also be removed if parameters of output layer are fixed. 
In addition, although we mainly focus on the optimization of two-layer neural network, our proposed method can also be applied to ensemble $h_\Theta$ of deep neural networks $h_{\theta_r}$.  

Suppose the parameters $\theta_r$ follow a probability distribution $q(\theta) \rd\theta$, then $h_\Theta$ can be viewed as a finite-particle discretization of the following expectation, 
\begin{equation}
\label{eq:mfl}
h_q(x) = \bE_{q}[h_{\theta}(x)].
\end{equation}
which we refer to as the \textit{mean field limit} of the neural network $h_\Theta$.
As previously discussed, when $h_\Theta$ is overparameterized,
optimizing $h_\Theta$ becomes ``close'' to directly optimizing the probability distribution on the parameter space $\Omega$, for which convergence to the optimal solution may be established under appropriate conditions \citep{nitanda2017stochastic,mei2018mean,chizat2018global}.  
Hence, the study of optimization of $h_q$ with respect to the probability distribution $q(\theta)\rd \theta$ may shed light on important properties of overparameterized neural networks.   

\subsection{Regularized Empirical Risk Minimization}\label{subsec:rerm}

We briefly outline our setting for regularized expected / empirical risk minimization.
The prediction error of a hypothesis is measured by the loss function $\ell(z,y)$ $(z,y \in \cY)$, such as the squared loss $\ell(z,y) = 0.5(z-y)^2$ for regression, or the logistic loss $\ell(z,y) = \log(1+\exp(-yz))$ for binary classification.
Let $\cD$ be a data distribution over $\cX \times \cY$. 
For expected risk minimization, the distribution $\cD$ is set to the true data distribution; whereas for empirical risk minimization, we take $\cD$ to be the empirical distribution defined by training data $\{(x_i,y_i)\}_{i=1}^n$ $(x_i \in \cX, y_i \in \cY)$ independently sampled from the data distribution. 
We aim to minimize the expected / empirical risk together with a regularization term, which controls the model complexity and also stabilizes the optimization.  
The regularized objective can be written as follows: for $\lambda_1, \lambda_2 > 0$,
\begin{equation}\label{eq:rerm}
\min_{q \in \cP_2} \left\{ \cL(q) \defeq \bE_{(X,Y)\sim \cD}[ \ell(h_q(X),Y)] + R_{\lambda_1,\lambda_2}(q) \right\},
\end{equation}
where $R_{\lambda_1, \lambda_2}$ is a regularization term composed of the weighted sum of the second-order moment and negative entropy with regularization parameters $\lambda_1$, $\lambda_2$:
\begin{equation}\label{eq:regularization}
R_{\lambda_1,\lambda_2}(q) \defeq \lambda_1 \bE_q[\|\theta\|_2^2] + \lambda_2 \bE_q[\log(q)].
\end{equation}
Note that this regularization is the KL divergence of $q$ from a Gaussian distribution. In our setting, such regularization ensures that the Gibbs distributions $q_*^{(t)}$ specified in Section~\ref{sec:particle-dual-averaging} are well defined. 

While our primary focus is the optimization of the objective \eqref{eq:rerm}, we can also derive a generalization error bound for the empirical risk minimizer of order of $O(n^{-1/2})$ for both the regression and binary classification settings, following \citet{chen2020generalized}. 
We defer the details to Appendix \ref{app:generalization_bounds}.

\subsection{The Langevin Algorithm} \label{subsec:langevin}
Before presenting our proposed method, we briefly review the Langevin algorithm.
For a given smooth potential function $f: \Omega \rightarrow \bR$, the Langevin algorithm performs the following update: given the initial $\theta^{(1)} \sim q^{(1)}(\theta)\rd\theta$, step size $\eta>0$, and Gaussian noise $\zeta^{(k)} \sim \cN(0,I_p)$,
\begin{equation}
\theta^{(k+1)} \leftarrow \theta^{(k)} - \eta \nabla_\theta f(\theta^{(k)}) + \sqrt{2\eta}\zeta^{(k)}. \label{eq:ula}    
\end{equation} 
Under appropriate conditions on $f$, it is known that $\theta^{(t)}$ converges to a stationary distribution proportional to $\exp(-f(\cdot))$ in terms of KL divergence at a linear rate (e.g., \cite{vempala2019rapid}) up to $O(\eta)$-error,  
where we hide additional factors in the big-$O$ notation.

Alternatively, note that when the normalization constant $\int\exp(-f(\theta))\rd\theta$ exists, the Boltzmann distribution in proportion to $\exp( -f(\cdot) )$ is the solution of the following optimization problem,
\begin{equation}\label{eq:linear-functional}
 \min_{q: \mathrm{density}} \left\{ \bE_q[f] + \bE_q[\log(q)] \right\}. 
\end{equation}
Hence we may interpret the Langevin algorithm as approximately solving an entropic regularized linear functional (i.e., free energy functional) on the probability space.
This connection between sampling and optimization (see \citet{dalalyan2017further,wibisono2018sampling,durmus2019analysis}) enables us to employ the Langevin algorithm to obtain (samples from) the closed-form Boltzmann distribution which is the minimizer of  \eqref{eq:linear-functional}; for example, many Bayesian inference problems fall into this category.  

However, the objective \eqref{eq:rerm} that we aim to optimize is beyond the scope of Langevin algorithm -- due to the \textit{nonlinearity} of loss $\ell(z,y)$ with respect to $z$, the stationary distribution cannot be described as a closed-form solution of \eqref{eq:linear-functional}. 
To overcome this limitation, we develop the particle dual averaging (PDA) algorithm which efficiently solves \eqref{eq:rerm} with quantitative runtime guarantees.

\section{Proposed Method}\label{sec:particle-dual-averaging}
We now propose the \textit{particle dual averaging} method to approximately solve the problem (\ref{eq:rerm}) by optimizing a two-layer neural network in the mean field regime; we also introduce the mean field limit of the proposed method to explain the algorithmic intuition and develop the convergence analysis.  

\subsection{Particle Dual Averaging}
Our proposed particle dual averaging method (Algorithm \ref{alg:pda}) is an optimization algorithm on the space of probability measures. The algorithm
consists of an inner loop and outer loop; we run Langevin algorithm in inner loop to approximate a Gibbs distribution, which is optimized in the outer loop so that it converges to the optimal distribution $q_*$. This outer loop update is designed to extend the classical dual averaging scheme \citep{nesterov2005smooth,nesterov2009primal,xiao2009dual} to infinite dimensional optimization problems (described in Section \ref{subsec:mf-pda}). Below we provide a more detailed explanation.  
\begin{itemize}[topsep=0mm,leftmargin=*,itemsep=0.5mm] 
    \item In the outer loop, the last iterate $\tilde{\Theta}^{(t)}$ of the previous inner loop is given. We compute $\partial_z \ell( h_{\tilde{\Theta}^{(t)}}(x_{t}),y_{t})$, which is a component of the Gibbs potential\footnote{In Algorithm \ref{alg:pda}, the terms $\partial_z \ell( h_{\tilde{\Theta}^{(s)}}(x_{s}),y_{s})$  appear in inner loop; but note that these terms only need to be computed in outer loop because they are independent to the inner loop iterates.}, and initialize a set of particles $\Theta^{(1)}$ at $\tilde{\Theta}^{(t)}$. 
    In Appendix~\ref{sec:proof} we introduce a different ``restarting'' scheme for the initialization.
    \item In the inner loop, we run the Langevin algorithm (noisy gradient descent) starting from $\Theta^{(1)}$, where the gradient at the $k$-th inner step is given by $\nabla_\theta \overline{g}^{(t)}(\theta_r^{(k)})$, which is a sum of weighted average of $\partial_z\ell( h_{\tilde{\Theta}^{(s)}}(x_{s}),y_{s}) \partial_\theta h(\theta_r^{(k)},x_{s})$ and the gradient of $\ell_2$-regularization (see Algorithm \ref{alg:pda}). 
\end{itemize}
\vspace{-1mm}

\begin{algorithm}[htb!]
  \caption{Particle Dual Averaging (PDA)}
  \label{alg:pda}
\begin{algorithmic}
  \STATE {\bfseries Input:}
  data distribution $\cD$,
  initial density $q^{(1)}$,
  number of outer-iterations $T$,
  learning rates $\{\eta_t\}_{t=1}^T$,
  number of inner-iterations $\{T_t\}_{t=1}^T$
\vspace{1mm}
  \STATE Randomly draw i.i.d.~initial parameters $\tilde{\theta}_r^{(1)} \sim q^{(1)}(\theta)\rd\theta$ $(r\in\{1,2,\ldots,M\})$
  \STATE $\tilde{\Theta}^{(1)} \leftarrow \{\tilde{\theta}_r^{(1)}\}_{r=1}^M$
\vspace{1mm}
   \FOR{$t=1$ {\bfseries to} $T$}
   \STATE Randomly draw data $(x_t,y_t)$ from $\cD$ \\
   $\Theta^{(1)}=\{\theta_r^{(1)}\}_{r=1}^M \leftarrow \tilde{\Theta}^{(t)}$
\vspace{1mm}
     \FOR{$k=1$ {\bfseries to} $T_t$}
       \STATE Run inexact noisy gradient descent for $r\in\{1,2,\ldots,M\}$\\
       $\nabla_\theta \overline{g}^{(t)}(\theta_r^{(k)}) \leftarrow \frac{2}{\lambda_2(t+2)(t+1)}\sum_{s=1}^t s \partial_z\ell( h_{\tilde{\Theta}^{(s)}}(x_{s}),y_{s}) \partial_\theta h(\theta_r^{(k)},x_{s}) + \frac{2\lambda_1 t}{\lambda_2(t+2)} \theta_r^{(k)}$\\
       $\theta^{(k+1)}_r \leftarrow \theta^{(k)}_r - \eta_t \nabla_\theta \overline{g}^{(t)}(\theta^{(k)}_r) + \sqrt{2\eta_t}\zeta_r^{(k)}$~(i.i.d.~Gaussian noise~$\zeta_r^{(k)} \sim \cN(0,I_p)$)
     \ENDFOR
\vspace{1mm}
   \STATE $\tilde{\Theta}^{(t+1)} \leftarrow \Theta^{(T_t+1)} = \{\theta_r^{(T_t+1)}\}_{r=1}^M$
   \ENDFOR
\vspace{1mm}   
\STATE Randomly pick up $t \in \{2,3,\ldots,T+1\}$ following the probability $\bP[t]=\frac{2t}{T(T+3)}$ and return $h_{\tilde{\Theta}^{(t)}}$
\end{algorithmic}
\end{algorithm}
\vspace{-1mm}

Figure~\ref{fig:PDA_illustration} provides a pictorial illustration of Algorithm~\ref{alg:pda}. Note that this procedure is a slight modification of the normal gradient descent algorithm: the first term of $\nabla_\theta \overline{g}^{(t)}$ is similar to the gradient of the loss $\partial_{\theta_r} \ell( h_{\Theta^{(k)}}(x), y) \sim \partial_z \ell( h_{\Theta^{(k)}}(x),y) \partial_{\theta} h(\theta_r^{(k)},x)$ where $\Theta^{(k)} = \{\theta_r^{(k)}\}_{r=1}^M$.
Indeed, if we set the number of inner-iterations $T_t=1$ and replace the direction $\nabla_\theta \overline{g}^{(t)}(\theta_r^{(k)})$ with the gradient of the $L_2$-regularized loss, then PDA exactly reduces to the standard noisy gradient descent algorithm considered in \cite{mei2018mean}.
Algorithm~\ref{alg:pda} can be extended to the minibatch variant in the obvious manner; for efficient implementation in the empirical risk minimization setting see Appendix \ref{sec:implementation}. 

\subsection{Mean Field View of PDA}\label{subsec:mf-pda}
In this subsection we discuss the mean field limit of PDA and explain its algorithmic intuition. 
Note that the inner loop of Algorithm \ref{alg:pda} is the Langevin algorithm with $M$ particles, which optimizes the potential function given by the weighted sum:
\[ \overline{g}^{(t)}(\theta) = \frac{2}{\lambda_2(t+2)(t+1)}\sum_{s=1}^t s \left(\partial_z\ell( h_{\tilde{\Theta}^{(s)}}(x_{s}),y_{s}) h(\theta,x_{s}) + \lambda_1\|\theta\|_2^2 \right). \]
Due to the rapid convergence of Langevin algorithm outlined in Subsection \ref{subsec:langevin}, the particles $\theta_r^{(k+1)}$ $(r \in \{1,\ldots,M\})$ can be regarded as (approximate) samples from the Boltzmann distribution: $\exp\left( - \overline{g}^{(t)}\right)$.
Hence, the inner loop of PDA returns an $M$-particle approximation of some stationary distribution, which is then modified in the outer loop. 
Importantly, the update on the stationary distribution is designed so that the algorithm converges to the optimal solution of the problem \eqref{eq:rerm}.

We now introduce the \textit{mean field limit} of PDA, i.e., taking the number of particles $M\to\infty$ and directly optimizing the problem (\ref{eq:rerm}) over $q$. 
We refer to this mean field limit simply as the dual averaging (DA) algorithm.
The dual averaging method was originally developed for the convex optimization in finite-dimensional spaces \citep{nesterov2005smooth,nesterov2009primal,xiao2009dual}, and here we adapt it to optimization on the probability space.
The detail of the DA algorithm is described in Algorithm \ref{alg:da}.

\begin{algorithm}[ht]
  \caption{Dual Averaging (DA)}
  \label{alg:da}
\begin{algorithmic}
  \STATE {\bfseries Input:}
  data distribution $\cD$ and
  initial density $q^{(1)}$
\vspace{1mm}
   \FOR{$t=1$ {\bfseries to} $T$}
   \STATE Randomly draw a data $(x_t,y_t)$ from $\cD$ \\
   $g^{(t)} \leftarrow \partial_z\ell( h_{q^{(t)}}(x_t),y_{t}) h(\cdot,x_{t}) + \lambda_1\|\cdot\|_2^2$ \\
   Obtain an approximation $q^{(t+1)}$ of the density function $q^{(t+1)}_* \propto \exp\left( - \frac{\sum_{s=1}^t 2s g^{(s)}}{\lambda_2 (t+2)(t+1)}\right)$
   \ENDFOR
\vspace{1mm}   
\STATE Randomly pick up $t \in \{2,3,\ldots,T+1\}$ following the probability $\bP[t]=\frac{2t}{T(T+3)}$ and return $h_{q^{(t)}}$
\end{algorithmic}
\end{algorithm}
 
Algorithm~\ref{alg:da} iteratively updates the density function $q^{(t+1)}_* \in \cP_2$ which is a solution to the objective:
\begin{equation}
\min_{q \in \cP_2} \left\{ \bE_q\Big[ \sum_{s=1}^t sg^{(s)} \Big] + \frac{\lambda_2}{2}(t+2)(t+1)\bE_q[\log(q)]  \right\}, \label{eq:subprob}    
\end{equation} 
where the function $g^{(t)} = \partial_z\ell( h_{q^{(t)}}(x_{t}),y_{t}) h(\cdot,x_{t}) + \lambda_1\|\cdot\|_2^2$ is the functional derivative of $\ell(h_q(x_{i_i}),y_{t}) + \lambda_1 \bE_{q}[\|\theta\|_2^2]$ with respect to $q$ at $q^{(t)}$.
In other words, the objective (\ref{eq:subprob}) is the sum of weighted average of linear approximations of loss function and the entropic regularization in the space of probability distributions.
In this sense, the DA method can be seen as an extension of the Langevin algorithm to handle entropic regularized nonlinear functionals on the probability space by iteratively \textit{linearizing} the objective. 

To sum up, we may interpret the DA method as approximating the optimal distribution $q_*$ by iteratively optimizing $q_*^{(t)}$, which takes the form of a Boltzmann distribution.
In the inner loop of the PDA algorithm, we obtain $M$ (approximate) samples from $q_*^{(t)}$ via the Langevin algorithm. In other words, PDA can be viewed as a finite-particle approximation of DA -- indeed, the stationary distributions obtained in PDA converges to $q_*^{(t+1)}$ by taking $M \rightarrow \infty$. In the following section, we present the convergence rate of the DA method, and also take into account the iteration complexity of the Langevin algorithm; we defer the finite-particle approximation error analysis to Appendix~\ref{sec:discretization}.

\section{Convergence Analysis}\label{sec:convergence-analysis}
We now provide quantitative global convergence guarantee for our proposed method in discrete time. We first derive the outer loop complexity, assuming approximate optimality of the inner loop iterates, which we then verify in the inner loop analysis. The total complexity is then simply obtained by combining the outer- and inner-loop runtime. 

\subsection{Outer Loop Complexity}
\label{subsec:outer-loop}
We first analyze the convergence rate of the dual averaging (DA) method (Algorithm \ref{alg:da}).
Our analysis will be made under the following assumptions.
\begin{assumption} \
\begin{description}[topsep=0mm,itemsep=0mm] \label{assump:convergence}
\item{{\bf(A1)}} $\cY \subset [-1,1]$. $\ell(z,y)$ is a smooth convex function w.r.t.~$z$ and $|\partial_z \ell(z,y)| \leq 2$ for $y,z \in \cY$.
\item{{\bf(A2)}} $|h(\theta,x)| \leq 1$ and $h(\theta,x)$ is smooth with respect to $\theta$ for $x \in \cX$.
\item{{\bf(A3)}} $\KL(q^{(t+1)}\| q^{(t+1)}_*) \leq 1 / t^2$.
\end{description}
\end{assumption}
\vspace{-2.5mm}

\paragraph{Remark.} {\bf (A2)} is satisfied by smooth activation functions such as sigmoid and tanh. Many loss functions including the squared loss and logistic loss satisfy {\bf (A1)} under the boundedness assumptions $\cY \subset [-1,1]$ and $|h_\theta(x)| \leq 1$. 
Note that constants in {\bf (A1)} and {\bf (A2)} are defined for simplicity and can be relaxed to any value. {\bf (A3)} specifies the precision of approximate solutions of sub-problems (\ref{eq:subprob}) to guarantee the global convergence of Algorithm \ref{alg:da}, which we verify in our inner loop analysis. 

We first introduce the following quantity for $q \in \cP_2$,
\[ e(q) \defeq \bE_q[\log(q)] - \frac{4}{\lambda_2} - \frac{p}{2}\left( \exp\left( \frac{4}{\lambda_2} \right) + \log\left(\frac{3\pi\lambda_2}{\lambda_1}\right)\right). \]
Observe that the expression consists of the negative entropy minus its lower bound for $q^{(t)}_*$ under Assumption {\bf(A1)}, {\bf(A2)}; in other words $e(q^{(t)}_*) \geq 0$. We have the following convergence rate of DA\footnote{In Appendix~\ref{sec:proof} we introduce a more general version of Theorem~\ref{thm:convergence} that allows for inexact $h_{q^{(t)}}(x)$, which simplifies the analysis of finite-particle discretization presented in Appendix~\ref{sec:discretization}.}. 
\begin{theorem}[Convergence of DA]\label{thm:convergence}
Under Assumptions {\bf(A1)}, {\bf(A2)}, and {\bf(A3)}, for arbitrary $q_* \in \cP_2$, iterates of the DA method (Algorithm \ref{alg:da}) satisfies
\begin{align*}
&\hspace{-7mm}\frac{2}{T(T+3)} \sum_{t=2}^{T+1}t \left( \bE[\cL(q^{(t)})] - \cL(q_*) \right) \\
&\leq 
O\Bigl( \frac{1}{T^2}\left(1+\lambda_1 \bE_{q_*}\left[ \|\theta\|_2^2\right] \right)
+ \frac{ \lambda_2  e(q_*) }{T} + \frac{ \lambda_2 }{T}(1 + \exp(8/\lambda_2))p^2 \log^2(T+2) \Bigr),
\end{align*}
where the expectation $\bE[\cL(q^{(t)})]$ is taken with respect to the history of examples.
\end{theorem}
Theorem \ref{thm:convergence} demonstrates the convergence rate of Algorithm \ref{alg:da} to the optimal value of the regularized objective~\eqref{eq:rerm} in expectation.
Note that $\frac{2}{T(T+3)} \sum_{t=2}^{T+1}t \bE[\cL(q^{(t)})]$ is the expectation of $\bE[\cL(q^{(t)})]$ according to the probability $\cP[t]=\frac{2t}{T(T+3)}$ $(t \in \{2,\ldots,T+1\})$
as specified in Algorithm \ref{alg:da}. 
If we take $p, \lambda_1, \lambda_2$ as constants and use $\tilde{O}$ to hide the logarithmic terms,
we can deduce that after $\tilde{O}(\epsilon^{-1})$ iterations, an $\epsilon$-accurate solution of the optimal distribution: $\cL(q) \leq \inf_{q\in \cP_2}\cL(q) + \epsilon$ is achieved in expectation. 
Importantly, this convergence rate applies to \textit{any} choice of regularization parameters, in contrast to the strong regularization required in \cite{hu2019mean,jabir2019mean}. 

On the other hand, due to the exponential dependence on $\lambda_2^{-1}$, our convergence rate is not informative under weak regularization $\lambda_2 \to 0$. Such dependence follows from the classical LSI perturbation lemma~\citep{holley1987logarithmic}, which is likely unavoidable for Langevin-based methods in the most general setting \citep{menz2014poincare}, unless additional assumptions are imposed (e.g., a student-teacher setup); we intend to further investigate these conditions in future work.  

\subsection{Inner Loop Complexity}
\label{subsec:inner-loop}
In order to derive the total complexity (i.e., taking both the outer loop and inner loop into account) towards a required accuracy, we also need to estimate the iteration complexity of Langevin algorithm.
We utilize the following convergence result under the log-Sobolev inequality (Definition~\ref{def:log-sobolev}):

\begin{theorem}[\cite{vempala2019rapid}]\label{thm:vempala-wibisono}
Consider a probability density $q(\theta) \propto \exp(-f(\theta))$ satisfying the log-Sobolev inequality with constant $\alpha$, and assume $f$ is smooth and $\nabla f$ is $L$-Lipschitz, i.e., $\|\nabla_\theta f(\theta) - \nabla_\theta f(\theta')\|_2 \leq L\|\theta - \theta'\|_2$.
If we run the Langevin algorithm (\ref{eq:ula}) with learning rate $0 < \eta \leq \frac{\alpha}{4L^2}$ and let $q^{(k)}(\theta)\rd\theta$ be a probability distribution that $\theta^{(k)}$ follows,
then we have,  
\begin{equation*}
\KL(q^{(k)} \| q) \leq \exp( -\alpha \eta k)\KL(q^{(1)} \| q) + 8\alpha^{-1}\eta p L^2.
\end{equation*}
\end{theorem}
Theorem \ref{thm:vempala-wibisono} implies that a $\delta$-accurate solution in KL divergence can be obtained by the Langevin algorithm with $\eta \leq \frac{\alpha}{4L^2} \min\left\{1,\frac{\delta}{4p}\right\}$ and $\frac{1}{\alpha\eta}\log\frac{2\KL(q^{(1)}\| q)}{\delta}$-iterations.

Since the optimal solution of a sub-problem in DA (Algorithm \ref{alg:da}) takes the forms of $q^{(t+1)}_* \propto \exp\left( -\frac{\sum_{s=1}^t 2sg^{(s)}}{\lambda_2 (t+2)(t+1)}\right)$, we can verify the LSI and determine the constant for $q_*^{(t+1)}(\theta)\rd\theta$ based on the LSI perturbation lemma from \citet{holley1987logarithmic} (see Lemma~\ref{lem:sobolev_perturbation} and Example~\ref{ex:quadratic_sobolev} in Appendix~\ref{app:log-sobolev}). 
Consequently, we can apply Theorem \ref{thm:vempala-wibisono} to $q_*^{(t+1)}$ for the inner loop complexity when $\nabla_\theta \log q_*^{(t+1)}$ is Lipschitz continuous, which motivates us to introduce the following assumption.  

\begin{assumption} \ 
\begin{description}[topsep=0mm,itemsep=0mm]  \label{assump:smoothness}
\item{{\bf(A4)}\,} $\partial_\theta h(\cdot,x)$ is $1$-Lipschitz continuous: $\|\partial_\theta h(\theta,x)- \partial_\theta h(\theta',x)\|_2 \leq \|\theta - \theta'\|_2$, $\forall x \in \cX$, $\theta,\theta' \in \Omega$.\!\!
\end{description} 
\end{assumption}
\vspace{-2mm}
\paragraph{Remark.} 
{\bf(A4)} is parallel to \citep[Assumption A3]{mei2018mean}, and is satisfied by two-layer neural network in Example \ref{eg:2nn} when the output or input layer is fixed and the input space $\cX$ is compact.
We remark that this assumption can be relaxed to H\"older continuity of $\partial_\theta h(\cdot,x)$ via the recent result of \citet{erdogdu2020convergence}, which allows us to extend Theorem~\ref{thm:convergence} to general $L_p$-norm regularizer for $p>1$. 
For now we work with {\bf(A4)} for simplicity of the presentation and proof.

Set $\delta_{t+1}$ to be the desired accuracy of an approximate solution $q^{(t+1)}$ specified in {\bf (A3)}: 
$\delta_{t+1} = 1/(t+1)^2$, we have the following guarantee for the inner loop.  
\begin{corollary}[Inner Loop Complexity] \label{cor:inner-complexity}
Under {\bf (A1)}, {\bf (A2)}, and {\bf (A4)}, if we run the Langevin algorithm with step size $\eta_t =  O\left(\frac{\lambda_1 \lambda_2 \delta_{t+1}}{p(1+\lambda_1)^2\exp(8/\lambda_2)}\right)$ on (\ref{eq:subprob}), 
then an approximate solution satisfying $\KL(q^{(t+1)}\| q^{(t+1)}_*) \leq \delta_{t+1}$ can be obtained within $O\left(\frac{\lambda_2\exp(8/\lambda_2)}{\lambda_1\eta_t}\log\frac{2\KL(q^{(t)}\| q^{(t+1)}_*)}{\delta_{t+1}} \right)$-iterations.
Moreover, $\KL(q^{(t)}\|q^{(t+1)}_*)$ $(t\in\{1,2,\ldots,T+1\})$ are uniformly bounded with respect to $t$ as long as $q^{(1)}$ is a Gaussian distribution and {\bf(A3)} is satisfied.
\end{corollary}

We comment that for the inner loop we utilized the \textit{overdamped} Langevin algorithm, since it is the most standard and commonly used sampling method for the objective~\eqref{eq:subprob}. Our analysis can easily incorporate other inner loop updates such as the underdamped Langevin algorithm \citep{cheng2018underdamped,eberle2019couplings} or the Metropolis-adjusted Langevin algorithm \citep{roberts1996exponential,dwivedi2018log}, which may improve the iteration complexity.  

\subsection{Total Complexity}
Combining Theorem \ref{thm:convergence} and Corollary \ref{cor:inner-complexity}, we can now derive the total complexity of our proposed algorithm.
For simplicity, we take $p, \lambda_1, \lambda_2$ as constants and hide logarithmic terms in $\tilde{O}$ and $\tilde{\Theta}$.
The following corollary establishes a $\tilde{O}(\epsilon^{-3})$ total iteration complexity to obtain an $\epsilon$-accurate solution in expectation because $T_t = \tilde{\Theta}(t^2) = \tilde{O}(\epsilon^{-2})$ for $t \leq T$.
\begin{corollary}[Total Complexity]\label{cor:total-complexity}
Let $\epsilon > 0$ be an arbitrary desired accuracy and $q^{(1)}$ be a Gaussian distribution.
Under assumptions {\bf (A1)}, {\bf (A2)}, {\bf (A3)}, and {\bf (A4)}, if we run Algorithm \ref{alg:da} for $T=\tilde{\Theta}(\epsilon^{-1})$ iterations on the outer loop, and the Langevin algorithm with step size $\eta_t =  \Theta\left(\frac{\lambda_1 \lambda_2 \delta_{t+1}}{p(1+\lambda_1)^2\exp(8/\lambda_2)}\right)$ for $T_t = \tilde{\Theta}(\eta_t^{-1})$ iterations on the inner loop, then an $\epsilon$-accurate solution: $\cL(q) \leq \inf_{q\in \cP_2}\cL(q) + \epsilon$ of the objective~\eqref{eq:rerm} is achieved in expectation. 
\end{corollary}

\paragraph{Quantitative convergence guarantee.} 
To translate the above convergence rate result to the finite-particle PDA (Algorithm~\ref{alg:pda}), we also characterize the finite-particle discretization error in Appendix~\ref{sec:discretization}. 
For the particle complexity analysis, we consider two versions of particle update: 
($i$) the \textit{warm-start} scheme described in Algorithm~\ref{alg:pda}, in which  $\Theta^{(1)}$ is initialized at the last iterate $\tilde{\Theta}^{(t)}$ of the previous inner loop, and ($ii$) the \textit{resampling} scheme, in which $\Theta^{(1)}$ is initialized from the initial distribution $q^{(1)}(\theta)\rd\theta$ (see Appendix~\ref{sec:proof} for details). 
Remarkably, for the resampling scheme, we provide convergence rate guarantee in time- and space-discretized settings that is \textit{polynomial in both the iterations and particle size}; specifically, the particle complexity of $\tilde{O}(\epsilon^{-2})$, together with the total iteration complexity of $\tilde{O}(\epsilon^{-3})$, suffices to obtain an $\epsilon$-accurate solution to the objective~\eqref{eq:rerm} (see Appendix~\ref{sec:proof} and \ref{sec:discretization} for precise statement).

\section{Experiments}\label{sec:experiments}

\subsection{Experiment Setup}

We employ our proposed algorithm in both synthetic student-teacher settings (see Figure~\ref{fig:experiment}(a)(b)) and real-world dataset (see Figure~\ref{fig:experiment}(c)). 
For the student-teacher setup, the labels are generated as $y_i = f_*(x_i) + \varepsilon_i$, where $f_*$ is the teacher model (target function), and $\varepsilon$ is zero-mean i.i.d.~label noise. 
For the student model $f$, we follow \citet[Section 2.1]{mei2018mean} 
and parameterize a two-layer neural network with fixed second layer as:   
\begin{align}
\label{eq:NN_experiment}
    f(x) = \frac{1}{M^\alpha} \sum_{r=1}^M \sigma (w_r^\top x + b_r),
\end{align}
which we train to minimize the objective~\eqref{eq:rerm} using PDA. Note that $\alpha=1$ corresponds to the mean field regime (which we are interested in), whereas setting $\alpha=1/2$ leads to the kernel (NTK) regime\footnote{We use the term \textit{kernel regime} only to indicate the parameter scaling $\alpha$; this does not necessarily imply that the NTK linearization is an accurate description of the trained model.}.

\vspace{-1mm}
\paragraph{Synthetic student-teacher setting.}
For Figure~\ref{fig:experiment}(a)(b) we design synthetic experiments for both regression and classification tasks, where the student model is a two-layer tanh network with $M=500$.  
For regression, we take the target function $f_*$ to be a multiple-index model with $m$ neurons: $f_*(x) = \frac{1}{\sqrt{m}}\sum_{i=1}^m\sigma_*(\langle w^*_i,x\rangle)$, and the input is drawn from a unit Gaussian $\mathcal{N}(0,I_p)$. 
For binary classification, we consider a simple two-dimensional dataset from \texttt{sklearn.datasets.make$\_$circles} \citep{scikit-learn}, in which the goal is to separate two groups of data on concentric circles (red and blue in Figure~\ref{fig:experiment}(b)). We include additional experimental results in Appendix~\ref{app:additional_experiment}.

\vspace{-1mm} 
\paragraph{PDA hyperparameters.}
We optimize the \textit{squared loss} for regression and the \textit{logistic loss} for binary classification. The model is trained by PDA with batch size 50. 
We scale the number of inner loop steps $T_t$ with $t$, and the step size $\eta_t$ with $1/\sqrt{t}$, where $t$ is the outer loop iteration; this heuristic is consistent with the required inner-loop accuracy in Theorem~\ref{thm:convergence} and Proposition~\ref{cor:total-complexity}.

\begin{figure}[htb!]
\centering
\hspace{-5mm}  
\begin{minipage}[t]{0.32\linewidth}
\centering
{\includegraphics[width=0.98\textwidth]{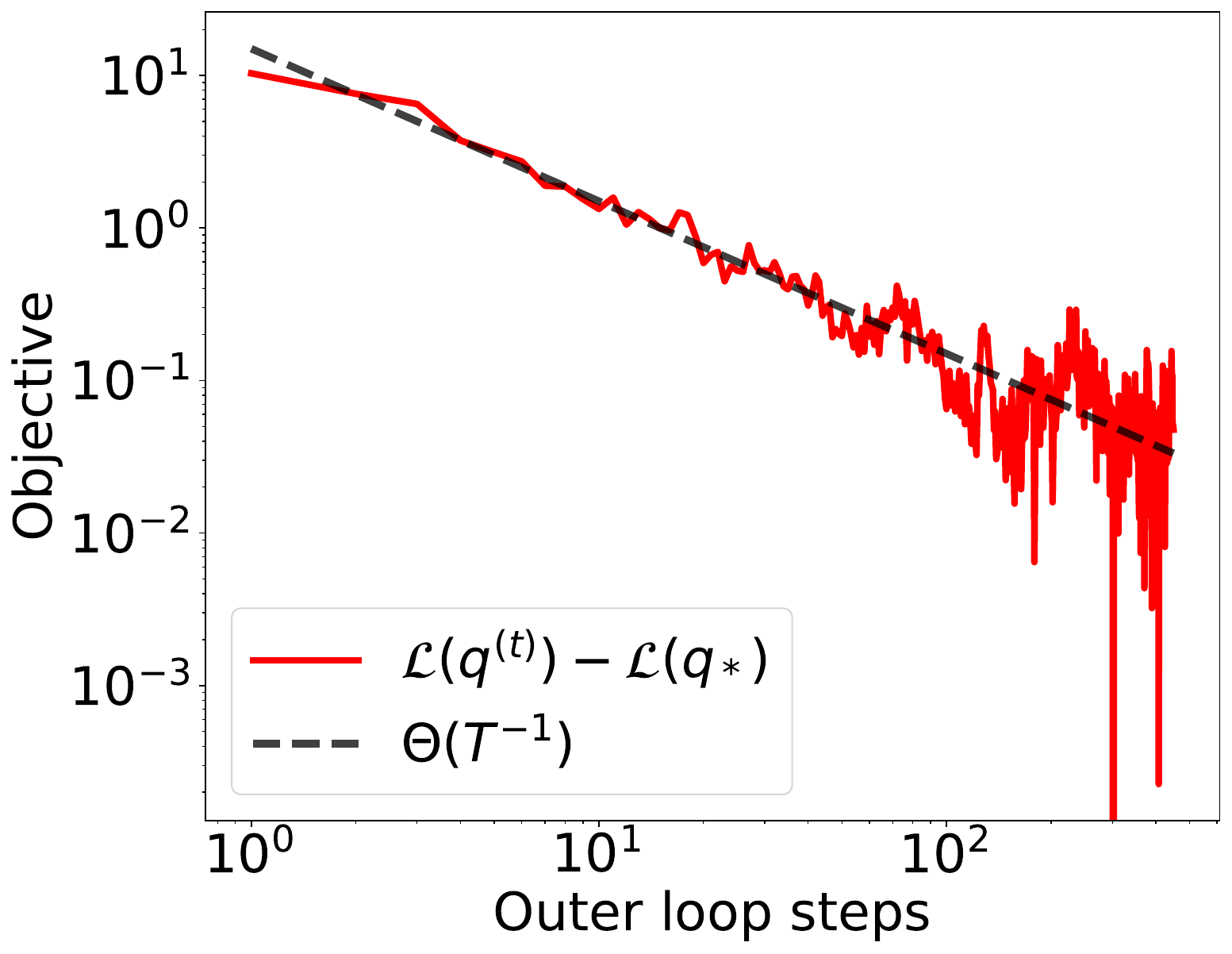}} \\ \vspace{-0.mm}
\small ~~(a) objective value \\~(regression). 
\end{minipage}
\begin{minipage}[t]{0.02\linewidth}
\quad 
\end{minipage}
\hspace{-5mm}  
\begin{minipage}[t]{0.32\linewidth}
\centering
{\vspace{-34mm} \includegraphics[width=0.92\textwidth]{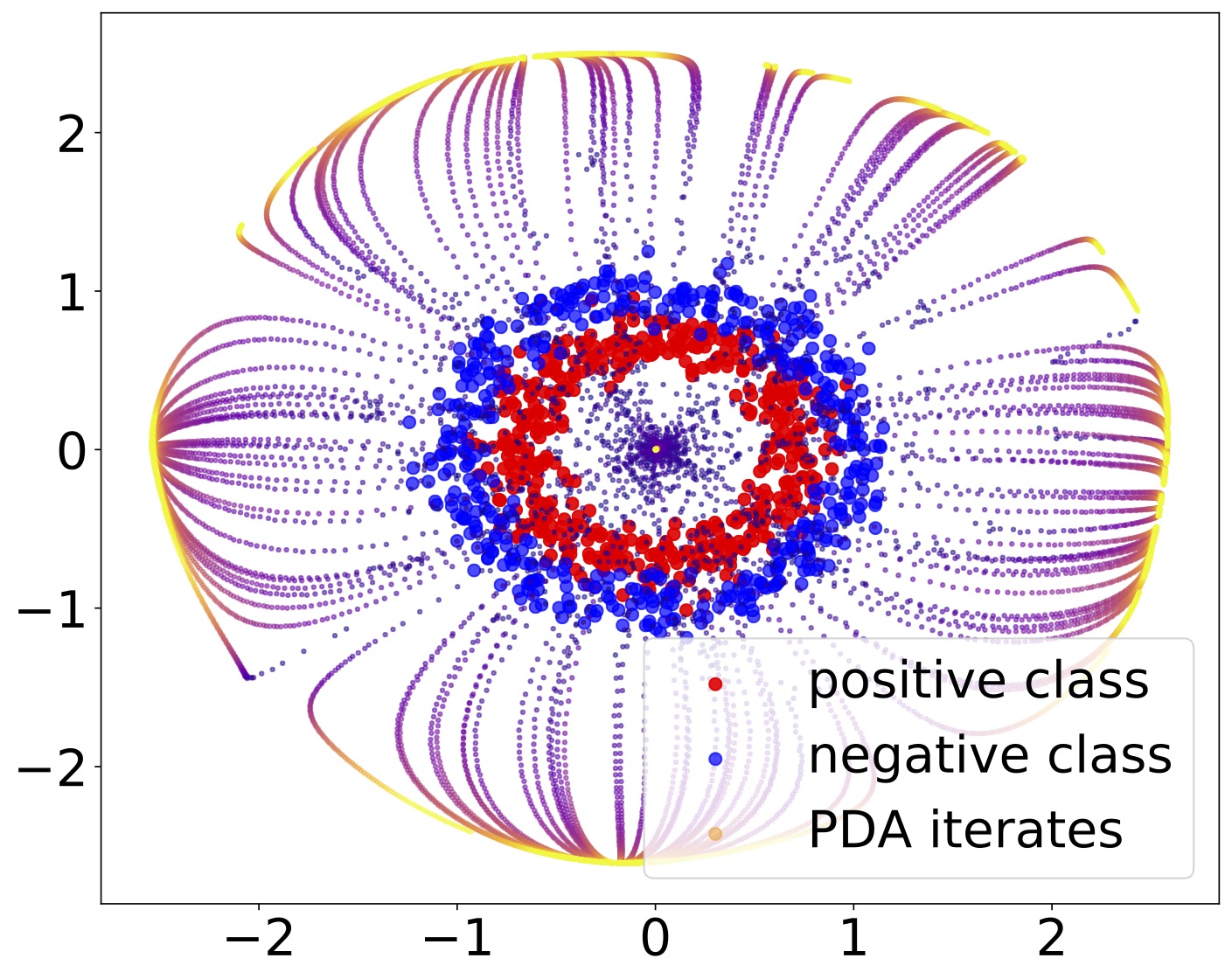}} \\ \vspace{1.1mm}
\small (b) parameter trajectory \\(classification). 
\end{minipage}
\begin{minipage}[t]{0.32\linewidth}
\centering 
{\includegraphics[width=1.1\textwidth]{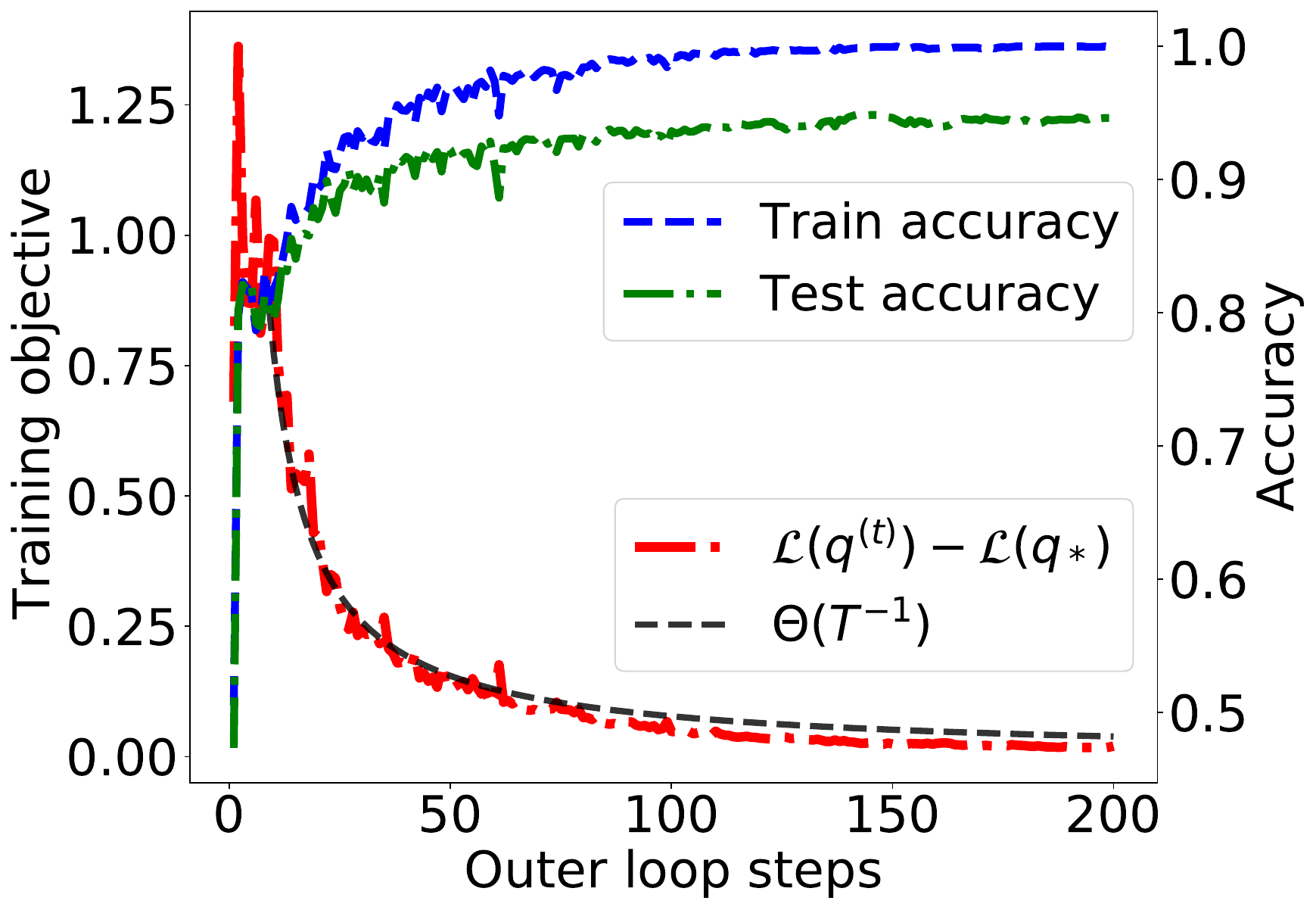}} \\ \vspace{-0.mm}
\small ~~(c) MNIST odd vs.~even \\~~(classification). 
\end{minipage} 
 \caption{\small (a) Iteration complexity of PDA: the $O(T^{-1})$ rate on the outer loop agrees with Theorem~\ref{thm:convergence}. (b) Parameter trajectory of PDA: darker color (purple) indicates earlier in training, and vice versa. (c) odd vs.~even classification on MNIST; we report the training loss (red) as well as the train and test accuracy (blue and green). }
\label{fig:experiment}
\end{figure} 

\subsection{Empirical Findings}

\paragraph{Convergence rate.} In Figure~\ref{fig:experiment}(a) we verify the $O(T^{-1})$ iteration complexity of the outer loop in Theorem~\ref{thm:convergence}. We apply PDA to optimize the expected risk (analogous to one-pass SGD) in the regression setting, in which the input dimensionality $p=1$ and the target function is a single-index model ($m=1$) with tanh activation. We employ the \textit{resampled} update (i.e., without warm-start; see Appendix~\ref{sec:proof}) with hyperparameters $\lambda_1=10^{-2}, \lambda_2=10^{-3}$. To compute the entropy in the objective \eqref{eq:rerm}, we adopt the $k$-nearest neighbors estimator \citep{kozachenko1987sample} with $k=10$.  
  
\vspace{-0.5mm} 
\paragraph{Presence of feature learning.}
In Figure~\ref{fig:experiment}(b) we visualize the evolution of neural network parameters optimized by PDA in a 2-dimensional classification problem. 
Due to structure of the input data (concentric rings), we expect that for a two-layer neural network to be a good separator, its parameters should also distribute on a circle. Indeed the converged solution of PDA (bright yellow) agrees with this intuition and demonstrates that PDA learns useful features beyond the kernel regime.  
  
\vspace{-0.5mm} 
\paragraph{Binary classification on MNIST.} In Figure~\ref{fig:experiment}(c) we report the training and test performance of PDA in separating odd vs.~even digits from the MNIST dataset. We subsample $n=2500$ training examples with binary labels, and learn a two-layer tanh network with width $M=2500$. We use the resampled update of PDA to optimize the cross entropy loss, with hyperparameters $\lambda_1=10^{-2},  \lambda_2=10^{-4}$. Observe that the algorithm achieves good generalization performance (green) and roughly maintains\footnote{Note that the estimated training objective (red) slightly deviates from the ideal $1/T$-rate; this may be due to inaccuracy in the entropy estimation, or non-convergence of the algorithm (i.e., overestimation of $\cL(q_*)$). } the $O(T^{-1})$ iteration complexity (red) in optimizing the training objective~\eqref{eq:rerm}.  
 

\section*{CONCLUSION}
We proposed the particle dual averaging (PDA) algorithm for optimizing two-layer neural networks in the mean field regime. Leveraging tools from finite-dimensional convex optimization developed in the original dual averaging method, we established \textit{quantitative} convergence rate of PDA for regularized empirical and expected risk minimization. We also provided particle complexity analysis and generalization bounds for both regression and classification problems. Our theoretical findings are aligned with experimental results on neural network optimization. Looking forward, we plan to investigate specific problem instances in which convergence rate can be obtained under vanishing regularization. It is also important to consider accelerated variants of PDA to further improve the convergence rate in the empirical risk minimization setting. Another interesting direction would be to explore other applications of PDA beyond two-layer neural networks, such as deep models \citep{araujo2019mean,nguyen2020rigorous,lu2020mean,pham2021global}, as well as other optimization problems for entropic regularized nonlinear functional.  

\bigskip
\section*{Acknowledgment}
The authors would like to thank Murat A.~Erdogdu and anonymous NeurIPS reviewers for their helpful feedback. 
AN was partially supported by JSPS Kakenhi (19K20337) and JST-PRESTO (JPMJPR1928).
DW was partially supported by NSERC and LG Electronics.
TS was partially supported by JSPS KAKENHI (18H03201), Japan Digital Design and JST CREST.

\bigskip

\bibliographystyle{apalike}
\bibliography{ref}

\clearpage
\onecolumn
\renewcommand{\thesection}{\Alph{section}}
\renewcommand{\thesubsection}{\Alph{section}. \arabic{subsection}}
\renewcommand{\thetheorem}{\Alph{theorem}}
\renewcommand{\thelemma}{\Alph{lemma}}
\renewcommand{\theproposition}{\Alph{proposition}}
\renewcommand{\thedefinition}{\Alph{definition}}
\renewcommand{\thecorollary}{\Alph{corollary}}
\renewcommand{\theassumption}{\Alph{assumption}}
\renewcommand{\theexample}{\Alph{example}}

\setcounter{section}{0}
\setcounter{subsection}{0}
\setcounter{theorem}{0}
\setcounter{lemma}{0}
\setcounter{proposition}{0}
\setcounter{definition}{0}
\setcounter{corollary}{0}
\setcounter{assumption}{0}

{
\renewcommand{\contentsname}{Table of Contents}
\tableofcontents
}
\newpage

\allowdisplaybreaks
\part*{\large MISSING PROOFS}
\section{Preliminaries}

\subsection{Entropic Regularized Linear Functional}
In this section, we explain the property of the optimal solution of the entropic regularized linear functional.
We here define the gradient of the negative entropy $\bE_q[\log(q)]$ with respect to $q$ over the probability space as $\nabla_q \bE_q[\log(q)] = \log(q)$.
Note that this gradient is well defined up to constants as a linear operator on the probability space: $q^\prime \mapsto \int (q^\prime-q)(\theta)\log(q(\theta))\rd\theta$.
The following lemma shows the strong convexity of the negative entropy.
\begin{lemma}\label{lem:st-conv}
Let $q, q^{\prime}$ be probability densities such that the entropy and Kullback-Leibler divergence $\KL(q^{\prime}\|q) = \int q^{\prime}(\theta)\log\left(\frac{q^{\prime}(\theta)}{q(\theta)}\right) \rd \theta$ are well defined.
Then, we have
\begin{align*}
&\bE_{q}[\log(q)] + \int (q^{\prime}-q)(\theta) \nabla_q\bE_{q}[\log(q)] \rd\theta + \KL(q^{\prime}\|q) = \bE_{q^{\prime}}[\log(q^{\prime})], \\
&\bE_{q}[\log(q)] + \int (q^{\prime}-q)(\theta) \nabla_q\bE_{q}[\log(q)] \rd\theta +\frac{1}{2}\| q^{\prime}-q\|_{L_1(\rd \theta)}^2
\leq \bE_{q^{\prime}}[\log(q^{\prime})].
\end{align*}
\end{lemma}
The first equality of this lemma can be shown by the direct computation of the entropy, and the second inequality can be obtained by Pinsker's inequality $\frac{1}{2}\| q^{\prime}-q\|_{L_1(\rd \theta)}^2 \leq \KL(q^{\prime}\|q)$.

Recall that $\cP_2$ is the set of positive densities on $\bR^p$ such that the second moment $\bE_q[\|\theta\|_2^2] < \infty$ and entropy $-\infty < -\bE_q[\log(q)] < +\infty$ are well defined.
We here consider the minimization problem of entropic regularized linear functional on $\cP_2$.
Let $\lambda_1, \lambda_2 > 0$ be positive real numbers and $H: \bR^p \rightarrow \bR$ be a bounded continuous function.
\begin{equation}\label{eq:entropy_and_linear_functional}
\min_{q \in \cP_2} \left\{ F(q) \defeq \bE_q[H(\theta)] + \lambda_1\bE_q[\|\theta\|_2^2]+ \lambda_2\bE_q[\log(q(\theta))] \right\}.
\end{equation}
Then, we can show $q \propto \exp\left( -\frac{H(\theta) + \lambda_1 \|\theta\|_2^2}{\lambda_2}\right)$ is an optimal solution of the problem (\ref{eq:entropy_and_linear_functional}) as follow.
Clearly, $q \in \cP_2$ and the assumption on $q$ in Lemma \ref{lem:st-conv} with $q^{\prime} \in \cP_2$ holds. 
Hence, for $\forall q^{\prime} \in \cP_2$, 
\begin{align}
F(q) 
&= \bE_q[H(\theta)] + \lambda_1\bE_q[\|\theta\|_2^2]+ \lambda_2\bE_q[\log(q(\theta))] \notag\\
&= \bE_{q^{\prime}}[H(\theta)] + \lambda_1\bE_{q^{\prime}}[\|\theta\|_2^2]+ \lambda_2\bE_{q^{\prime}}[\log(q^{\prime}(\theta))] \notag\\
&+ \int (q-q^{\prime})(\theta)\left( H(\theta) + \lambda_1 \|\theta\|_2^2\right) \rd\theta 
+ \lambda_2 \left( \bE_{q}[\log(q(\theta))] - \bE_{q^{\prime}}[\log(q^{\prime}(\theta))]  \right) \notag\\
&= F(q^{\prime}) 
+ \int (q-q^{\prime})(\theta)\left( H(\theta) + \lambda_1 \|\theta\|_2^2\right) \rd\theta  
+ \lambda_2 \left( \bE_{q}[\log(q(\theta))] - \bE_{q^{\prime}}[\log(q^{\prime}(\theta))]  \right) \notag\\
&\leq 
F(q^{\prime}) 
+ \int (q-q^{\prime})(\theta)\left( H(\theta) + \lambda_1 \|\theta\|_2^2\right) \rd\theta  
- \lambda_2 \left( \int (q^{\prime}-q)(\theta) \nabla_q\bE_{q}[\log(q)] \rd\theta +\frac{1}{2}\| q^{\prime}-q\|_{L_1(\rd \theta)}^2 \right) \notag\\
&=
F(q^{\prime}) 
+ \int (q-q^{\prime})(\theta)\left( H(\theta) + \lambda_1 \|\theta\|_2^2 + \lambda_2 \log(q(\theta))\right) \rd\theta - \frac{\lambda_2}{2}\| q^{\prime}-q\|_{L_1(\rd \theta)}^2 \notag\\
&=
F(q^{\prime})  - \frac{\lambda_2}{2}\| q^{\prime}-q\|_{L_1(\rd \theta)}^2. \label{eq:st-conv}
\end{align}
For the inequality we used Lemma \ref{lem:st-conv} and for the last equality we used $q \propto \exp\left( -\frac{H(\theta) + \lambda_1 \|\theta\|_2^2}{\lambda_2}\right)$.
Therefore, we conclude that $q$ is a minimizer of $F$ on $\cP_2$ and the strong convexity of $F$ holds at $q$ with respect to $L_1(\rd\theta)$-norm.
This crucial property is used in the proof of Theorem \ref{thm:convergence}.

\subsection{Log-Sobolev and Talagrand's Inequalities}
\label{app:log-sobolev}
The log-Sobolev inequality is useful in establishing the convergence rate of Langevin algorithm.
\begin{definition}[Log-Sobolev inequality]
\label{def:log-sobolev}
Let $\rd \mu = p(\theta)\rd \theta$ be a probability distribution with a positive smooth density $p>0$ on $\bR^p$.
We say that $\mu$ satisfies the log-Sobolev inequality with constant $\alpha>0$ if for any smooth function $f: \bR^p \rightarrow \bR$, 
\[ \bE_\mu[f^2 \log f^2] - \bE_{\mu}[f^2]\log\bE_\mu[f^2] \leq \frac{2}{\alpha}\bE_\mu[\|\nabla f\|_2^2]. \]
\end{definition}
This inequality is analogous to strong convexity in optimization: let $\rd \nu = q(\theta)\rd \mu$ be a probability distribution on $\bR^p$ such that $q$ is smooth and positive.
Then, if $\mu$ satisfies the log-Sobolev inequality with $\alpha$, it follows that 
\[ \KL(\nu||\mu) \leq \frac{1}{2\alpha} \bE_\nu[ \| \nabla_\theta \log q \|_2^2]. \]
The above relation is directly obtained by setting $f=\sqrt{q}$ in the definition of log-Sobolev inequality.
Note that the right hand side is nothing else but the squared norm of functional gradient of $\KL(\nu \| \mu)$ with respect to a transport map for $\nu$. 

It is well-known that strong log-concave densities satisfy the LSI with a dimension-free constant (up to the spectral norm of the covariance).
\begin{example}[\cite{bakry1985diffusions}]\label{ex:quadratic_sobolev}
Let $q \propto \exp(-f)$ be a probability density, where $f : \bR^p \rightarrow \bR$ is a smooth function. 
If there exists $c > 0$ such that $\nabla^2 f \succeq cI_p$, then $q(\theta)\rd\theta$ satisfies Log-Sobolev inequality with constant $c$.
\end{example}

In addition, the LSI is preserved under bounded perturbation, as originally shown in \cite{holley1987logarithmic}. We also provide a proof for completeness. 

\begin{lemma}[\cite{holley1987logarithmic}]
\label{lem:sobolev_perturbation}
Let $q(\theta)\rd\theta$ be a probability distribution on $\bR^p$ satisfying the log-Sobolev inequality with a constant $\alpha$.
For a bounded function $B: \bR^p \rightarrow \bR$, we define a probability distribution $q_B(\theta)\rd\theta$ as follows:
\[ q_B(\theta)\rd \theta = \frac{\exp(B(\theta))q(\theta)}{ \bE_q[\exp(B(\theta))]} \rd\theta. \]
Then, $q_B \rd\theta$ satisfies the log-Sobolev inequality with a constant $\alpha / \exp(4\|B\|_\infty)$.
\end{lemma}
\begin{proof}
Taking an expectation $\bE_{q_B}$ of the Bregman divergence defined by a convex function $x\log x$, for $\forall a >0$,
\begin{align*}
0 &\leq \bE_{q_B}\left[ f^2(\theta)\log (f^2(\theta)) - (a\log(a) + (\log(a) +1)( f^2(\theta)-a)) \right] \\
&=  \bE_{q_B}\left[ f^2(\theta)\log (f^2(\theta)) - (f^2(\theta) \log(a) + f^2(\theta) -a)\right].
\end{align*}
Since the minimum is attained at $a=\bE_{q_B}[f^2(\theta)]$,
\begin{align*}
0 &\leq \bE_{q_B}\left[ f^2(\theta)\log (f^2(\theta)) \right] - \bE_{q_B}[f^2(\theta)] \log \bE_{q_B}[f^2(\theta)] \\
&=  \inf_{a > 0} \bE_{q_B}\left[ f^2(\theta)\log (f^2(\theta)) - (f^2(\theta) \log(a) + f^2(\theta) -a)\right] \\
&\leq \exp(2\|B\|_\infty) \inf_{a > 0} \bE_{q}\left[ f^2(\theta)\log (f^2(\theta)) - (f^2(\theta) \log(a) + f^2(\theta) -a)\right] \\
&= \exp(2\|B\|_\infty) \left( \bE_{q}\left[ f^2(\theta)\log (f^2(\theta)) \right] - \bE_{q}[f^2(\theta)] \log \bE_{q}[f^2(\theta)] \right) \\
&\leq \frac{2\exp(2\|B\|_\infty) }{\alpha} \bE_q \left[ \| \nabla f \|_2^2 \right] \\
&= \frac{2\exp(2\|B\|_\infty) }{\alpha} \bE_{q_B} \left[ \frac{ \bE_{q}[\exp(B(\theta))]}{ \exp(B(\theta)) }\| \nabla f \|_2^2 
\right] \\
&\leq \frac{2\exp(4\|B\|_\infty) }{\alpha} \bE_{q_B} \left[ \| \nabla f \|_2^2 \right],
\end{align*}
where we used the non-negativity of the integrand for the second inequality.
\end{proof}

We next introduce Talagrand's inequality.
\begin{definition}[Talagrand's inequality]
We say that a probability distribution $q(\theta)\rd\theta$ satisfies Talagrand's inequality with a constant $\alpha > 0$ if for any probability distribution $q^{\prime}(\theta)\rd\theta$ ,
\[  \frac{\alpha}{2}W_2^2(q^{\prime},q) \leq \KL(q^{\prime} \| q), \]
where $W_2(q^{\prime},q)$ denotes the $2$-Wasserstein distance between $q(\theta)\rd\theta$ and $q^{\prime}(\theta)\rd\theta$.
\end{definition}

The next theorem gives a relationship between KL divergence and $2$-Wasserstein distance.
\begin{theorem}[\cite{otto2000generalization}]\label{thm:otto-villani}
If a probability distribution $q(\theta)\rd\theta$ satisfies the log-Sobolev inequality with constant $\alpha>0$, 
then $q(\theta)\rd\theta$ satisfies Talagrand's inequality with the same constant.
\end{theorem}



\section{Proof of Main Results} \label{sec:proof}
\subsection{Extension of Algorithm} \label{subsec:extension}
In this section, we prove the main theorem that provides the convergence rate of the dual averaging method. 
We first introduce a slight extension of PDA (Algorithm~\ref{alg:pda}) which incorporates two different initializations at each outer loop step. We refer to the two versions as the \textit{warm-start} and the \textit{resampled} update, respectively. Note that Algorithm~\ref{alg:pda} in the main text only includes the warm-start update. In Appendix~\ref{sec:discretization} we provide particle complexity analysis for both updates. We remark that the benefit of resampling strategy is the simplicity of estimation of approximation error $| h_x^{(t)} - h_{q^{(t)}}(x_t)|$, because  $h_x^{(t)}$ is composed of i.i.d particles and a simple concentration inequality can be applied to estimate this error. 

\begin{algorithm}[ht]
  \caption{Particle Dual Averaging (\textit{general version})}
  \label{alg:pda2}
\begin{algorithmic}
  \STATE {\bfseries Input:}
  data distribution $\cD$,
  initial density $q^{(1)}$,
  number of outer-iterations $T$,
  learning rates $\{\eta_t\}_{t=1}^T$,
  number of inner-iterations $\{T_t\}_{t=1}^T$
\vspace{1mm}
  \STATE Randomly draw i.i.d.~initial parameters $\tilde{\theta}_r^{(1)} \sim q^{(1)}(\theta)\rd\theta$ $(r\in\{1,2,\ldots,M\})$
  \STATE $\tilde{\Theta}^{(1)} \leftarrow \{\tilde{\theta}_r^{(1)}\}_{r=1}^M$
\vspace{1mm}
   \FOR{$t=1$ {\bfseries to} $T$}
   \STATE Randomly draw a data $(x_t,y_t)$ from $\cD$ \\
   {\color{blue} \textbf{Either} $\Theta^{(1)}=\{\theta_r^{(1)}\}_{r=1}^M \leftarrow \tilde{\Theta}^{(t)}$ ~(\textit{warm-start})  \\\textbf{Or} randomly initialize $\Theta^{(1)}$ from $q^{(1)}(\theta)\rd\theta$ ~(\textit{resampling}) }
\vspace{1mm}
     \FOR{$k=1$ {\bfseries to} $T_t$}
       \STATE Run an inexact noisy gradient descent for $r\in\{1,2,\ldots,M\}$\\
       $\nabla_\theta \overline{g}^{(t)}(\theta_r^{(k)}) \leftarrow \frac{2}{\lambda_2(t+2)(t+1)}\sum_{s=1}^t s \partial_z\ell( h_{\tilde{\Theta}^{(s)}}(x_{s}),y_{s}) \partial_\theta h(\theta_r^{(k)},x_{s}) + \frac{2\lambda_1 t}{\lambda_2(t+2)} \theta_r^{(k)}$\\
       $\theta^{(k+1)}_r \leftarrow \theta^{(k)}_r - \eta_t \nabla_\theta \overline{g}^{(t)}(\theta^{(k)}_r) + \sqrt{2\eta_t}\zeta_r^{(k)}$~(i.i.d.~Gaussian noise~$\zeta_r^{(k)} \sim \cN(0,I_p)$)
     \ENDFOR
\vspace{1mm}
   \STATE $\tilde{\Theta}^{(t+1)} \leftarrow \Theta^{(T_t+1)} = \{\theta_r^{(T_t+1)}\}_{r=1}^M$
   \ENDFOR
\vspace{1mm}   
\STATE Randomly pick up $t \in \{2,3,\ldots,T+1\}$ following the probability $\bP[t]=\frac{2t}{T(T+3)}$ and return $h_{\tilde{\Theta}^{(t)}}$
\end{algorithmic}
\end{algorithm}

We also extend the mean field limit (Algorithm~\ref{alg:da}) to take into account the inexactness in computing $h_{q^{(t)}}(t)$. 
This relaxation is useful in convergence analysis of Algorithm \ref{alg:pda2} with resampling because it allows us to regard this method as an instance of the generalized DA method (Algorithm \ref{alg:da2}) by setting an inexact estimate $h_x^{(t)} = h_{\tilde{\Theta}^{(t)}}(x_t)$, instead of the exact value of $h_{q^{(t)}}(t)$, which is actually used to defined the potential for which Langevin algorithm run in Algorithm \ref{alg:pda2}. This means convergence analysis of Algorithm \ref{alg:da2} (Theorem \ref{thm:convergence2}) immediately provides a convergence guarantee for Algorithm \ref{alg:pda2} if the discretization error $| h_x^{(t)} - h_{q^{(t)}}(x_t)|$ can be estimated (as in the resampling scheme).

On the other hands, the convergence analysis of warm-start scheme requires the convergence of mean field limit due to certain technical difficulties, that is, we show the convergence of Algorithm \ref{alg:pda2} with warm-start by coupling the update with its mean field limit (Algorithm \ref{alg:da}) and taking into account the discretization error which stems from finite-particle approximation.

\begin{algorithm}[ht]
  \caption{Dual Averaging (\textit{general version})}
  \label{alg:da2}
\begin{algorithmic}
  \STATE {\bfseries Input:}
  data distribution $\cD$ and
  initial density $q^{(1)}$
\vspace{1mm}
   \FOR{$t=1$ {\bfseries to} $T$}
   \STATE Randomly draw a data $(x_t,y_t)$ from $\cD$ \\
   \STATE \textcolor{blue}{Compute an approximation $h^{(t)}_x$ of $h_{q^{(t)}}(x_{t})$} \\
   \textcolor{blue}{$g^{(t)} \leftarrow \partial_z\ell( h^{(t)}_x,y_{t}) h(\cdot,x_{t}) + \lambda_1\|\cdot\|_2^2$} \\
   Obtain an approximation $q^{(t+1)}$ of the density function $q^{(t+1)}_* \propto \exp\left( - \frac{\sum_{s=1}^t 2s g^{(s)}}{\lambda_2 (t+2)(t+1)}\right)$
   \ENDFOR
\vspace{1mm}   
\STATE Randomly pick up $t \in \{2,3,\ldots,T+1\}$ following the probability $\bP[t]=\frac{2t}{T(T+3)}$ and return $h_{q^{(t)}}$
\end{algorithmic}
\end{algorithm}

We now present generalized version of the outer loop convergence rate of DA. 
We highlight the tolerance factor $\epsilon$ in the generalized assumption \textbf{(A3')} in blue.
\begin{assumption} \textcolor{blue}{Let $\epsilon > 0$ be a given accuracy. }
\begin{description}[topsep=0mm,itemsep=0mm] \label{assump:convergence2}
\item{{\bf(A1')}} $\cY \subset [-1,1]$. $\ell(z,y)$ is a smooth convex function w.r.t.~$z$ and $|\partial_z \ell(z,y)| \leq 2$ for $y,z \in \cY$ \textcolor{blue}{and $\partial \ell(\cdot,y)$ is $1$-Lipschitz continuous for $y \in \cY$}. 
\item{{\bf(A2')}} $|h_\theta(x)| \leq 1$ and $h(\theta,x)$ is smooth w.r.t.~$\theta$ for $x \in \cX$.
\item{{\bf(A3')}} $\KL(q^{(t+1)}\| q^{(t+1)}_*) \leq 1 / t^2$ and \textcolor{blue}{$| h^{(t)}_{x} - h_{q^{(t)}}(x_{t}) | \leq \epsilon$} for $t\geq 1$.
\end{description}
\end{assumption} 

\vspace{-2.5mm}
\paragraph{Remark.}
The new condition of {\bf (A3')} allows for inexactness of computing $h_{q^{(t)}}(x_t)$. 
When showing solely the convergence of the Algorithm \ref{alg:da} which is the exact mean-field limit, the original assumptions {\bf(A1)}, {\bf(A2)}, and {\bf(A3)} are sufficient, in other words, we can take $\epsilon=0$ and Lipschitz continuity of $\partial_z \ell(\cdot,y)$ in {\bf(A1')} can be relaxed.

\begin{theorem}[Convergence of general DA]\label{thm:convergence2}
Under Assumptions {\bf(A1')}, {\bf(A2')}, and {\bf(A3')} \textcolor{blue}{with $\epsilon \geq 0$}, for arbitrary $q_* \in \cP_2$, iterates of the general DA method (Algorithm \ref{alg:da2}) satisfies
\begin{align*}
&\hspace{-7mm}\frac{2}{T(T+3)} \sum_{t=2}^{T+1}t \left( \bE[\cL(q^{(t)})] - \cL(q_*) \right) \\
&\leq 
\textcolor{blue}{2\epsilon+}
O\Bigl( \frac{1}{T^2}\left(1+\lambda_1 \bE_{q_*}\left[ \|\theta\|_2^2\right] \right)
+ \frac{ \lambda_2  e(q_*) }{T} + \frac{ \lambda_2 }{T}(1 + \exp(8/\lambda_2))p^2 \log^2(T+2) \Bigr),
\end{align*}
where the expectation $\bE[\cL(q^{(t)})]$ is taken with respect to the history of examples.
\end{theorem}

\paragraph{Notation.} In the proofs, we use the following notations which are consistent with the description of Algorithm \ref{alg:pda2} and \ref{alg:da2}:
\begin{align*}
g^{(t)} 
&= \partial_z\ell( h^{(t)}_x,y_{t}) h(\cdot,x_{t}) + \lambda_1\|\cdot\|_2^2, \\
\overline{g}^{(t)}
&= \frac{2}{\lambda_2(t+2)(t+1)}\sum_{s=1}^t sg^{(s)} \\
&= \frac{2}{\lambda_2(t+2)(t+1)}\sum_{s=1}^t s \partial_z\ell( h_x^{(s)},y_{s}) h(\cdot,x_{s}) + \frac{\lambda_1 t}{\lambda_2(t+2)} \|\cdot|_2^2, \\
q^{(t+1)}_* &\propto \exp\left( - \overline{g}^{(t)}\right) \\
&= \exp\left( - \frac{\sum_{s=1}^t 2s g^{(s)}}{\lambda_2 (t+2)(t+1)}\right).
\end{align*}
When considering the resampling scheme, $h^{(t)}_x$ is set to the approximation $h_{\tilde{\Theta}^{(t)}}(x_t)$, whereas when considering the warm-start scheme, $h^{(t)}_x$ is set to $h_{q^{(t)}}(x_t)$ with the mean field limit $M\rightarrow \infty$ and without tolerance ($\epsilon=0$).

\subsection{Auxiliary Lemmas}
We introduce several auxiliary results used in the proof of Theorem \ref{thm:convergence} (Theorem~\ref{thm:convergence2}) and Corollary \ref{cor:inner-complexity}.  
The following lemma provides a tail bound for Chi-squared variables~\citep{laurent2000adaptive}.
\begin{lemma}[Tail bound for Chi-squared variable]\label{lem:chi_tail}
Let $\theta \sim \cN(0,\sigma^2 I_p)$ be a Gaussian random variable on $\bR^p$.
Then, we get for $\forall c \geq p \sigma^2$,
\[ \bP\left[\|\theta\|_2^2 \geq 2c \right] \leq \exp\left( -\frac{c}{10\sigma^2}\right). \]
\end{lemma}

Based on Lemma \ref{lem:chi_tail}, we get the following bound.
\begin{lemma}\label{lem:tail_expectation}
Let $\theta \sim \cN(0,\sigma^2 I_p)$ be Gaussian random variable on $\Theta = \bR^p$.
Then, we get for $\forall R \geq p \sigma^2$, 
\begin{equation*}
\bE\left[ \|\theta\|_2^2 \bone[\|\theta\|_2^2 > 2R] \right]=
\frac{1}{Z}\int_{\|\theta\|_2^2 > 2R} \|\theta\|_2^2 \exp\left(- \frac{\|\theta\|_2^2}{2\sigma^2} \right)\rd\theta 
\leq 2(R + 10\sigma^2)\exp\left( -\frac{R}{10\sigma^2}\right),
\end{equation*}
where $Z=\int \exp\left(-\frac{\|\theta\|_2^2}{2\sigma^2} \right) \rd \theta$.
\end{lemma}
\begin{proof}
We set $p(\theta) = \exp(-\|\theta\|_2^2/2\sigma^2)/Z$.
Then, 
\begin{align*}
\int_{\|\theta\|_2^2 > 2R} \|\theta\|_2^2 p(\theta)\rd \theta 
&= \int_\Theta p(\theta)\bone[\|\theta\|_2^2 > 2R]  \int_{0}^\infty \bone[\|\theta\|_2^2 > r] \rd r \rd \theta \\
&= \int_\Theta \int_{0}^\infty p(\theta)\bone \left[\|\theta\|_2^2 > \max\{2R,r\} \right]  \rd r \rd \theta \\
&\leq 2R \int_\Theta p(\theta)\bone \left[\|\theta\|_2^2 > 2R\right] \rd \theta
+ \int_\Theta \int_{2R}^\infty p(\theta)\bone \left[\|\theta\|_2^2 > r \right]  \rd r \rd \theta \\
&= 2R \bP[ \|\theta\|_2^2 > 2R ] + \int_{2R}^\infty \bP[ \|\theta\|_2^2 > r ] \rd r \\
&\leq 2R \exp\left( -\frac{R}{10\sigma^2}\right) 
+ \int_{2R}^\infty \exp\left( -\frac{r}{20\sigma^2}\right)  \rd r \\
&\le 2(R + 10\sigma^2)\exp\left( -\frac{R}{10\sigma^2}\right).
\end{align*}
\end{proof}

\begin{proposition}[Continuity]\label{prop:continuity_entropy}
Let $q_*(\theta) \propto \exp\left(-H(\theta)-\lambda\|\theta\|_2^2\right)$ $(\lambda>0)$ be a density on $\bR^p$ such that $\|H\|_\infty \leq c$.
Then, for $\forall \delta >0$ and a density $\forall q \in \cP_2$, 
\begin{align*}
\left| \int \|\theta\|_2^2 (q-q_*)(\theta)\rd\theta \right|
&\leq \frac{(2+\delta + 1/\delta)\exp(4c)}{\lambda}\KL(q\|q_*) + \frac{\delta(1+\delta)p\exp(2c)}{2\lambda}, \\
\left| \int q(\theta)\log(q(\theta)) \rd \theta -  \int q_*(\theta)\log(q_*(\theta)) \rd \theta  \right| 
&\leq \left( 1 + (2+\delta + 1/\delta) \exp(4c) \right) \KL(q\|q_*) + c\sqrt{2\KL(q\|q_*)} \\
&+ \frac{\delta(1+\delta)p\exp(2c)}{2}.
\end{align*}
\end{proposition}

\begin{proof}
Let $\gamma$ be an optimal coupling between $q\rd\theta$ and $q_*\rd\theta$.
Using Young's inequality, we have
\begin{align}
\int \|\theta\|_2^2 q(\theta)\rd\theta 
&= \int \|\theta\|_2^2 \rd\gamma(\theta,\theta') \notag\\
&= \int \left( \|\theta - \theta'\|_2^2 + \|\theta'\|_2^2 + 2(\theta-\theta')^\top \theta' \right) \rd\gamma(\theta,\theta') \notag\\
&\leq \int \left( \|\theta - \theta'\|_2^2 + \|\theta'\|_2^2 
+ \frac{1}{\delta}\|\theta-\theta'\|_2^2 + \delta \| \theta' \|_2^2 \right) \rd\gamma(\theta,\theta') \notag\\
&= (1+1/\delta) \int  \|\theta - \theta'\|_2^2 \rd\gamma(\theta,\theta') 
+ (1+\delta) \int  \|\theta'\|_2^2 q_*(\theta')\rd\theta' \notag\\
&= (1+1/\delta) W_2^2(q,q_*) + (1+\delta) \int  \|\theta'\|_2^2 q_*(\theta')\rd\theta'. \label{eq:second-moment_bound}
\end{align}

The last term can be bounded as follows:
\begin{align}
\int  \|\theta\|_2^2 q_*(\theta)\rd\theta 
&= \int  \|\theta\|_2^2 \frac{\exp\left(-H(\theta)-\lambda\|\theta\|_2^2\right)}{ \int \exp\left(-H(\theta)-\lambda\|\theta\|_2^2\right) \rd\theta} \rd\theta \notag\\
&\leq \exp(2c)\int  \|\theta\|_2^2 \frac{\exp\left(-\lambda\|\theta\|_2^2\right)}{ \int \exp\left(-\lambda\|\theta\|_2^2\right) \rd\theta} \rd\theta \notag\\
&=\frac{p\exp(2c)}{2\lambda}, \label{eq:second-moment_bound2}
\end{align}
where the last equality comes from the variance of Gaussian distribution.

From (\ref{eq:second-moment_bound}) and (\ref{eq:second-moment_bound2}),
\begin{align*}
\int \|\theta\|_2^2 (q-q_*)(\theta)\rd\theta
&\leq (1+1/\delta) W_2^2(q,q_*) + \delta\int  \|\theta\|_2^2 q_*(\theta)\rd\theta \\
&\leq (1+1/\delta) W_2^2(q,q_*) + \frac{\delta p\exp(2c)}{2\lambda}.
\end{align*}
From the symmetry of (\ref{eq:second-moment_bound}), and applying (\ref{eq:second-moment_bound}) again with (\ref{eq:second-moment_bound2}),
\begin{align*}
\int \|\theta\|_2^2 (q_*-q)(\theta)\rd\theta
&\leq (1+1/\delta) W_2^2(q,q_*) + \delta\int  \|\theta\|_2^2 q(\theta)\rd\theta \\
&\leq (2+\delta + 1/\delta) W_2^2(q,q_*) + \delta(1+\delta) \int  \|\theta\|_2^2 q_*(\theta)\rd\theta \\
&\leq (2+\delta + 1/\delta) W_2^2(q,q_*) + \frac{\delta(1+\delta)p\exp(2c)}{2\lambda}. 
\end{align*}

From Lemma \ref{lem:sobolev_perturbation} and Example \ref{ex:quadratic_sobolev}, we see 
$q_*$ satisfies the log-Sobolev inequality with a constant $2\lambda/\exp(4c)$.
As a result, $q_*$ satisfies Talagrand's inequality with the same constant from Theorem \ref{thm:otto-villani}.
Hence, by combining the above two inequalities, we have
\begin{align*}
\left| \int \|\theta\|_2^2 (q-q_*)(\theta)\rd\theta \right|
&\leq (2+\delta + 1/\delta) W_2^2(q,q_*) + \frac{\delta(1+\delta)p\exp(2c)}{2\lambda} \\
&\leq \frac{(2+\delta + 1/\delta)\exp(4c)}{\lambda}\KL(q\|q_*) + \frac{\delta(1+\delta)p\exp(2c)}{2\lambda}
\end{align*}

Therefore, we know that 
\begin{align*}
&\left| \int q(\theta)\log(q(\theta))\rd\theta - \int q_*(\theta)\log(q_*(\theta))\rd\theta \right| \\
&\leq \KL(q\|q_*) + \left| \int  (q_*-q)(\theta)\left(H(\theta) + \lambda\|\theta\|_2^2\right) \rd\theta \right|\\
&\leq \KL(q\|q_*) + c\|q-q_*\|_{L_1(\rd\theta)} + (2+\delta + 1/\delta)\exp(4c)\KL(q\|q_*) + \frac{\delta(1+\delta)p\exp(2c)}{2} \\
&\leq \KL(q\|q_*) + c\sqrt{2\KL(q\|q_*)} +  (2+\delta + 1/\delta)\exp(4c)\KL(q\|q_*) + \frac{\delta(1+\delta)p\exp(2c)}{2}.
\end{align*}
where we used Pinsker's theorem for the last inequality.
This finishes the proof.
\end{proof}

\begin{proposition}[Maximum Entropy]\label{prop:maximum_entropy}
Let $q_*(\theta) \propto \exp\left(-H(\theta)-\lambda\|\theta\|_2^2\right)$ $(\lambda>0)$ on $\bR^p$ be a density such that $\|H\|_\infty \leq c$.
Then, 
\begin{align*}
-\bE_{q_*}[ \log(q_*)] \leq 2c + \frac{p}{2}\left(\exp(2c) + \log\left(\frac{\pi}{\lambda}\right)\right).
\end{align*}
\end{proposition}
\begin{proof}
It follows that 
\begin{align*}
- \bE_{q_*}[ \log(q_*)] 
&= \bE_{q_*}[ H(\theta) + \lambda\|\theta\|_2^2 ] + \log \int \exp( -H(\theta)-\lambda\|\theta\|_2^2 ) \rd\theta \\
&\leq c + \lambda\bE_{q_*}[\|\theta\|_2^2 ] + \log \int \exp( c-\lambda\|\theta\|_2^2 ) \rd\theta \\
&= 2c + \lambda\bE_{q_*}[\|\theta\|_2^2 ] + \log \int \exp( -\lambda\|\theta\|_2^2 ) \rd\theta \\
&\leq 2c + \frac{p\exp(2c)}{2} + \frac{p}{2}\log\left(\frac{\pi}{\lambda}\right),
\end{align*}
where we used (\ref{eq:second-moment_bound2}) and Gaussian integral for the last inequality.
\end{proof}

\begin{proposition}[Boundedness of KL-divergence]\label{prop:boundedness-of-KL}
Let $q_*(\theta) \propto \exp\left(-H_*(\theta)-\lambda_*\|\theta\|_2^2\right)$ $(\lambda_*>0)$ be a density on $\bR^p$ such that $\|H_*\|_\infty \leq c_*$, 
and $q_{\sharp}(\theta) \propto \exp\left(-H_{\sharp}(\theta)-\lambda_{\sharp}\|\theta\|_2^2\right)$ $(\lambda_{\sharp}>0)$ be a density on $\bR^p$ such that $\|H_{\sharp}\|_\infty \leq c_{\sharp}$.
Then, for any density $q$,
\begin{align*}
\KL(q\|q_*) 
&\leq 4c_* + 2c_{\sharp} 
+ \frac{3}{2}\left( 1 + \frac{\lambda_*}{\lambda_{\sharp}} \right) p\exp(2c_{\sharp})
+ \frac{p}{2} \log\left(\frac{\lambda_{\sharp}}{\lambda_*}\right) \\
&+ \left( 1+ 4 \left( 1 + \frac{\lambda_*}{\lambda_{\sharp}} \right) \exp(4c_{\sharp}) \right)\KL(q\|q_{\sharp}) + c_{\sharp}\sqrt{2\KL(q\|q_{\sharp})}.
\end{align*}
\end{proposition}
\begin{proof}
Applying Proposition \ref{prop:continuity_entropy} with $\delta=1$,
\begin{align*}
\KL(q\|q_*) 
&= \int q(\theta)\log\left( \frac{q(\theta)}{q_*(\theta)}\right) \rd \theta \\
&= \int q_{\sharp}(\theta)\log\left( \frac{q_{\sharp}(\theta)}{q_*(\theta)}\right) \rd \theta 
+ \int (q_{\sharp}(\theta)-q(\theta))\log(q_*(\theta))\rd \theta \\
&+ \int q(\theta)\log( q(\theta) ) \rd \theta - \int q_{\sharp}(\theta)\log( q_{\sharp}(\theta) ) \rd \theta\\
&\leq \int q_{\sharp}(\theta)\log\left( \frac{q_{\sharp}(\theta)}{q_*(\theta)}\right) \rd \theta 
+ \int (q(\theta)-q_{\sharp}(\theta)) (H_*(\theta) + \lambda_* \|\theta\|_2^2) \rd \theta \\
&+ (1+4\exp(4c_{\sharp}))\KL(q\|q_{\sharp}) + c_{\sharp}\sqrt{2\KL(q\|q_{\sharp})} + p\exp(2c_{\sharp}) \\
&\leq \int q_{\sharp}(\theta)\log\left( \frac{q_{\sharp}(\theta)}{q_*(\theta)}\right) \rd \theta 
+ 2c_* + \frac{4\lambda_*\exp(4c_{\sharp})}{\lambda_{\sharp}}\KL(q\|q_{\sharp}) + \frac{p\lambda_* \exp(2c_{\sharp})}{\lambda_{\sharp}} \\
&+ (1+4\exp(4c_{\sharp}))\KL(q\|q_{\sharp}) + c_{\sharp}\sqrt{2\KL(q\|q_{\sharp})} + p\exp(2c_{\sharp}).
\end{align*}
We next bound the first term in the last equation as follows.
\begin{align*}
\int q_{\sharp}(\theta)\log\left( \frac{q_{\sharp}(\theta)}{q_*(\theta)}\right) \rd \theta 
&= \int q_{\sharp}(\theta)\log\left(\frac{ \exp( -H_{\sharp}(\theta)-\lambda_{\sharp}\|\theta \|_2^2 ) }{ \exp( -H_*(\theta)-\lambda_*\|\theta \|_2^2 ) }\right) \rd \theta 
+ \log \frac{ \int \exp( -H_*(\theta)-\lambda_*\|\theta \|_2^2 ) \rd\theta }{ \int \exp( -H_{\sharp}(\theta)-\lambda_{\sharp}\|\theta \|_2^2 ) \rd\theta } \\
&= \int q_{\sharp}(\theta) \left( H_*(\theta) - H_{\sharp}(\theta) + (\lambda_* - \lambda_{\sharp})\|\theta\|_2^2\right) \rd \theta \\
&+ \log \int \exp( -H_*(\theta)-\lambda_*\|\theta \|_2^2 ) \rd\theta 
- \log \int \exp( -H_{\sharp}(\theta)-\lambda_{\sharp}\|\theta \|_2^2 ) \rd\theta \\
&\leq c_* + c_{\sharp} 
+ \frac{1}{2}\left(1 + \frac{\lambda_*}{\lambda_{\sharp}}\right) p\exp(2c_{\sharp})\\
&+ \log \int \exp( c_*-\lambda_*\|\theta \|_2^2 ) \rd\theta 
- \log \int \exp( -c_{\sharp}-\lambda_{\sharp}\|\theta \|_2^2 ) \rd\theta \\
&\leq 2c_* + 2c_{\sharp}
+ \frac{1}{2}\left(1 + \frac{\lambda_*}{\lambda_{\sharp}}\right) p\exp(2c_{\sharp})
+ \frac{p}{2} \log\left(\frac{\lambda_{\sharp}}{\lambda_*}\right),
\end{align*}
where for the first inequality we used a similar inequality as in (\ref{eq:second-moment_bound2}) and for the second inequality we used the Gaussian integral.
Hence, we get
\begin{align*}
\KL(q\|q_*) 
&\leq  4c_* + 2c_{\sharp} 
+ \frac{3}{2}\left( 1 + \frac{\lambda_*}{\lambda_{\sharp}} \right) p\exp(2c_{\sharp})
+ \frac{p}{2} \log\left(\frac{\lambda_{\sharp}}{\lambda_*}\right) \\
&+ \left( 1+ 4 \left( 1 + \frac{\lambda_*}{\lambda_{\sharp}} \right) \exp(4c_{\sharp}) \right)\KL(q\|q_{\sharp}) + c_{\sharp}\sqrt{2\KL(q\|q_{\sharp})}.
\end{align*}
\end{proof}

\begin{lemma}\label{lem:acc_subprob}
Suppose Assumption {\bf (A1')} and {\bf (A2')} hold.
If $\KL(q^{(t)}\|q^{(t)}_*) \leq \frac{1}{t^2}$ for $t\geq 2$, then
\[ t \left| \int  g^{(t)}(\theta) (q^{(t)}(\theta) - q^{(t)}_* (\theta)) \rd \theta \right|
,~\lambda_2 t\left|e(q^{(t)}) - e(q^{(t)}_*) \right| 
= O\left(1+\lambda_2 + p \lambda_2 \exp(8/\lambda_2) \right).\]
\end{lemma}
\begin{proof}
Recall the definition of $g^{(t)}, \overline{g}^{(t)}$ and $q^{(t)}_*$ (see notations in subsection \ref{subsec:extension}).
We set $\gamma_{t+1} = \frac{\sum_{s=1}^{t} s}{\lambda_2 \sum_{s=1}^{t+1} s} = \frac{t}{\lambda_2(t+2)}$.
Note that for $t\geq 1$,
\begin{align}
-2 + \lambda_1 \|\theta\|_2^2 &\leq g^{(t)}(\theta) \leq 2 + \lambda_1 \|\theta\|_2^2, \label{eq:misc1}\\
\gamma_{t+1}(-2 + \lambda_1 \|\theta\|_2^2) 
&\leq \overline{g}^{(t)}(\theta)
\leq \gamma_{t+1}(2 + \lambda_1 \|\theta\|_2^2), \label{eq:misc2}\\
\frac{1}{3\lambda_2} &\leq \gamma_{t+1} \leq \frac{1}{\lambda_2}. \label{eq:misc3}
\end{align}

Therefore, we have for $t\geq 2$ from Proposition \ref{prop:continuity_entropy} with $\delta = 1/t < 1$,
\begin{align*}
&t\left| \int g^{(t)}(\theta) (q^{(t)}(\theta) - q^{(t)}_* (\theta)) \rd \theta \right| \\
\leq& 2t\|q^{(t)}-q^{(t)}_*\|_{L_1(\rd\theta)} 
+ \lambda_1 t\left| \int \|\theta\|_2^2 (q^{(t)}(\theta) - q^{(t)}_* (\theta)) \rd \theta \right| \\
\leq& 2t \sqrt{2 \KL(q^{(t)} \| q^{(t)}_*)} \\
&+ \lambda_2(t+2)
\left( (2+\delta + 1/\delta)\exp(8/\lambda_2)\KL(q^{(t)}\|q^{(t)}_*) + \frac{\delta(1+\delta)p\exp(4/\lambda_2)}{2} \right) \\
\leq& 2\sqrt{2} + 3 \lambda_2 \left( 4\exp(8/\lambda_2) + p\exp(4/\lambda_2) \right) \\
=& O\left( 1 + p\lambda_2\exp(8/\lambda_2)\right).
\end{align*}
 
Moreover, we have for $t \geq 2$,
\begin{align*}
&\lambda_2 t \left|e(q^{(t)}) - e(q^{(t)}_*) \right|  \\
\leq& \lambda_2 t \left( \left( 1 + (2+\delta + 1/\delta) \exp(8/\lambda_2) \right) \KL(q^{(t)}\|q^{(t)}_*) + \frac{2}{\lambda_2}\sqrt{2\KL(q^{(t)}\|q^{(t)}_*)} 
+ \frac{\delta(1+\delta)p\exp(4/\lambda_2)}{2} \right) \\
\leq& \lambda_2 t \left( \left( 1 + (3 + t) \exp(8/\lambda_2) \right) \frac{1}{(t-1)^2} + \frac{2\sqrt{2}}{\lambda_2 (t-1)}
+ \frac{p\exp(4/\lambda_2)}{t} \right) \\
=& O\left( 1 + \lambda_2 + p \lambda_2 \exp(8/\lambda_2) \right).
\end{align*}
This finishes the proof.
\end{proof}
 
\subsection{Outer Loop Complexity}
Based on the auxiliary results and the convex optimization theory developed in \citet{nesterov2009primal,xiao2009dual}, we now prove Theorem \ref{thm:convergence2} which is an extension of Theorem \ref{thm:convergence}.
\begin{proof}[Proof of Theorem \ref{thm:convergence2}]
For $t \geq 1$ we define, 
\begin{align*}
V_t(q) = -\bE_{q}\left[ \sum_{s=1}^{t} s g^{(s)}\right] - \lambda_2 e(q) \sum_{s=1}^{t+1} s.
\end{align*}
From the definition, the density $q^{(t+1)}_* \in \cP_2$ calculated in Algorithm \ref{alg:da2} maximizes $V_t(q)$.
We denote $V_t^* = V(q^{(t+1)}_*)$.
Then, for $t \geq 2$, we get
\begin{align}
V_t^* &=  -\bE_{q^{(t+1)}_*}\left[ \sum_{s=1}^{t-1} s g^{(s)}\right] -  \lambda_2 e(q^{(t+1)}_*) \sum_{s=1}^{t} s
- \bE_{q^{(t+1)}_*}\left[ t g^{(t)}\right] -  \lambda_2 (t+1) e(q^{(t+1)}_*) \notag\\
&\leq V_{t-1}^* - \frac{\lambda_2 \sum_{s=1}^t s}{2} \|q^{(t+1)}_* - q^{(t)}_*\|_{L_1(\rd\theta)}^2 
- \bE_{q^{(t)}_*}\left[ t g^{(t)}\right] -  \lambda_2 (t+1) e(q^{(t+1)}_*) \notag\\
&+ t \int (q^{(t)}_* - q^{(t+1)}_*)(\theta) g^{(t)}(\theta)\rd \theta \notag\\
&\leq V_{t-1}^* - \frac{\lambda_2 \sum_{s=1}^t s}{2} \|q^{(t+1)}_* - q^{(t)}_*\|_{L_1(\rd\theta)}^2 
- \bE_{q^{(t)}}\left[ t g^{(t)}\right] -  \lambda_2 (t+1) e(q^{(t+1)}_*) \notag\\
&+ t \left| \int  g^{(t)}(\theta) (q^{(t)}(\theta) - q^{(t)}_* (\theta)) \rd \theta \right|
+ t \int (q^{(t)}_* - q^{(t+1)}_*)(\theta) g^{(t)}(\theta)\rd \theta \notag\\
&\leq V_{t-1}^* - \frac{\lambda_2 \sum_{s=1}^t s}{2} \|q^{(t+1)}_* - q^{(t)}_*\|_{L_1(\rd\theta)}^2 
- \bE_{q^{(t)}}\left[ t g^{(t)}\right] -  \lambda_2 (t+1) e(q^{(t+1)}_*) \notag\\
&+ t \int (q^{(t)}_* - q^{(t+1)}_*)(\theta) g^{(t)}(\theta)\rd \theta + O(1+\lambda_2 + p\lambda_2\exp(8/\lambda_2) ), \label{eq:v_bound}
\end{align}
where for the first inequality we used the optimality of $q^{(t)}_*$ and the strong convexity (\ref{eq:st-conv}) at $q^{(t)}_*$, and for the final inequality we used Lemma \ref{lem:acc_subprob}.

We set $R_t = \left(\frac{3}{2}p+15\right)\frac{\lambda_2}{\lambda_1}\log(1+t)$ and also $\gamma_{t+1} = \frac{\sum_{s=1}^{t} s}{\lambda_2 \sum_{s=1}^{t+1} s} = \frac{t}{\lambda_2(t+2)}$, as done in the proof of Lemma \ref{lem:acc_subprob}.

From Assumptions {\bf (A1')}, {\bf (A2')} and  $q^{(t)}_* = \exp\left( - \frac{\sum_{s=1}^{t-1} s g^{(s)}}{\lambda_2\sum_{s=1}^{t} s}\right)/ \int  \exp\left( - \frac{\sum_{s=1}^{t-1} s g^{(s)}(\theta)}{\lambda_2\sum_{s=1}^{t} s}\right) \rd\theta$ $(t\geq 2)$, 
we have for $t \geq 2$,
\begin{align}
q^{(t)}_*(\theta) 
&\leq  \exp( \gamma_t(2 - \lambda_1 \|\theta\|_2^2)) / \int  \exp( \gamma_t(-2-\lambda_1\|\theta\|_2^2)) \rd\theta \notag\\
&\leq \exp( 4\gamma_t ) \exp( - \gamma_t\lambda_1 \|\theta\|_2^2) / \int  \exp( -\gamma_t\lambda_1\|\theta\|_2^2) \rd\theta \notag\\
&\leq \exp( 4/\lambda_2 ) \exp( - \gamma_t\lambda_1 \|\theta\|_2^2) / \int  \exp( -\gamma_t\lambda_1\|\theta\|_2^2) \rd\theta. \label{eq:q_bound}
\end{align}
Using (\ref{eq:q_bound}) and applying Lemma \ref{lem:tail_expectation} with $\sigma^2 = \frac{1}{2\gamma_t \lambda_1},~\frac{1}{2\gamma_{t+1}\lambda_1}$ and $R=R_t$, we have for $t\geq 2$,
\\
\begin{align*}
&~~~\left|\int (q^{(t)}_* - q^{(t+1)}_*)(\theta) g^{(t)}(\theta)\rd \theta \right| \\ 
&\leq 2\|q^{(t)}_* - q^{(t+1)}_*\|_{L_1(\rd \theta)} +
\lambda_1\int \|\theta\|_2^2 |(q^{(t)}_* - q^{(t+1)}_*)(\theta) | \rd \theta \\
&\leq (2+ 2\lambda_1R_t)\|q^{(t)}_* - q^{(t+1)}_*\|_{L_1(\rd \theta)}
+ \lambda_1\int_{\|\theta\|_2^2 > 2R_t} \|\theta\|_2^2 (q^{(t)}_* + q^{(t+1)}_*)(\theta) \rd \theta \\
&\leq (2+ 2\lambda_1R_t)\|q^{(t)}_* - q^{(t+1)}_*\|_{L_1(\rd \theta)} + \lambda_1 \exp(4/\lambda_2)
\int_{\|\theta\|_2^2 >2R_t} \|\theta\|_2^2  \frac{\exp( - \gamma_t\lambda_1 \|\theta\|_2^2)}{\int  \exp( -\gamma_t\lambda_1\|\theta\|_2^2) \rd\theta} \rd \theta \\
&~~~+ \lambda_1 \exp(4/\lambda_2)
\int_{\|\theta\|_2^2 >2R_t} \|\theta\|_2^2  \frac{\exp( - \gamma_{t+1}\lambda_1 \|\theta\|_2^2)}{\int  \exp( -\gamma_{t+1}\lambda_1\|\theta\|_2^2) \rd\theta} \rd \theta \\
&\leq (2+ 2\lambda_1R_t)\|q^{(t)}_* - q^{(t+1)}_*\|_{L_1(\rd \theta)} + 2\lambda_1 \exp(4/\lambda_2) \left(R_t + \frac{5}{\lambda_1\gamma_t}\right)\exp\left( - \frac{\lambda_1 R_t \gamma_t}{5} \right) \\
&~~~+ 2\lambda_1 \exp(4/\lambda_2)  \left(R_t + \frac{5}{\lambda_1\gamma_t}\right)\exp\left( - \frac{\lambda_1 R_t \gamma_{t+1}}{5} \right) \\
&\leq  (2+ 2\lambda_1R_t)\|q^{(t)}_* - q^{(t+1)}_*\|_{L_1(\rd \theta)} + 4\lambda_1 \exp(4/\lambda_2)\left(R_t + 15\frac{\lambda_2}{\lambda_1}\right)\exp\left( - \frac{R_t\lambda_1 }{15\lambda_2} \right) \\
&\leq  \left(2+ 2\left(\frac{3}{2}p+15\right)\lambda_2\log(1+t) \right)\|q^{(t)}_* - q^{(t+1)}_*\|_{L_1(\rd \theta)} +8 \exp(4/\lambda_2) \left(\frac{3}{2}p+15\right) \frac{\lambda_2\log(1+t)}{(1+t)^{1+\frac{p}{10}}},
\end{align*}
where for the fifth inequality we used (\ref{eq:misc3}) and for the sixth inequality we used $15\lambda_2/\lambda_1 \leq R_t$.

Applying Young's inequality $ab \leq \frac{a^2}{2\delta} + \frac{\delta b^2}{2}$ with $a = \left(2+ 2\left(\frac{3}{2}p+15\right)\lambda_2\log(1+t) \right)$, $b=\|q^{(t)}_* - q^{(t+1)}_*\|_{L_1(\rd \theta)}$, and $\delta=\frac{\lambda_2}{2}(t+1)$, we get

\begin{align}
\left| \int (q^{(t)}_* - q^{(t+1)}_*)(\theta) g^{(t)}(\theta)\rd \theta \right|
&\leq \frac{ \left(2+ 2\left(\frac{3}{2}p+15\right)\lambda_2\log(1+t) \right)^2 }{ \lambda_2(t+1) }
+ \frac{ \lambda_2(t+1)\|q^{(t)}_* - q^{(t+1)}_*\|_{L_1(\rd \theta)}^2}{4} \notag\\
&+ 8 \exp(4/\lambda_2) \left(\frac{3}{2}p+15\right) \frac{\lambda_2\log(1+t)}{(1+t)^{1+\frac{p}{10}}}. \label{eq:v_bound2}
\end{align}

Combining (\ref{eq:v_bound}) and (\ref{eq:v_bound2}), we have for $t\geq 2$,
\begin{align}
V_t^* 
&\leq V_{t-1}^* - \bE_{q^{(t)}}\left[ t g^{(t)}\right] -  \lambda_2 (t+1) e(q^{(t+1)}_*) + O(1+\lambda_2 + p\lambda_2\exp(8/\lambda_2)) \notag\\
&+ \frac{1}{\lambda_2} \left(2+ 2\left(\frac{3}{2}p+15\right)\lambda_2 \log(1+t) \right)^2 
+ 8 \exp(4/\lambda_2) \left(\frac{3}{2}p+15\right) \frac{\lambda_2\log(1+t)}{(1+t)^{\frac{p}{10}}} \notag\\
&= V_{t-1}^* - \bE_{q^{(t)}}\left[ t g^{(t)}\right] -  \lambda_2 (t+1) e(q^{(t+1)}_*) + O(1+\lambda_2 + p\lambda_2\exp(8/\lambda_2)) \notag\\
&+ O\left( \frac{1}{\lambda_2} + p^2\lambda_2 \log^2(1+t) + p \lambda_2 \exp(4/\lambda_2) \right) \notag\\
&= V_{t-1}^* - \bE_{q^{(t)}}\left[ t g^{(t)}\right] -  \lambda_2 (t+1) e(q^{(t+1)}_*) 
+ O\left(p\lambda_2\exp(8/\lambda_2) + p^2\lambda_2 \log^2(1+t) \right) \notag\\
&= V_{t-1}^* - \bE_{q^{(t)}}\left[ t g^{(t)}\right] -  \lambda_2 (t+1) e(q^{(t+1)}_*) 
+ O\left( (1 + \exp(8/\lambda_2))p^2 \lambda_2 \log^2(1+t) \right) \notag \\
&= V_{t-1}^* - \bE_{q^{(t)}}\left[ t g^{(t)}\right] -  \lambda_2 (t+1) e(q^{(t+1)}_*) 
+ \alpha_t, \label{eq:v_recursive}
\end{align}
where we set $\alpha_t = O\left( (1 + \exp(8/\lambda_2))p^2 \lambda_2 \log^2(1+t) \right)$.

From Proposition \ref{prop:maximum_entropy}, (\ref{eq:misc2}), and (\ref{eq:misc3}),
\begin{equation*}
-\bE_{q^{(t)}_*}[\log(q^{(t)}_*)] \leq \frac{4}{\lambda_2} + \frac{p}{2}\left( \exp\left( \frac{4}{\lambda_2} \right) + \log\left(\frac{3\pi\lambda_2}{\lambda_1}\right)\right),
\end{equation*}
meaning $e(q^{(t)}_*) \geq 0$.
Hence, 
\[V_1^* = - \bE_{q^{(2)}_*}[g^{(1)}] - 3\lambda_2 e(q^{(2)}_*) 
\leq 2 - 3 \lambda_2 e(q^{(2)}_*)
\leq 2 - 2 \lambda_2 e(q^{(2)}_*). \]

Summing the inequality (\ref{eq:v_recursive}) over $t\in \{2,\ldots,T+1\}$,
\begin{align}
V_{T+1}^* 
&\leq 2 - 2\lambda_2 e(q^{(2)}_*)
+ \sum_{t=2}^{T+1}\left\{ - \bE_{q^{(t)}}\left[ t g^{(t)}\right] - \lambda_2 (t+1) e(q^{(t+1)}_*) 
+ \alpha_t \right\} \notag\\
&= 2
- \sum_{t=2}^{T+1}t \left\{ \bE_{q^{(t)}}\left[ g^{(t)}\right] + \lambda_2 e(q^{(t)}_*) \right\}
+ \sum_{t=2}^{T+1}\alpha_t
- \lambda_2 (T+2) e(q^{(T+2)}_*) \notag \\
&\leq 2
- \sum_{t=2}^{T+1}t \left\{ \bE_{q^{(t)}}\left[ g^{(t)}\right] + \lambda_2 e(q^{(t)}) \right\}
+ \sum_{t=2}^{T+1}\alpha_t, \label{eq:v_expand}
\end{align}
where we used $ \lambda_2 t \left| e(q^{(t)}) - e(q^{(t)}_*) \right | = \alpha_t$ (Lemma \ref{lem:acc_subprob}), $2\alpha_t = O(\alpha_t)$, and $e(q^{(T+2)}_*) \geq 0$.

On the other hand, for $\forall q_* \in \cP_2$,
\begin{align}
 V_{T+1}^* 
 = \max_{q \in \cP_2}\left\{  -\bE_{q}\left[ \sum_{t=1}^{T+1} t g^{(t)}\right] - \lambda_2 e(q) \sum_{t=1}^{T+2} t \right\} 
\geq  -\bE_{q_*}\left[ \sum_{t=1}^{T+1} t g^{(t)}\right] - \lambda_2 e(q_*) \sum_{t=1}^{T+2} t. \label{eq:v_bound3}
\end{align}

Using {\bf (A1')}, {\bf (A2')}, and {\bf (A3')}, we have for any density function $q$,
\begin{align} \label{eq:h_approx}
 \left| (\partial_z\ell(h_{q^{(t)}}(x_{t}),y_{t})
 - \partial_z\ell(h^{(t)}_{x},y_{t})) \bE_{q}[h(\cdot,x_{t})] \right| 
 \leq \epsilon.
\end{align}
Hence, from (\ref{eq:v_expand}), (\ref{eq:v_bound3}), (\ref{eq:h_approx}), and the convexity of the loss, 
\begin{align*}
\frac{2}{T(T+3)}& \sum_{t=2}^{T+1}t \Bigl\{ \ell(h_{q^{(t)}}(x_{t}),y_{t}) + \lambda_1 \bE_{q^{(t)}}[ \|\theta\|_2^2] + \lambda_2\bE_{q^{(t)}}[\log(q^{(t)})] \\
&\quad\quad -\ell(h_{q_*}(x_{t}),y_{t}) - \lambda_1 \bE_{q_*}[ \|\theta\|_2^2] - \lambda_2\bE_{q_*}[\log(q_*)] \Bigr\}  \\
&\leq \frac{2}{T(T+3)} \sum_{t=2}^{T+1}t \Bigl\{ 
\partial_z\ell(h_{q^{(t)}}(x_{t}),y_{t})\left(\bE_{q^{(t)}}[h(\cdot,x_{t})] - \bE_{q_*}[h(\cdot,x_{t})]\right) \\
&+ \lambda_1 \left(\bE_{q^{(t)}}[\|\theta\|_2^2] - \bE_{q_*}[\|\theta\|_2^2] \right)
+ \lambda_2 \left(\bE_{q^{(t)}}[\log(q^{(t)})] - \bE_{q_*}[\log(q_*)] \right) \Bigr\} \\
&\leq\frac{2}{T(T+3)} \sum_{t=2}^{T+1}t \left\{ 
2\epsilon + \bE_{q^{(t)}}[g^{(t)}] - \bE_{q_*}[g^{(t)}]
+ \lambda_2 \left( e(q^{(t)}) - e(q_*) \right) \right\} \\
&\leq 2\epsilon + \frac{2}{T(T+3)}\left( 2 -V_{T+1}^*+\sum_{t=2}^{T+1}\alpha_t  - \sum_{t=2}^{T+1}t \left(\bE_{q_*}[g^{(t)}]
+ \lambda_2 e(q_*))\right) \right) \\
&\leq 2\epsilon + \frac{2}{T(T+3)}\left( 2
+\bE_{q_*}\left[ g^{(1)}\right] + \lambda_2 (T+3) e(q_*) 
+\sum_{t=2}^{T+1}\alpha_t  \right) \\
&\leq 2\epsilon + \frac{2}{T(T+3)}\left(4+\lambda_1 \bE_{q_*}\left[ \|\theta\|_2^2\right] \right)
+ \frac{2 \lambda_2 e(q_*)  }{T}
+\frac{2}{T}  O\left( (1 + \exp(8/\lambda_2))p^2 \lambda_2 \log^2(T+2) \right).
\end{align*}

Taking the expectation with respect to the history of examples, we have
\begin{align*}
\frac{2}{T(T+3)} &\sum_{t=2}^{T+1}t \left( \bE[\cL(q^{(t)})] - \cL(q_*) \right) \\
&= 
2\epsilon +
O\left( \frac{1}{T^2}\left(1+\lambda_1 \bE_{q_*}\left[ \|\theta\|_2^2\right] \right)
+ \frac{ \lambda_2 }{T}\left( e(q_*) + (1 + \exp(8/\lambda_2))p^2 \log^2(T+2) \right) \right).
\end{align*}
\end{proof}

\subsection{Inner Loop Complexity}
We next prove Corollary \ref{cor:inner-complexity} which gives an estimate of inner loop iteration complexity. 
This result is derived by utilizing the convergence rate of the Langevin algorithm under LSI developed in \cite{vempala2019rapid}. 
We here consider the ideal Algorithm \ref{alg:da} (i.e., warm-start and exact mean field limit ($\epsilon=0$)). 

\begin{proof}[Proof of Corollary \ref{cor:inner-complexity}]
We verify the assumptions required in Theorem \ref{thm:vempala-wibisono}.
We recall that $q^{(t+1)}_*$ takes the form of Boltzmann distribution: for $t \geq 1$,
\begin{align*}
q^{(t+1)}_* &\propto \exp\left( - \frac{\sum_{s=1}^{t} s g^{(s)}}{\lambda_2\sum_{s=1}^{t+1} s}\right) \\
&= \exp\left( - \frac{1}{\lambda_2\sum_{s=1}^{t+1} s}\sum_{s=1}^{t} s\partial_z\ell( h^{(t)}_x,y_{t}) h(\cdot,x_{t})
- \frac{\lambda_1 t}{\lambda_2 (t+2)}\|\theta\|_2^2 \right).
\end{align*}
Note that $\frac{\lambda_1}{\lambda_2} \geq \frac{\lambda_1 t}{\lambda_2 (t+2)} \geq \frac{\lambda_1}{3\lambda_2}$ $(t\geq 1)$ and $\left|\frac{1}{\lambda_2\sum_{s=1}^{t+1} s}\sum_{s=1}^{t} s\partial_z\ell( h^{(t)}_x,y_{t}) h(\cdot,x_{t}) \right| \leq \frac{2t}{\lambda_2(t+2)} \leq \frac{2}{\lambda_2}$.
Therefore, from Example \ref{ex:quadratic_sobolev} and Lemma \ref{lem:sobolev_perturbation}, we know that $q^{(t+1)}_*$ satisfies the log-Sobolev inequality with a constant $\frac{2\lambda_1}{3\lambda_2 \exp(8/\lambda_2)}$; in addition, the gradient of $\log(q^{(t+1)}_*)$ is $\frac{2}{\lambda_2}(1+\lambda_1)$-Lipschitz continuous.
Therefore, from Theorem \ref{thm:vempala-wibisono} we deduce that Langevin algorithm with learning rate $\eta_t \leq  \frac{\lambda_1 \lambda_2 \delta_{t+1}}{96p(1+\lambda_1)^2\exp(8/\lambda_2)}$ yields $q^{t+1}$ satisfying 
$\KL(q^{(t+1)} \| q^{(t+1)}_*) \leq \delta_{t+1}$ within $\frac{3\lambda_2\exp(8/\lambda_2)}{2\lambda_1\eta_t}\log\frac{2\KL(q^{(t)}\| q^{(t+1)}_*)}{\delta_{t+1}}$-iterations.

We next bound $\KL(q^{(t)}\| q^{(t+1)}_*)$.
Apply Proposition \ref{prop:boundedness-of-KL} with $q=q^{(t)}$, $q_* = q^{(t+1)}_*$, and $q_\sharp = q^{(t)}_*$. 
Note that in this setting, constants $c_*, c_\sharp, \lambda_*,$ and $\lambda_\sharp$ satisfy 
\begin{align*}
&c_* \leq \frac{2}{\lambda_2},~~\frac{\lambda_1}{3\lambda_2} \leq \lambda_* \leq \frac{\lambda_1}{\lambda_2}, \\
&c_\sharp \leq \frac{2}{\lambda_2},~~\frac{\lambda_1}{3\lambda_2} \leq \lambda_\sharp \leq \frac{\lambda_1}{\lambda_2}. 
\end{align*}

Then, we get
\begin{align*}
\KL(q^{(t)}\|q^{(t+1)}_*) 
&\leq \frac{12}{\lambda_2}
+ 6 p\exp\left( \frac{4}{\lambda_2} \right)
+ \frac{p}{2} \log3 
+ \left( 1+ 16 \exp\left( \frac{8}{\lambda_2}\right) \right)\KL(q^{(t)}\|q^{(t)}_*) \\
&+ \frac{2}{\lambda_2}\sqrt{2\KL(q^{(t)}\|q^{(t)}_{*})}.
\end{align*}

Hence, we can conclude $\KL(q^{(t)}\|q^{(t+1)}_*)$ are uniformly bounded with respect to $t \in \{1,\ldots,T\}$ as long as $\KL(q^{(t)}\|q^{(t)}_*) \leq \delta_t$ and $q^{(1)}$ is a Gaussian distribution.
\end{proof}

\paragraph{Case of resampling.} We note that for resampling scheme, the similar inner loop complexity of $O\left(\frac{\lambda_2\exp(8/\lambda_2)}{\lambda_1\eta_t}\log\frac{2\KL(q^{(1)}\| q^{(t+1)}_*)}{\delta_{t+1}} \right)$ can be immediately obtained by replacing the initial distribution of Langevin algorithm with $q^{(1)}(\theta)\rd \theta$. Moreover, the uniform boundedness of $\KL(q^{(1)}\| q^{(t+1)}_*)$ with respect to $t$ is also guaranteed by applying Proposition \ref{prop:boundedness-of-KL} with $q=q_\sharp = q^{(1)}$ and $q_* = q_*^{(t+1)}$ as long as $q^{(1)}(\theta)\rd\theta$ is a Gaussian distribution.

\bigskip
\part*{\large ADDITIONAL RESULTS AND DISCUSSIONS}
\section{Discretization Error of Finite Particles}
\label{sec:discretization}
\subsection{Case of Resampling}
As discussed in subsection \ref{subsec:extension}, to establish the finite-particle convergence guarantees of Algorithm \ref{alg:pda2} with resampling up to $O(\epsilon)$-error, we need to show that $h_x^{(t)} = h_{\tilde{\Theta}^{(t)}}(x_t)$ satisfies the condition $| h^{(t)}_{x} - h_{q^{(t)}}(x_{t}) | \leq \epsilon$ in {\bf (A3')}.
Hence, we are interested in characterizing the discretization error that stems from using finitely many particles.

For the resampling scheme, we can easily derive that the required number of particles is $O(\epsilon^{-2}\log(T/\delta))$ with high probability $1-\delta$, because i.i.d. particles are obtained by the Langevin algorithm and Hoeffding's inequality is applicable.

\begin{lemma}[Hoeffding's inequality] \label{lemma:hoeffding_lemma}
Let $Z, Z_1,\ldots,Z_m$ be i.i.d. random variables taking values in $[-a,a]$ for $a>0$.
Then, for any $\rho > 0$, we get
  \begin{equation*}
    \bP\left[ \left|\frac{1}{M}\sum_{r=1}^M Z_r - \bE[Z] \right| > \rho \right] 
    \leq 2\exp\left( - \frac{\rho^2 M}{2a^2}\right).
  \end{equation*} 
\end{lemma} 
 
\subsection{Case of Warm-start} 
We next consider the warm-start scheme. 
Note that the convergence of PDA with warm-start is guaranteed by coupling it with its mean-field limit $M\rightarrow \infty$ and applying Theorem \ref{thm:convergence} without tolerance (i.e., $\epsilon=0$).
To analyze the particle complexity, we make an additional assumption regarding the regularity of the loss function and the model.
\begin{assumption}\ 
\begin{description}[topsep=0mm,itemsep=0mm]
\item{{\bf(A5)}} $h(\cdot,x)$ is $1$-Lipschitz continuous\footnote{WLOG the Lipschitz constant is set to 1, since the same analysis works for any fixed constant.} for $\forall x \in \cX$.
\end{description}
\vspace{2mm}
\end{assumption}
\paragraph{Remark.} The above regularity assumption is common in the literature and cover many important problem settings in the optimization of two-layer neural network in the mean field regime.
Indeed, {\bf(A5)} is satisfied for two-layer network in Example \ref{eg:2nn} when the output or input layer is fixed and when the activation function is Lipschitz continuous. 

The following proposition shows the convergence of Algorithm \ref{alg:pda} to Algorithm \ref{alg:da} as $M\rightarrow \infty$.

\begin{proposition}[Finite Particle Approximation]\label{prop:finite-particle}
For training examples $\{x_t\}_{t=1}^T$ and any example $\tilde{x}$, define 
\[\rho_{T,M} = \max_{ \substack{s \in \{1,\ldots,T\} \\ t \in \{1,\ldots,T+1\}}}\left| h_{q^{(t)}}(x_s)  - h_{\tilde{\Theta}^{(t)}}(x_s) \right| \lor \left| h_{q^{(t)}}(\tilde{x})  - h_{\tilde{\Theta}^{(t)}}(\tilde{x}) \right|.\]
Under {\bf (A1')}, {\bf (A2)}, {\bf (A4)}, and {\bf (A5)}, if we run PDA (Algorithm \ref{alg:pda}) on $\tilde{\Theta}$ and the corresponding mean field limit DA (Algorithm \ref{alg:da}) on $q$, then with high probability $\lim_{M\rightarrow\infty} \rho_{T,M}=0.$
Moreover, if we set $\eta_t \leq \frac{\lambda_2}{2\lambda_1}$, $\lambda_1 \geq \frac{3}{2}$, and $T_t \geq \frac{3\lambda_2 \log\left(4 \right)}{(2\lambda_1 - 1) \eta_t}$, then with probability at least $1-\delta$,
\begin{align*}
\rho_{T,M} \leq \left(1 + \frac{4}{2\lambda_1 - 1} \right)\sqrt{ \frac{2}{M}\log\left( \frac{2 (T+1)^2}{\delta} \right)}. 
\end{align*}
\end{proposition} 

\paragraph{Remark.}
Proposition~\ref{prop:finite-particle} together with Corollary~\ref{cor:total-complexity} imply that under appropriate regularization, a prediction on any point with an $\epsilon$-gap from an $\epsilon$-accurate solution of the regularized objective~\eqref{eq:regularization} can be achieved with high probability by running PDA with warm-start (Algorithm~\ref{alg:pda}) in $\mathrm{poly}(\epsilon^{-1})$ steps using $\mathrm{poly}(\epsilon^{-1})$ particles, where we omit dependence on hyperparameters and logarithmic factors.
Note that specific choices of hyper-parameters in Proposition~\ref{prop:finite-particle} are consistent with those in Corollary~\ref{cor:total-complexity}. 
We also remark that under weak regularization (vanishing $\lambda_1$), our current derivation suggests that the required particle size could be exponential in the time horizon, due to the particle correlation in the warm-start scheme.
Finally, we remark that for the empirical risk minimization, the term $\log(2(T+1)^2/\delta)$ could be changed to $\log(2n(T+1)/\delta)$ in the obvious way.

\begin{proof}[Proof of Proposition \ref{prop:finite-particle}]
We analyze an error of finite particle approximation for a fixed history of data $\{x_t\}_{t=1}^T$.
To Algorithm \ref{alg:da} with the corresponding particle dynamics (Algorithm \ref{alg:pda}), we construct an {\it semi particle dual averaging} update, which is an intermediate of these two algorithms.
In particular, the semi particle dual averaging method is defined by replacing $h_{\tilde{\Theta}^{(t)}}$ in Algorithm \ref{alg:pda} with $h_{q^{(t)}}$ for $q^{(t)}$ in Algorithm \ref{alg:da}.
Let $\tilde{\Theta}^{\prime(t)} = \{\tilde{\theta}_r^{\prime(t)}\}_{r=1}^M$ be parameters obtained in outer loop of the semi particle dual averaging.
We first estimate the gap between Algorithm \ref{alg:da} and the semi particle dual averaging. 

Note that there is no interaction among $\tilde{\Theta}^{'(t)}$; in other words these are i.i.d.~particles sampled from $q^{(t)}$, and we can thus apply Hoeffding's inequality (Lemma \ref{lemma:hoeffding_lemma}) to $h_{\tilde{\Theta}^{\prime(t)}}(\tilde{x})$ and $h_{\tilde{\Theta}^{\prime(t)}}(x_s)$ $(s\in \{1,\ldots,T\}, t\in \{1,\ldots,T+1\})$.
Hence, for $\forall \delta > 0$, $\forall s \in \{1,\ldots,T\}$, and $\forall t \in \{1,\ldots,T+1\}$, with the probability at least $1-\delta$
\begin{align}
\left| h_{\tilde{\Theta}^{\prime(t)}}(x_s) - h_{q^{(t)}}(x_s) \right|
&= \left| \frac{1}{M}\sum_{r=1}^Mh_{\tilde{\theta}_r^{\prime(t)}}(x_s) - h_{q^{(t)}}(x_s) \right| 
\leq \sqrt{ \frac{2}{M}\log\left( \frac{2 (T+1)^2}{\delta} \right)}, \label{eq:hoeffding_bound} \\
\left| h_{\tilde{\Theta}^{\prime(t)}}(\tilde{x}) - h_{q^{(t)}}(\tilde{x}) \right| 
&= \left| \frac{1}{M}\sum_{r=1}^Mh_{\tilde{\theta}_r^{\prime(t)}}(\tilde{x}) - h_{q^{(t)}}(\tilde{x}) \right| 
\leq \sqrt{ \frac{2}{M}\log\left( \frac{2 (T+1)^2}{\delta} \right)}.  \label{eq:hoeffding_bound_2}
\end{align}

We next bound the gap between the semi particle dual averaging and Algorithm \ref{alg:pda} sharing a history of Gaussian noises and initial particles.
That is, $\tilde{\theta}_r^{(1)} = \tilde{\theta}_r^{\prime(1)}$.
Let $\Theta^{(k)}=\{\theta_r^{(k)}\}_{r=1}$ and $\Theta^{\prime(k)}=\{\theta_r^{\prime(k)}\}_{r=1}$ denote inner iterations of these methods.

($i$) Here we show the first statement of the proposition.
We set $\rho_1 = 0$ and $\overline{\rho}_1 = 0$.
We define $\rho_{t}$ and $\overline{\rho}_t$ recursively as follows.
\begin{align}\label{eq:gap_on_semipda-and-pda}
\rho_{t+1} 
&\defeq \left( 1 + \frac{2(1 + \lambda_1) t \eta_t}{\lambda_2(t+2)} \right)^{T_t} \overline{\rho}_t \notag \\
&+ \frac{t\eta_t}{\lambda_2 (t+2)} \left( \overline{\rho}_t + \sqrt{ \frac{2}{M}\log\left( \frac{2(T+1)^2}{\delta} \right)} \right)
\sum_{s=0}^{T_t-1}\left( 1 + \frac{2(1 + \lambda_1) t \eta_t}{\lambda_2(t+2)} \right)^s,
\end{align}
and $\overline{\rho}_{t+1} = \max_{s \in \{1,\ldots,t+1\}} \rho_{s}$.
We show that for any event where (\ref{eq:hoeffding_bound}) and (\ref{eq:hoeffding_bound_2}) hold, $\|\tilde{\theta}_r^{(t)} - \tilde{\theta}_r^{\prime(t)}\|_2 \leq \rho_t$ $(\forall t \in \{1,\ldots,T+1\}$, $\forall r \in \{1,\ldots,M\})$ by induction.
Suppose $\| \tilde{\theta}_r^{(s)} - \tilde{\theta}_r^{\prime(s)}\|_2 \leq \rho_{s}$ $(\forall s \in \{1,\ldots,t\}$, $\forall r \in \{1,\ldots,M\})$ holds.
Then, for any $x$ and $s \in \{1,\ldots,t\}$
\begin{align}
\left| h_{\tilde{\Theta}^{(s)}}(x) - h_{\tilde{\Theta}^{\prime (s)}}(x) \right|
&\leq \frac{1}{M} \sum_{r=1}^M \left| h(\tilde{\theta}_r^{(s)},x) - h(\tilde{\theta}_r^{\prime (s)}, x) \right| \notag \\
&\leq \frac{1}{M} \sum_{r=1}^M \left\| \tilde{\theta}_r^{(s)} - \tilde{\theta}_r^{\prime (s)} \right\|_2 \leq \rho_{s}. \label{eq:alg_gap}
\end{align}

Consider the inner loop at $t$-the outer step.
Then, for an event where (\ref{eq:hoeffding_bound}) holds,
\begin{align*}
&\quad~\| \theta_r^{(k+1)} - \theta_r^{\prime(k+1)}\|_2 \\
&\leq 
\biggl\| \theta_r^{(k)} 
- \frac{2\eta_t}{\lambda_2 (t+2)(t+1)}\sum_{s=1}^t s \left( \partial_z\ell( h_{\tilde{\Theta}^{(s)}}(x_{s}),y_{s}) \partial_\theta h(\theta_r^{(k)},x_{s}) + 2\lambda_1 \theta_r^{(k)} \right)\\
&- \theta_r^{\prime(k)} 
+ \frac{2\eta_t}{\lambda_2 (t+2)(t+1)}\sum_{s=1}^t s \left( \partial_z\ell( h_{q^{(s)}}(x_{s}),y_{s}) \partial_\theta h(\theta_r^{\prime(k)},x_{s}) + 2\lambda_1 \theta_r^{\prime(k)} \right) \biggr\|_2 \\
&\leq 
\left( 1 + \frac{2\lambda_1 t \eta_t}{\lambda_2(t+2)} \right) \| \theta_r^{(k)} - \theta_r^{\prime(k)}\|_2 \\
&+ \frac{2\eta_t}{\lambda_2 (t+2)(t+1)}\sum_{s=1}^t s \| \partial_z\ell( h_{\tilde{\Theta}^{(s)}}(x_{s}),y_{s}) \partial_\theta h(\theta_r^{(k)},x_{s}) - \partial_z\ell( h_{q^{(s)}}(x_{s}),y_{s}) \partial_\theta h(\theta_r^{\prime(k)},x_{s}) \|_2 \\
&\leq 
\left( 1 + \frac{2\lambda_1 t \eta_t}{\lambda_2(t+2)} \right) \| \theta_r^{(k)} - \theta_r^{\prime(k)}\|_2 \\
&+ \frac{2\eta_t}{\lambda_2 (t+2)(t+1)}\sum_{s=1}^t s \left\| (\partial_z\ell( h_{\tilde{\Theta}^{(s)}}(x_{s}),y_{s}) - \partial_z\ell( h_{q^{(s)}}(x_{s}),y_{s}) ) \partial_\theta h(\theta_r^{(k)},x_{s}) \right\|_2 \\
&+ \frac{2\eta_t}{\lambda_2 (t+2)(t+1)}\sum_{s=1}^t s \left\|\partial_z\ell( h_{q^{(s)}}(x_{s}),y_{s}) ( \partial_\theta h(\theta_r^{\prime(k)},x_{s}) - \partial_\theta h(\theta_r^{(k)},x_{s}) ) \right\|_2 \\
&\leq 
\left( 1 + \frac{2(1 + \lambda_1) t \eta_t}{\lambda_2(t+2)} \right) \| \theta_r^{(k)} - \theta_r^{\prime(k)}\|_2 
+ \frac{2\eta_t}{\lambda_2 (t+2)(t+1)}\sum_{s=1}^t s \left| h_{\tilde{\Theta}^{(s)}}(x_{s}) - h_{q^{(s)}}(x_{s})
\right| \\
&\leq 
\left( 1 + \frac{2(1 + \lambda_1) t \eta_t}{\lambda_2(t+2)} \right) \| \theta_r^{(k)} - \theta_r^{\prime(k)}\|_2 
+ \frac{2\eta_t}{\lambda_2 (t+2)(t+1)}\sum_{s=1}^t s \left(\rho_s + \sqrt{ \frac{2}{M}\log\left( \frac{2(T+1)^2}{\delta} \right)}\right) \\
&\leq 
\left( 1 + \frac{2(1 + \lambda_1) t \eta_t}{\lambda_2(t+2)} \right) \| \theta_r^{(k)} - \theta_r^{\prime(k)}\|_2 
+ \frac{t\eta_t}{\lambda_2 (t+2)} \left( \overline{\rho}_t + \sqrt{ \frac{2}{M}\log\left( \frac{2(T+1)^2}{\delta} \right)} \right).
\end{align*}
Expanding this inequality, 
\begin{align*}
&\quad \| \tilde{\theta}_r^{(t+1)} - \tilde{\theta}_r^{\prime(t+1)}\|_2 \\
&\leq 
\left( 1 + \frac{2(1 + \lambda_1) t \eta_t}{\lambda_2(t+2)} \right)^{T_t} \overline{\rho}_t
+ \frac{t\eta_t}{\lambda_2 (t+2)} \left( \overline{\rho}_t + \sqrt{ \frac{2}{M}\log\left( \frac{2 (T+1)^2}{\delta} \right)} \right)
\sum_{s=0}^{T_t-1}\left( 1 + \frac{2(1 + \lambda_1) t \eta_t}{\lambda_2(t+2)} \right)^s \\
&= \rho_{t+1}.
\end{align*}
Hence, $\| \tilde{\theta}_r^{(t)} - \tilde{\theta}_r^{\prime(t)}\|_2  \leq \overline{\rho}_{T+1}$ for $\forall t \in \{1,\ldots,T+1\}$.

Noting that $\overline{\rho}_1 = 0$ and 
\begin{align*} 
\rho_{t+1}
&= \left( \left( 1 + \frac{2(1 + \lambda_1) t \eta_t}{\lambda_2(t+2)} \right)^{T_t}  + \frac{t\eta_t}{\lambda_2 (t+2)} \sum_{s=0}^{T_t-1}\left( 1 + \frac{2(1 + \lambda_1) t \eta_t}{\lambda_2(t+2)} \right)^s \right) \overline{\rho}_t \\
&+ \frac{t\eta_t}{\lambda_2 (t+2)} \sqrt{ \frac{2}{M}\log\left( \frac{2 (T+1)^2}{\delta} \right)} 
\sum_{s=0}^{T_t-1}\left( 1 + \frac{2(1 +\lambda_1) t \eta_t}{\lambda_2(t+2)} \right)^s,
\end{align*}
we see $\overline{\rho}_{T+1} \rightarrow 0$ as $M\rightarrow +\infty$.
Then, the proof is finished because for $\forall t \in \{1,\ldots,T+1\}$ and $\forall s \in \{1,\ldots,T\}$ with high probability $1-\delta$,
\begin{align*}
\left| h_{\tilde{\Theta}^{(t)}}(x_s) - h_{q^{(t)}}(x_s) \right| 
&\leq 
\left| h_{\tilde{\Theta}^{(t)}}(x_s) - h_{\tilde{\Theta}^{\prime(t)}}(x_s) \right| 
+ \left| h_{\tilde{\Theta}^{\prime(t)}}(x_s) - h_{q^{(t)}}(x_s) \right| \\
&\leq \overline{\rho}_{T+1} + \sqrt{\frac{2}{M}\log \left( \frac{2(T+1)^2}{\delta} \right)}, \\
\left| h_{\tilde{\Theta}^{(t)}}(\tilde{x}) - h_{q^{(t)}}(\tilde{x}) \right| 
&\leq 
\left| h_{\tilde{\Theta}^{(t)}}(\tilde{x}) - h_{\tilde{\Theta}^{\prime(t)}}(\tilde{x}) \right| 
+ \left| h_{\tilde{\Theta}^{\prime(t)}}(\tilde{x}) - h_{q^{(t)}}(\tilde{x}) \right| \\
&\leq \overline{\rho}_{T+1} + \sqrt{\frac{2}{M}\log \left( \frac{2(T+1)^2}{\delta} \right)}.
\end{align*}

($ii$) We next show the second statement of the proposition.
We change the definition (\ref{eq:gap_on_semipda-and-pda}) of $\rho_{t+1}$ as follows:
\begin{equation*}
\rho_{t+1} \defeq 
\frac{3}{4} \overline{\rho}_t + \frac{1}{2\lambda_1-1} \sqrt{ \frac{2}{M}\log\left( \frac{2 (T+1)^2}{\delta} \right)}.
\end{equation*}

We prove that for any event where (\ref{eq:hoeffding_bound}) and (\ref{eq:hoeffding_bound_2}) hold, $\|\tilde{\theta}_r^{(t)} - \tilde{\theta}_r^{\prime(t)}\|_2 \leq \rho_t$ $(\forall t \in \{1,\ldots,T+1\}$, $\forall r \in \{1,\ldots,M\})$ by induction.
Suppose $\| \tilde{\theta}_r^{(s)} - \tilde{\theta}_r^{\prime(s)}\|_2 \leq \rho_s$ $(\forall s \in \{1,\ldots,t\}$, $\forall r \in \{1,\ldots,M\})$ holds.
Consider the inner loop at $t$-step.
Note that $\eta_t \leq \frac{\lambda_2}{2\lambda_1}$ implies $1-\frac{2\lambda_1 t \eta_t}{\lambda_2 (t+2)} > 0$.
Therefore, by the similar argument as above, we get
\begin{align*}
&\quad\| \theta_r^{(k+1)} - \theta_r^{\prime(k+1)}\|_2  \\
&\leq 
\biggl\| \theta_r^{(k)} 
- \frac{2\eta_t}{\lambda_2 (t+2)(t+1)}\sum_{s=1}^t s \left( \partial_z\ell( h_{\tilde{\Theta}^{(s)}}(x_{s}),y_{s}) \partial_\theta h(\theta_r^{(k)},x_{s}) + 2\lambda_1 \theta_r^{(k)} \right)\\
&- \theta_r^{(k)} 
+ \frac{2\eta_t}{\lambda_2 (t+2)(t+1)}\sum_{s=1}^t s \left( \partial_z\ell( h_{q^{(s)}}(x_{s}),y_{s}) \partial_\theta h(\theta_r^{\prime(k)},x_{s}) + 2\lambda_1 \theta_r^{\prime(k)} \right) \biggr\|_2 \\
&\leq 
\left( 1 - \frac{2\lambda_1 t \eta_t}{\lambda_2(t+2)} \right) \| \theta_r^{(k)} - \theta_r^{\prime(k)}\|_2 \\
&+ \frac{2\eta_t}{\lambda_2 (t+2)(t+1)}\sum_{s=1}^t s \| \partial_z\ell( h_{\tilde{\Theta}^{(s)}}(x_{s}),y_{s}) \partial_\theta h(\theta_r^{(k)},x_{s}) - \partial_z\ell( h_{q^{(s)}}(x_{s}),y_{s}) \partial_\theta h(\theta_r^{\prime(k)},x_{s}) \|_2 \\
&\leq 
\left( 1 + \frac{(1 - 2\lambda_1) t \eta_t}{\lambda_2(t+2)} \right) \| \theta_r^{(k)} - \theta_r^{\prime(k)}\|_2 
+ \frac{t\eta_t}{\lambda_2 (t+2)} \left( \overline{\rho}_t + \sqrt{ \frac{2}{M}\log\left( \frac{2 (T+1)^2}{\delta} \right)} \right).
\end{align*}

Expanding this inequality, 
\begin{align*}
&\quad\| \tilde{\theta}_r^{(t+1)} - \tilde{\theta}_r^{\prime(t+1)}\|_2 \\ 
&\leq 
\left( 1 + \frac{(1 - 2\lambda_1) t \eta_t}{\lambda_2(t+2)} \right)^{T_t} \overline{\rho}_t
+ \frac{t\eta_t}{\lambda_2 (t+2)} \left( \overline{\rho}_t + \sqrt{ \frac{2}{M}\log\left( \frac{2 (T+1)^2}{\delta} \right)} \right)
\sum_{s=0}^{T_t-1}\left( 1 + \frac{(1 - 2\lambda_1) t \eta_t}{\lambda_2(t+2)} \right)^s \\
&\leq 
\left(\left( 1 + \frac{(1 - 2\lambda_1) t \eta_t}{\lambda_2(t+2)} \right)^{T_t} + \frac{1}{2\lambda_1-1}\right) \overline{\rho}_t + \frac{1}{2\lambda_1-1} \sqrt{ \frac{2}{M}\log\left( \frac{2 (T+1)^2}{\delta} \right)}  \\
&\leq 
\left(\left( 1 + \frac{(1 - 2\lambda_1) t \eta_t}{\lambda_2(t+2)} \right)^{T_t} + \frac{1}{2}\right) \overline{\rho}_t + \frac{1}{2\lambda_1-1} \sqrt{ \frac{2}{M}\log\left( \frac{2 (T+1)^2}{\delta} \right)}, 
\end{align*}
where we used $0 < 1 + \frac{(1 - 2\lambda_1) t \eta_t}{\lambda_2(t+2)} < 1$ and $\lambda_1 \geq \frac{3}{2}$.

Noting that $(1-x)^{1/x} \leq \exp(-1)$ for $\forall x \in (0,1]$, we see that
\begin{align*}
\left( 1 - \frac{( 2\lambda_1 - 1) t \eta_t}{\lambda_2(t+2)} \right)^{T_t}    
&\leq \left( 1 - \frac{( 2\lambda_1 - 1) t \eta_t}{\lambda_2(t+2)} \right)^{\frac{3\lambda_2}{(2\lambda_1 - 1) \eta_t} \log\left(4\right)}  \\
&= \left( 1 - \frac{( 2\lambda_1 - 1) t \eta_t}{\lambda_2(t+2)} \right)^{\frac{\lambda_2(t+2)}{(2\lambda_1 - 1) t \eta_t} \frac{3t}{t+2} \log\left(4\right)}\\
&\leq \exp\left( - \frac{3t}{t+2} \log\left(4\right) \right) \\
&\leq \exp\left( - \log\left(4\right) \right) \\
&= \frac{1}{4},
\end{align*}
where we used $T_t \geq \frac{3\lambda_2 \log\left(4 \right)}{(2\lambda_1 - 1) \eta_t}$.
Hence, we know that for $t$,
\begin{equation}
\| \tilde{\theta}_r^{(t+1)} - \tilde{\theta}_r^{\prime(t+1)}\|_2 
\leq 
\frac{3}{4} \overline{\rho}_t + \frac{1}{2\lambda_1-1} \sqrt{ \frac{2}{M}\log\left( \frac{2 (T+1)^2}{\delta} \right)}. \label{eq:rho_ineq}
\end{equation}

This means that $\|\tilde{\theta}_r^{(t+1)} - \tilde{\theta}_r^{\prime(t+1)}\|_2 \leq \rho_{t+1}$ and finishes the induction.

Next, we show 
\begin{equation} 
\overline{\rho}_t \leq \frac{4}{2\lambda_1-1} \sqrt{ \frac{2}{M}\log\left( \frac{2 (T+1)^2}{\delta} \right)}. \label{eq:rho-ineq}
\end{equation}

This inequality obviously holds for $t=1$ because $\overline{\rho}_1 = 0$. 
We suppose it is true for $t \leq T$.
Then, 
\begin{align*}
\rho_{t+1} 
&= \frac{3}{4} \overline{\rho}_t + \frac{1}{2\lambda_1-1} \sqrt{ \frac{2}{M}\log\left( \frac{2 (T+1)^2}{\delta} \right)} \\
&\leq \frac{4}{2\lambda_1-1} \sqrt{ \frac{2}{M}\log\left( \frac{2 (T+1)^2}{\delta} \right)}.
\end{align*}

Hence, the inequality (\ref{eq:rho-ineq}) holds for $\forall t \in \{1,\ldots,T+1\}$, yielding
\[ 
\|\tilde{\theta}_r^{(t+1)} - \tilde{\theta}_r^{\prime(t+1)}\|_2 \leq \frac{4}{2\lambda_1-1} \sqrt{ \frac{2}{M}\log\left( \frac{2 (T+1)^2}{\delta} \right)}.
\]

In summary, it follows that for $\forall t \in \{1,\ldots,T+1\}$ and $\forall s \in \{1,\ldots,T\}$ with high probability $1-\delta$,
\begin{align*}
\left| h_{\tilde{\Theta}^{(t)}}(x_s) - h_{q^{(t)}}(x_s) \right| 
&\leq 
\left| h_{\tilde{\Theta}^{(t)}}(x_s) - h_{\tilde{\Theta}^{\prime(t)}}(x_s) \right| 
+ \left| h_{\tilde{\Theta}^{\prime(t)}}(x_s) - h_{q^{(t)}}(x_s) \right| \\
&\leq 
\left(1 + \frac{4}{2\lambda_1 - 1} \right)
\sqrt{ \frac{2}{M}\log\left( \frac{2 (T+1)^2}{\delta} \right)}, \\
\left| h_{\tilde{\Theta}^{(t)}}(\tilde{x}) - h_{q^{(t)}}(\tilde{x}) \right| 
&\leq 
\left| h_{\tilde{\Theta}^{(t)}}(\tilde{x}) - h_{\tilde{\Theta}^{\prime(t)}}(\tilde{x}) \right| 
+ \left| h_{\tilde{\Theta}^{\prime(t)}}(x_s) - h_{q^{(t)}}(x_s) \right| \\
&\leq 
\left(1 + \frac{4}{2\lambda_1 - 1} \right)
\sqrt{ \frac{2}{M}\log\left( \frac{2 (T+1)^2}{\delta} \right)}, 
\end{align*}
where we used (\ref{eq:alg_gap}).
This completes the proof.
\end{proof}
\bigskip
\section{Generalization Bounds for Empirical Risk Minimization}
\label{app:generalization_bounds}
In this section, we give generalization bounds for the problem (\ref{eq:rerm}) in the context of \textit{empirical risk minimization}, by using techniques developed by \citet{chen2020generalized}.
We consider the smoothed hinge loss and squared loss for binary classification and regression problems, respectively.

\subsection{Auxiliary Results}
For a set $\cF$ of functions from a space $\cZ$ to $\bR$ and a set $S=\{z_i\}_{i=1}^n \subset \cZ$, 
the empirical Rademacher complexity $\rademacher_S(\cF)$ is defined as follows:
\[ \rademacher_S(\cF) = \bE_\sigma\left[ \sup_{f\in\cF}\frac{1}{n}\sum_{i=1}^n\sigma_i f(z_i)\right], \]
where $\sigma =(\sigma_i)_{i=1}^n$ are i.i.d random variables taking $-1$ or $1$ with equal probability.

We introduce the uniform bound using the empirical Rademacher complexity (see \citet{mohri2012foundations}).
\begin{lemma}[Uniform bound] \label{lem:uniform_bound}
  Let $\cF$ be a set of functions from $\cZ$ to $[-C,C]$ $(C \in \bR)$ and 
  $\cD$ be a distribution over $\cZ$.
  Let $S=\{z_i\}_{i=1}^n \subset \cZ$ be a set of size $n$ drawn from $\cD$.
  Then, for any $\delta \in (0,1)$, with probability at least $1-\delta$ over the choice of $S$, we have
  \begin{equation*}
    \sup_{f \in \cF} \left\{ \bE_{Z\sim \cD}[f(Z)] - \frac{1}{n}\sum_{i=1}^nf(z_i) \right\} 
    \leq 2 \rademacher_S(\cF) + 3C\sqrt{\frac{1}{2n}\log\frac{2}{\delta}}.
  \end{equation*}
\end{lemma}

The contraction lemma (see \citet{shalev2014understanding}) is useful in estimating the Rademacher complexity.
\begin{lemma}[Contraction lemma] \label{lem:contraction}
  Let $\phi_i : \bR \rightarrow \bR$ $(i\in \{1,\ldots,n\})$ be $\rho$-Lipschitz functions 
  and $\cF$ be a set of functions from $\cZ$ to $\bR$. 
  Then it follows that for any $\{z_i\}_{i=1}^n \subset \cZ$,
  \begin{equation*}
    \bE_\sigma\left[ \sup_{f\in\cF} \frac{1}{n}\sum_{i=1}^n\sigma_i \phi_i\circ f(z_i)\right]
    \leq \rho \bE_\sigma\left[ \sup_{f \in \cF} \frac{1}{n}\sum_{i=1}^n\sigma_i \circ f(z_i)\right].
  \end{equation*}
\end{lemma}

Let $p_0(\theta) \mathrm{d}\theta$ be a distribution in proportion to $\exp\left( -\frac{\lambda_1}{\lambda_2} \|\theta\|_2^2\right)\mathrm{d}\theta$.
We define a family of mean field neural networks as follows: for $R>0$,
\[ \cF_{\KL}(R) = \left\{ h_q : \cX \rightarrow \bR \mid q \in \cP_2,~ \KL(q \| p_0 ) \leq R  \right\}. \]
The Rademacher complexity of this function class is obtained by \citet{chen2020generalized}.
\begin{lemma}[\citet{chen2020generalized}] \label{lem:kl-rademacher}
Suppose $|h_\theta(x) | \leq 1$ holds for $\forall \theta \in \Omega$ and $\forall x \in \cX$. 
We have for any constant $R \leq \frac{1}{2}$ and set $S \subset \cX$ of size $n$,
\[ \rademacher_S( \cF_{\KL}(R)) \leq 2\sqrt{\frac{R}{n}}. \]
\end{lemma}

\subsection{Generalization Bound on the Binary Classification Problems}
We here give a generalization bound for the binary classification problems.
Hence, we suppose $\cY=\{-1,1\}$ and consider the problem (\ref{eq:rerm}) with the smoothed hinge loss defined below. 
\begin{equation*}
\ell(z,y) = \left\{
\begin{array}{ll}
0 &~\textrm{if}~zy \geq 1/2, \\
(1-2zy)^2 &~\textrm{if}~0\leq zy < 1/2, \\
1-4zy &~\textrm{else}.
\end{array}
\right.
\end{equation*}
We also define the $0$-$1$ loss as $\ell_{01}(z,y) = \mathbbm{1}[zy < 0]$.
\begin{theorem}\label{thm:gen-bound-classification}
Let $\cD$ be a distribution over $\cX \times \cY$. 
Suppose there exists a true distribution $q^{\circ} \in \cP_2$ satisfying $h_{q^{\circ}}(x)y \geq 1/2$ for $\forall (x,y) \in \mathrm{supp}(\cD)$ and $\KL(q^{\circ} \| p_0) \leq 1/2$.
Let $S=\{(x_i,y_i)\}_{i=1}^n$ be training examples independently sampled from $\cD$.
Suppose $|h_\theta(x) | \leq 1$ holds for $\forall (\theta,x) \in \Omega\times \cX$.
Then, for the minimizer $q_*\in \cP_2$ of the problem (\ref{eq:rerm}), it follows that with probability at least $1-\delta$ over the choice of $S$, 
\[ \bE_{(X,Y)\sim \cD}[ \ell_{01}(h_{q_*}(X),Y)] 
\leq \lambda_2 \KL(q^{\circ} \| p_0) + 16 \sqrt{\frac{\KL(q^{\circ} \| p_0)}{n}} + 15\sqrt{\frac{1}{2n} \log\frac{2}{\delta}}. \]
\end{theorem}
\begin{proof}
We first estimate a radius $R$ to satisfy $q_* \in \cF_{\KL}(R)$.
Note that the regularization term of objective $\cL(q)$ is $\lambda_2 \KL(q \| p_0)$ and that 
$\ell(h_{q^{\circ}}(x_i),y_i)=0$ from the assumption on $q^{\circ}$ and the definition of the smoothed hinge loss.
Since $\cL(q_*) \leq \cL(q^{\circ})$, we get
\begin{align}
\KL(q_* \| p_0) 
&\leq \frac{1}{\lambda_2}\cL(q^{\circ}) 
= \KL(q^{\circ}\| p_0), \label{eq:kl-bound}\\
\frac{1}{n}\sum_{i=1}^n \ell( h_{q_*}(x_i),y_i) 
&\leq \cL(q^{\circ})
= \lambda_2 \KL(q^{\circ} \| p_0). \label{eq:loss-bound}
\end{align}
Especially, setting $R= \KL(q^{\circ}\| p_0)$, we see $q_* \in \cF_{\KL}(R)$.

We next define the set of composite functions of loss and mean field neural networks as follows:
\begin{equation}
\cF(R) = \{ (x,y) \in \cX \times \cY \longmapsto \ell( h(x), y) \mid~h \in \cF_{\KL}(R) \}. \label{eq:loss-with-nn}    
\end{equation} 

Since $\ell(z,y)$ is $4$-Lipschitz continuous with respect to $z$, we can estimate the Rademacher complexity $\rademacher_S(\cF(R))$ by using 
Lemma \ref{lem:contraction} with $\phi_i(\cdot) = \ell(\cdot,y_i)$ as follows:
\begin{align}
\rademacher_S(\cF(R))
&= \bE_\sigma\left[ \sup_{ h \in\cF_{\KL}(R)}\frac{1}{n}\sum_{i=1}^n\sigma_i \ell(h(x_i),y_i) \right] \notag\\
&\leq 4\bE_\sigma\left[ \sup_{ h \in\cF_{\KL}(R)}\frac{1}{n}\sum_{i=1}^n\sigma_i h(x_i) \right] \notag\\
&= 4\rademacher_{\{x_i\}_{i=1}^n}(\cF_{\KL}(R)) \notag\\
&\leq 8 \sqrt{\frac{R}{n}}, \label{eq:rademacher-eval}
\end{align}
where we used Lemma \ref{lem:kl-rademacher} for the last inequality.

From the boundedness assumption on $h_q$, we have $0 \leq \ell (h_q(x),y) \leq 5$ for $\forall q \in \cP_2$.
Applying Lemma \ref{lem:uniform_bound} with $\cF = \cF(R)$, we have with probability at least $1-\delta$,
\begin{align*}
\bE_{(X,Y)\sim \cD}[ \ell_{01}(h_{q_*}(X),Y)] 
&\leq \bE_{(X,Y)\sim \cD}[ \ell(h_{q_*}(X),Y)] \\
&\leq \frac{1}{n} \sum_{i=1}^n \ell( h_{q_*}(x_i),y_i)
+ 2\rademacher_S(\cF(R)) + 15\sqrt{\frac{1}{2n} \log\frac{2}{\delta}}\\
&\leq 
\lambda_2 \KL(q^{\circ} \| p_0) + 16 \sqrt{\frac{R}{n}} + 15\sqrt{\frac{1}{2n} \log\frac{2}{\delta}} \\
&= \lambda_2 \KL(q^{\circ} \| p_0) + 16 \sqrt{\frac{\KL(q^{\circ} \| p_0)}{n}} + 15\sqrt{\frac{1}{2n} \log\frac{2}{\delta}},
\end{align*}
where we used $\ell_{01}(z,y) \leq \ell(z,y)$, (\ref{eq:loss-bound}) and (\ref{eq:rademacher-eval}).
\end{proof}

This theorem results in the following corollary:
\begin{corollary}
Suppose the same assumptions in Theorem \ref{thm:gen-bound-classification} hold.
Moreover, we set $\lambda_1 = \lambda/\sqrt{n}$ $(\lambda > 0)$ and $\lambda_2 = 1/\sqrt{n}$.
Then, the following bound holds with the probability at least $1-\delta$ over the choice of training examples,
\[ \bE_{(X,Y)\sim \cD}[ \ell_{01}(h_{q_*}(X),Y)] 
\leq \frac{\KL(q^{\circ}\| p_0^\prime)}{\sqrt{n}} + 16 \sqrt{\frac{\KL(q^{\circ} \| p_0^\prime)}{n}} + 15\sqrt{\frac{1}{2n} \log\frac{2}{\delta}}, \]
where $p_0^\prime$ is the Gaussian distribution in proportion to $\exp(-\lambda \|\cdot\|_2^2)$.
\end{corollary}

\subsection{Generalization Bound on the Regression Problem}
We here give a generalization bound for the regression problems.
We consider the squared loss $\ell(z,y)=0.5(z-y)^2$ and the bounded label $\cY \subset [-1,1]$.

\begin{theorem}\label{thm:gen-bound-regression}
Let $\cD$ be a distribution over $\cX \times \cY$. 
Suppose there exists a true distribution $q^{\circ} \in \cP_2$ satisfying $y = h_{q^{\circ}}(x)$ for $\forall (x,y) \in \mathrm{supp}(\cD)$ and $\KL(q^{\circ} \| p_0) \leq 1/2$.
Let $S=\{(x_i,y_i)\}_{i=1}^n$ be training examples independently sampled from $\cD$.
Suppose $|h_\theta(x) | \leq 1$ holds for $\forall (\theta,x) \in \Omega\times \cX$.
Then, for the minimizer $q_*\in \cP_2$ of the problem (\ref{eq:rerm}), it follows that with probability at least $1-\delta$ over the choice of $S$, 
\[ \bE_{(X,Y)\sim \cD}[ \ell(h_{q_*}(X),Y)] 
\leq \lambda_2 \KL(q^{\circ} \| p_0) + 8 \sqrt{\frac{\KL(q^{\circ} \| p_0)}{n}} + 6\sqrt{\frac{1}{2n} \log\frac{2}{\delta}}. \]
\end{theorem}
\begin{proof}
The proof is very similar to that of Theorem \ref{thm:gen-bound-classification}.
Note that $\ell(h_{q^{\circ}}(x_i),y_i)=0$ from the assumption on $q^{\circ}$ and that inequalities (\ref{eq:kl-bound}) and (\ref{eq:loss-bound}) hold in this case too.
Hence, setting $R= \KL(q^{\circ}\| p_0)$, we see $q_* \in \cF_{\KL}(R)$.

Since $\ell(z,y)$ is $2$-Lipschitz continuous with respect to $z \in [-1,1]$ for any $y \in \cY \subset [-1,1]$, 
we can estimate the Rademacher complexity $\rademacher_S(\cF(R))$ of $\cF(R)$ (defined in (\ref{eq:loss-with-nn})) in the same way as Theorem \ref{thm:gen-bound-classification}:
\begin{align}
\rademacher_S(\cF(R)) \leq 4 \sqrt{\frac{R}{n}}. \label{eq:rademacher-eval2}
\end{align}

From the boundedness assumption on $h_q$ and $\cY$, we have $0 \leq \ell (h_q(x),y) \leq 2$ for $\forall q \in \cP_2$.
Hence, applying Lemma \ref{lem:uniform_bound} with $\cF = \cF(R)$, we have with probability at least $1-\delta$,
\begin{align*}
\bE_{(X,Y)\sim \cD}[ \ell(h_{q_*}(X),Y)] 
&\leq \frac{1}{n} \sum_{i=1}^n \ell( h_{q_*}(x_i),y_i)
+ 2\rademacher_S(\cF(R)) + 6\sqrt{\frac{1}{2n} \log\frac{2}{\delta}}\\
&\leq \lambda_2 \KL(q^{\circ} \| p_0) + 8 \sqrt{\frac{\KL(q^{\circ} \| p_0)}{n}} + 6\sqrt{\frac{1}{2n} \log\frac{2}{\delta}},
\end{align*}
where we used (\ref{eq:loss-bound}) and (\ref{eq:rademacher-eval2}).
\end{proof}

This theorem results in the following corollary:
\begin{corollary}
Suppose the same assumptions in Theorem \ref{thm:gen-bound-regression} hold.
Moreover, we set $\lambda_1 = \lambda/\sqrt{n}$ $(\lambda > 0)$ and $\lambda_2 = 1/\sqrt{n}$.
Then, the following bound holds with the probability at least $1-\delta$ over the choice of training examples,
\[ \bE_{(X,Y)\sim \cD}[ \ell(h_{q_*}(X),Y)] 
\leq \frac{\KL(q^{\circ}\| p_0^\prime)}{\sqrt{n}} + 8 \sqrt{\frac{\KL(q^{\circ} \| p_0^\prime)}{n}} + 6\sqrt{\frac{1}{2n} \log\frac{2}{\delta}}, \]
where $p_0^\prime$ is the Gaussian distribution in proportion to $\exp(-\lambda \|\cdot\|_2^2)$.
\end{corollary}

\bigskip
\section{Additional Discussions}

\subsection{Efficient Implementation of PDA}
\label{sec:implementation}
Note that similar to SGD, Algorithm \ref{alg:pda} only requires gradient queries (and additional Gaussian noise); in particular, a weighted average $\overline{g}^{(t)}$ of functions $g^{(t)}$ is updated and its derivative with respect to parameters is calculated.
In the case of empirical risk minimization, this procedure can be implemented as follows.
We use $\{w_i\}_{i=1}^n$ (initialized as zeros) to store the weighted sums of $\partial_z\ell(h_{\tilde{\Theta}^{(t)}}(x_{i_t}), y_{i_t})$.
At step $t$ in the outer loop, $w_{i_t}$ is updated as 
\[ w_{i_t} \leftarrow w_{i_t} + t \partial_z\ell(h_{\tilde{\Theta}^{(t)}}(x_{i_t}), y_{i_t}). \]
The average $\nabla_{\theta_r}\overline{g}^{(t)}(\Theta^{(k)})$ can then be computed as 
\[ \frac{2}{\lambda_2 (t+2)(t+1)}\sum_{i=1}^n w_i \partial_\theta h(\theta_r^{(k)},x_{i}) + \frac{2\lambda_1 t}{\lambda_2(t+2)} \theta_r^{(k)}, \]
where we use $\{\theta_r^{(k)}\}_{k=1}^M$ to denote parameters $\Theta^{(k)}$ at step $k$ of the inner loop.
This formulation makes Algorithm \ref{alg:pda} straightforward to implement.

In addition, the PDA algorithm can also be implemented with mini-batch update, in which a set of data indices $I_t = \{i_{t,1},\ldots,i_{t,b}\} \subset \{1,2,\ldots,n\}$ is selected per outer loop step instead of one single index $i_t$.
Due to the reduced variance, mini-batch update can stabilize the algorithm and lead to faster convergence. 
Our theoretical results in the sequel trivially extends to the mini-batch setting.

\subsection{Extension to Multi-class Classification}
\label{sec:multiclass-classification}
We give a natural extension of PDA method to multi-class classification settings.
Let $\cC$ denote the finite set of all class labels and $|\cC|$ denote its cardinality.
For multi-class classification problems, we define a component $h(\theta,x)$ of an ensemble as follows.
Let $a_r \in \bR^{|\cC|}$ and $b_r \in \bR^{d}$ ($r \in \{1,\ldots, M\}$) be parameters for output and input layers, respectively, and set $\theta_r = (a_r, b_r)$ and $\Theta = \{\theta_r\}_{r=1}^M$.
Then, we define $h_{\theta_r}(x) = h(\theta,x) = \sigma_2( a_r \sigma_1 (b_r^\top x) )$\footnote{Here, $a_r \sigma_1 (b_r^\top x)$ is a scalar $\sigma_1 (b_r^\top x)$ times a vector $a_r$.} which is a neural network with one hidden neuron, and denote 
\[ h_\Theta(x) = \frac{1}{M}\sum_{r=1}^M h_{\theta_r}(x). \]
Note that $h_\Theta(x)$ is a natural two-layer neural network with multiple outputs.
Suppose that each parameter $\theta_r$ follows $q(\theta)\rd\theta$. Then the mean field limit can be defined as 
\[ h_q(\cdot) = \bE_{\theta \sim q}[h_\theta(\cdot)]: \bR^d \to \bR^{|\cC|}. \]

Let $\ell(z,y)$ ($z=\{z_y\}_{y \in \cC} \in \bR^{|\cC|}, y \in \cC$) be the loss for multi-class classification problems. 
A typical choice is the cross-entropy loss with the soft-max activation, that is 
\[ \ell(z,y) = -\log \frac{\exp( z_y ) }{ \sum_{y' \in \cC}\exp( z_{y'})} = -z_y + \log  \sum_{y' \in \cC}\exp( z_{y'}). \]

In this case, the functional derivative of $\ell( h_q(x), y)$ with respect to $q$ is 
\[ - h_y(\theta,x) + \frac{\sum_{y' \in \cC} \exp( h_{q,y'}(x) ) h_{y'}(\theta,x)}{ \sum_{y' \in \cC} \exp( h_{q,y'}(x) ) }\]
where we supposed the outputs of $h_\theta$ and $h_q$ are also indexed by $\cC$.
Hence, the counterpart of $g^{(t)}$ in Algorithm \ref{alg:da} in this setting is 
\begin{equation*}
     g^{(t)} = - h_{y_t}(\cdot,x_t) 
     + \frac{\sum_{y' \in \cC} \exp( h_{q^{(t)},y'}(x_t) ) h_{y'}(\cdot,x_t)}{ \sum_{y' \in \cC} \exp( h_{q^{(t)},y'}(x_t) ) } 
     + \lambda_1\|\cdot\|_2^2. 
\end{equation*}
Using this function, the DA method for multi-class classification problems can be obtained in the same manner as Algorithm \ref{alg:da}.
Moreover, its discretization can be also immediately derived by replacing the function $\overline{g}^{(t)}$ used in Algorithm \ref{alg:pda} with 
\begin{equation*}
    \overline{g}^{(t)}
    = \frac{2}{\lambda_2(t+2)(t+1)}\sum_{s=1}^t s \left(
    - h_{y_s}(\cdot,x_s) 
    + \frac{\sum_{y' \in \cC} \exp( h_{\tilde{\Theta}^{(s)},y'}(x_s) ) h_{y'}(\cdot,x_s)}{ \sum_{y' \in \cC} \exp( h_{\tilde{\Theta}^{(s)},y'}(x_s) ) } 
    + \lambda_1\|\cdot\|_2^2 \right).
\end{equation*}

In the case of empirical risk minimization, we can adopt an efficient implementation as done in Section \ref{sec:implementation}.
We use $\{w_{i,y}\}_{i\in \{1,\ldots,n\}, y \in \cC}$ (initialized as zeros) to store the coefficients of $h_{y}(\cdot, x_i)$. 
At step $t$ in the outer loop, $w_{i_t,y}$ ($y \in \cC$) are updated as 
\begin{equation*}
    w_{i_t, y} 
    \leftarrow 
    \begin{cases}
        w_{i_t, y} + t \left( -1 + \frac{\exp( h_{\tilde{\Theta}^{(t)},y}(x_{i_t}) ) }{ \sum_{y' \in \cC} \exp( h_{\tilde{\Theta}^{(t)},y'}(x_{i_t}) ) }  \right) & y=y_{i_t}, \\
        w_{i_t, y} + t \frac{\exp( h_{\tilde{\Theta}^{(t)},y}(x_{i_t}) )}{ \sum_{y' \in \cC} \exp( h_{\tilde{\Theta}^{(t)},y'}(x_{i_t}) ) } & y \neq y_{i_t}.
    \end{cases}
\end{equation*}

Then, $\nabla_{\theta_r}\overline{g}^{(t)}(\Theta^{(k)})$ can be computed as 
\[ \frac{2}{\lambda_2 (t+2)(t+1)}\sum_{i=1}^n \sum_{y\in \cC} w_{i,y} \partial_\theta h_y(\theta_r^{(k)},x_{i}) + \frac{2\lambda_1 t}{\lambda_2(t+2)} \theta_r^{(k)}, \]
where we use $\{\theta_r^{(k)}\}_{k=1}^M$ to denote parameters $\Theta^{(k)}$ at step $k$ of the inner loop.

Finally, we remark that while we here utilize a simple network $h_\theta(x)$ to recover a normal two-layer neural network, it is also possible to use deep narrow networks or narrow convolutional neural networks as a component $h_\theta(x)$; in other words $h_\Theta$ can represent an ensemble of various types of small network. While such extensions are not covered by our current theoretical analysis, they may achieve better practical performance.

\subsection{Correspondence with Finite-dimensional Dual Averaging Method}
\label{sec:original-da-method}
We explain the correspondence between the finite-dimensional dual averaging method developed by \cite{nesterov2005smooth,nesterov2009primal,xiao2009dual} and our proposed method (Algorithm \ref{alg:da}); our goal here is to provide an intuitive understanding of the derivation of Algorithm \ref{alg:da} in the context of the classical dual averaging method.  

First, we introduce the (regularized) dual averaging method \citep{nesterov2009primal,xiao2009dual} in a more general form for solving the regularized optimization problem on the finite-dimensional space. Let $w \in \bR^m$ be a parameter, $l(w,z)$ be a convex loss in $w$, where $z$ is a random variable which represents an example, and $\Psi(w)$ is a regularization function. 
Then, the problem solved by the dual averaging method is given as 
\begin{equation*}
    \min_{w \in \bR^m}\left\{ \bE_z[ l(w,z) ] + \Psi(w) \right\}.
\end{equation*}
Let $\{w^{(s)}\}_{s=1}^t$ and $\{f^{(s)}\}_{s=1}^{t} = \{\partial_w l(w^{(s)},z_s) \}_{s=1}^t$ be histories of iterates and stochastic gradients. The subproblems to produce the next iterate in the dual averaging method is designed by using the strongly convex function $d(w)$ and positive hyperparameters $\{\alpha_s\}_{s=1}^\infty$ and $\{\beta_s\}_{s=2}^\infty$.
Specifically, the next iterate $w^{(t+1)}$ is defined as the minimizer of the following problem in which the loss function is linearized and weighted sum of which is taken over the history:
\begin{equation}\label{eq:subprob-original-da}
    \min_{w \in \bR^m}\left\{ \sum_{s=1}^{t} \alpha_s f^{(s)\top}w + \sum_{s=1}^{t} \alpha_s \Psi(w) + \beta_{t+1} d(w) \right\}.
\end{equation}

Next, we consider our problem setting of optimizing the probability distribution and reformulate the subproblem (\ref{eq:subprob}) solved in Algorithm \ref{alg:da} as follows:
\begin{equation}
\min_{q \in \cP_2} \left\{ \bE_q\Big[ \sum_{s=1}^t sg^{(s)} \Big] + \sum_{s=1}^t s \lambda_2\bE_q[\log(q)] 
+ (t+1)\lambda_2\bE_q[\log(q)]  \right\}, \label{eq:subprob2}    
\end{equation} 
By comparing (\ref{eq:subprob-original-da}) and (\ref{eq:subprob2}), we arrive at the following correspondence: $\alpha_s = \beta_s = s,~f^{(s)} \sim g^{(s)},~d(w)=\Psi(w) \sim \lambda_2\bE_q[\log(q)]$.
We note that in our problem setting the expectation by $q$ can be seen as an inner product with the integrand and $\lambda_2\bE_q[\log(q)]$ is also set to $d(w)$ because \textit{the negative entropy acts as a strongly convex function} (Lemma \ref{lem:st-conv}).
\bigskip 
\section{Additional Experiments}
\label{app:additional_experiment}

\subsection{Comparison of Generalization Error} 

\begin{wrapfigure}{R}{0.34\textwidth}  
\centering 
\includegraphics[width=0.34\textwidth]{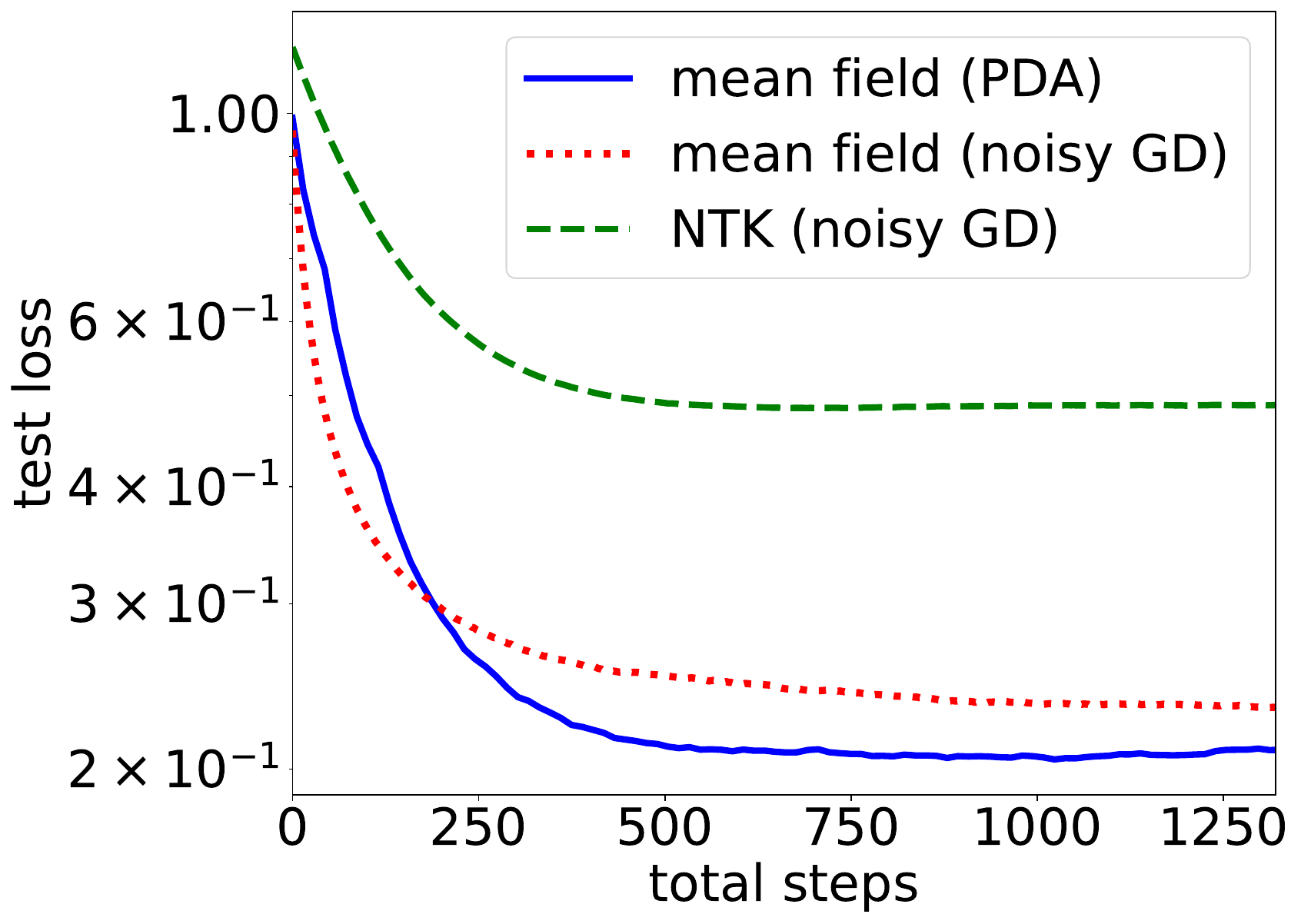}
\vspace{-5.1mm} 
\caption{\small Test error of mean field neural networks ($\alpha=1$) trained with noisy GD (red) and PDA (blue), and network in the kernel regime ($\alpha=1/2$) optimized by GD (green).}  
\label{fig:generalization} 
\vspace{-7.5mm} 
\end{wrapfigure} 

We provide additional experimental results on the generalization performance of PDA. 
We consider empirical risk minimization for a regression problem (squared loss): the input $x_i\sim\mathcal{N}(0,I_p)$, and $f_*$ is a single index model: $f_*(x) = \mathrm{sign}(\langle w_*,\vx\rangle)$.
W set $n=1000$, $p=50$, $M=200$, and implement both noisy gradient descent \citep{mei2018mean} using full-batch gradient and our proposed Algorithm~\ref{alg:pda} (PDA) using mini-batch update with batch size 50.

Figure~\ref{fig:generalization} we compare the generalization performance of different training methods: noisy GD and PDA in the mean field regime, and also noisy GD in the kernel regime. We fix the $\ell_2$ and entropy regularization to be the same across all settings: $\lambda_1 = 10^{-2}$, $\lambda_2 = 5\times 10^{-4}$. We set the \textit{total} number of iterations (outer + inner loop steps) in PDA to be the same as GD, and tuned the learning rate for optimal generalization. Observe that
\begin{itemize}[leftmargin=*,topsep=0.75mm,itemsep=0.5mm]
    \item Model with the NTK scaling (green) generalizes worse than the mean field models (red and blue). This is consistent with observations in \citet{chizat2018note}.
\end{itemize} 
\vspace{-1.5mm} 

\begin{itemize}[leftmargin=*,topsep=0.75mm,itemsep=0.5mm]
    \item For the mean field scaling, PDA (under early stopping) leads to slightly lower test error than noisy GD. We intend to further investigate this difference in the generalization performance. (see Appendix \ref{app:generalization_bounds} for generalization bounds of the PDA solution)
\end{itemize}

\subsection{PDA Beyond $\ell_2$ Regularization}
Note that our current formulation~\eqref{eq:regularization} considers $\ell_2$ regularization, which allows us to establish polynomial runtime guarantee for the inner loop via the Log-Sobolev inequality.
As remarked in Section~\ref{sec:convergence-analysis}, our global convergence analysis can easily be extended to H\"older-smooth gradient via the convergence rate of Langevin algorithm given in \citet{erdogdu2020convergence}. 
Although we do not provide details for this extension in the current work (due to the use of \citet{vempala2019rapid}), we empirically demonstrate one of its applications in handling $\ell_p$ regularized objectives for $p>1$ in the following form,
\begin{align}
R^p_{\lambda_1,\lambda_2}(q) \defeq \lambda_1 \bE_q[\|\theta\|_p^p] + \lambda_2 \bE_q[\log(q)].
\label{eq:regularizer-general}
\end{align}
\citet{erdogdu2020convergence} cannot directly cover the non-smooth $\ell_1$ regularization, but we can still obtain relatively sparse solution by setting $p$ close to 1. 
Intuitively speaking, when the underlying task exhibits certain low-dimensional or sparse structure, we expect a sparsity-promoting regularization to achieve better generalization performance. 

Figure \ref{fig:sparsity}(a) demonstrates the advantage of $L_p$-norm regularization for $p<2$ in empirical risk minimization, when the target function exhibits sparse structure. We set $n=1000, p=50$; the teacher is a multiple-index model ($m=2$) with binary activation, and parameters of each neuron are $1$-sparse. We optimize the student model with PDA (warm-start), where we set $\lambda_1=10^{-2}$, $\lambda_2=10^{-4}$, and vary the norm penalty $p$ from 1.01 to 2. 
Note that smaller $p$ results in favorable generalization due to the induced sparsity. 
On the other hand, we expect the benefit of sparse regularization to diminish when the target function is not sparse. This intuition is confirmed in \ref{fig:sparsity}(b), where we control the target sparsity by randomly selecting $r$ parameters to be non-zero, and we define $s=r/d$ to be the sparsity level. 
Observe that the benefit of sparsity-inducing regularization (smaller $p$) is more prominent under small $s$ (brighter color), which indicates a sparse target function. 

\begin{figure}[htb!]
\centering
\begin{minipage}[t]{0.4\linewidth}
\centering
{\includegraphics[width=1\textwidth]{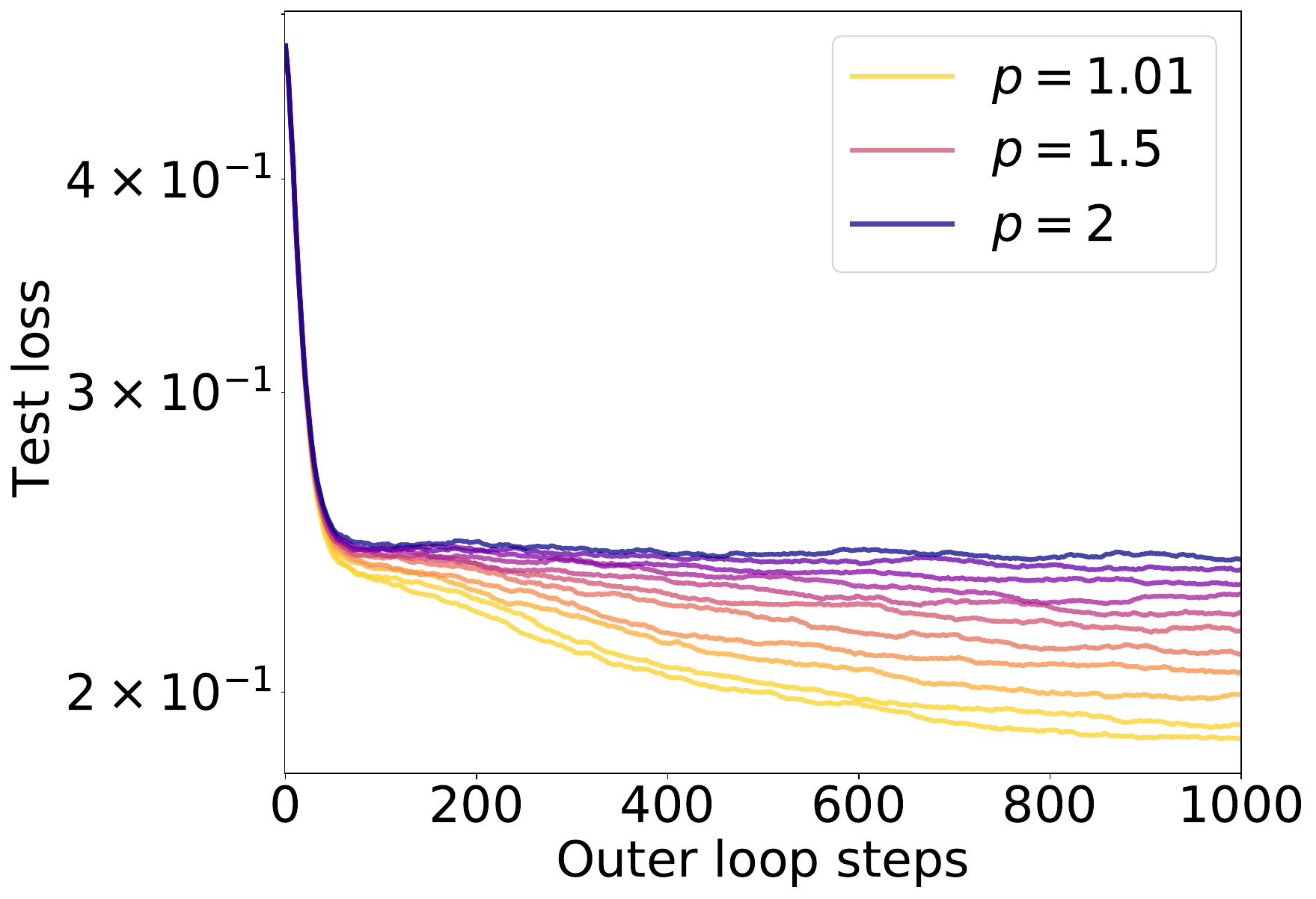}} \\ \vspace{-0.mm}
\small (a) Impact of $L_p$ regularization.
\end{minipage}
\begin{minipage}[t]{0.4\linewidth}
\centering
{\includegraphics[width=0.93\textwidth]{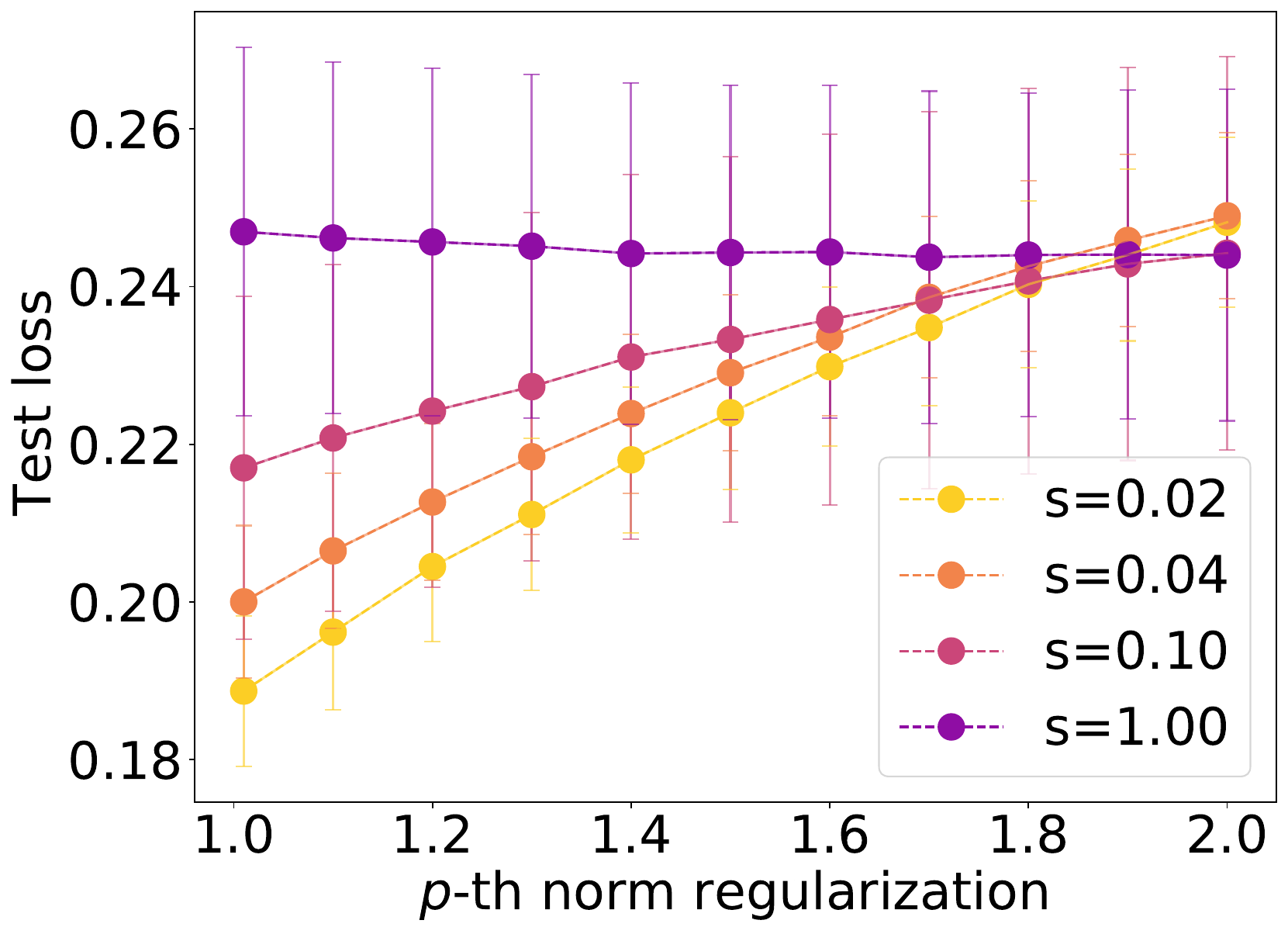}} \\ \vspace{-0.mm}
\small (b) Generalization under sparse teacher. 
\end{minipage} 
\caption{\small PDA with general $\ell_p$ regularizer (objective~\eqref{eq:regularizer-general}). (a) Generalization error vs.~training time in learning a 1-sparse target function. (b) generalization error vs.~sparsity of the target function $s$.} 
\label{fig:sparsity} 
\end{figure}

\subsection{On the Role of Entropy Regularization}

Our objective \eqref{eq:rerm} includes an entropy regularization with magnitude $\lambda_2$. In this section we illustrate the impact of this regularization term. In Figure~\ref{fig:entropy}(a) we consider a synthetic 1D dataset ($n=15$) and plot the output of a two-layer tanh network with 200 neurons trained by SGD and PDA to minimize the \textit{squared loss} till convergence. We use the same $\ell_2$ regularization ($\lambda_1=10^{-3}$) for both algorithms, and for PDA we set the entropic term $\lambda_2=10^{-4}$.  
Observe that SGD with weak regularization (red) almost interpolates the noisy training data, whereas PDA with entropy regularization finds low-complexity solution that is smoother (blue). 

We therefore expect entropy regularization to be beneficial when the labels are noisy and the underlying target function (teacher) is ``simple''. We verify this intuition in Figure~\ref{fig:entropy}(b). We set $n=500$, $d=50$ and $M=500$, and the teacher model is a linear function on the input features. We employ SGD or PDA to optimize the squared error. For both algorithms we use the same $\ell_2$ regularization $\lambda_1=10^{-2}$, 
but PDA includes a small entropy term $\lambda_2=5\times10^{-4}$. We plot the generalization error of the converged model under varying amount of label noise. Note that as the labels becomes more corrupted, PDA (blue) results in lower test error due to the entropy regularization\footnote{Note that entropy regularization is not the only way to reduce overfitting -- such capacity control can also be achieved by proper early stopping or other types of explicit regularization.}. On the other hand, model under the kernel scaling (green) generalizes poorly compared to the mean field models.
Furthermore, Figure~\ref{fig:entropy}(c) demonstrates that entropy regularization can be beneficial under low noise (or even noiseless) cases as well. We construct the teacher model to be a multiple-index model with binary activation. Note that in this setting PDA achieves lower stationary risk across all noise level, and the advantage amplifies as labels are further corrupted.
 
\begin{figure}[htb!]
\centering
\begin{minipage}[t]{0.325\linewidth}
\centering
{\includegraphics[width=0.92\textwidth]{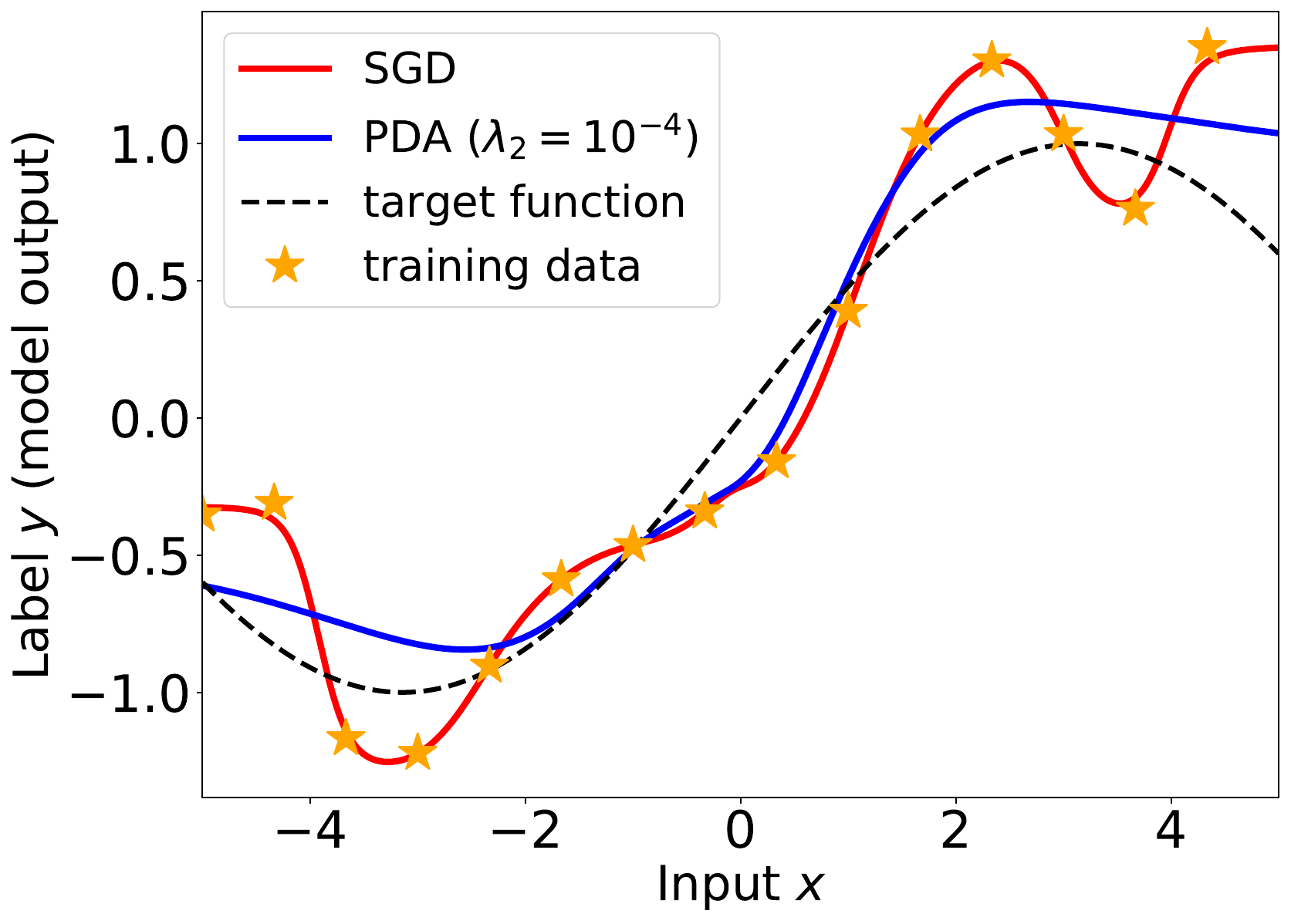}} \\ \vspace{-0.5mm}
\small (a) Impact of entropy regularization \\ (one-dimensional).
\end{minipage}
\begin{minipage}[t]{0.325\linewidth}
\centering
{\includegraphics[width=1.\textwidth]{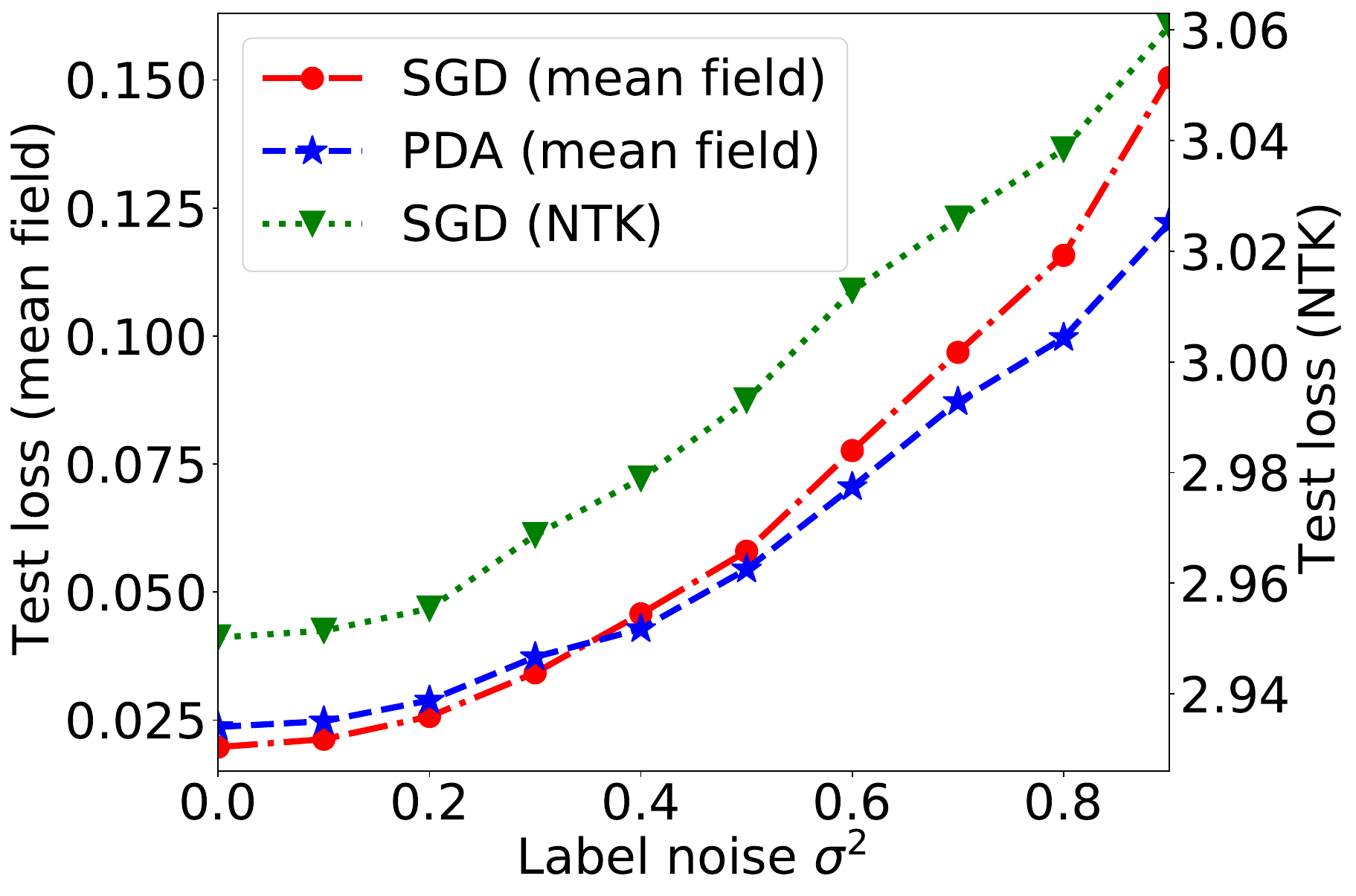}} \\ \vspace{-0.5mm}
\small (b) Stationary risk vs.~label noise \\
(linear teacher). 
\end{minipage} 
\begin{minipage}[t]{0.325\linewidth}
\centering
{\includegraphics[width=1.\textwidth]{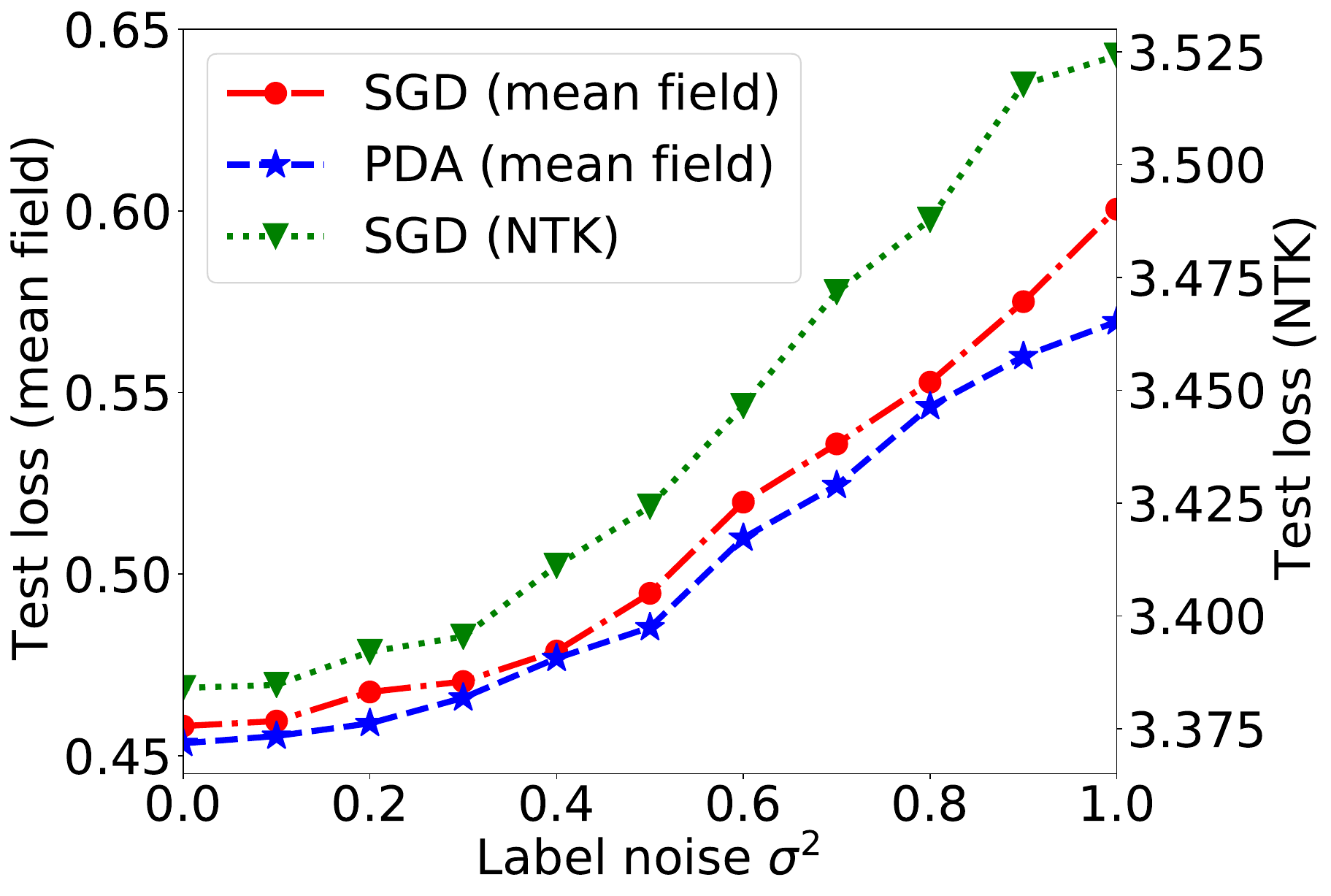}} \\ \vspace{-0.5mm}
\small (c) Stationary risk vs.~label noise \\
(multiple index teacher). 
\end{minipage} 
\vspace{-0.5mm} 
 \caption{\small (a) 1D illustration of the impact of entropy regularization in two-layer tanh network: PDA (blue) finds a smoother solution that does not interpolate the training data due to entropy regularization. 
 (b)(c) Test error of two-layer tanh network trained till convergence. PDA (blue) becomes advantageous compared to SGD (red) when labels become noisy, and the NTK model (green, note that the y-axis is on different scale) generalizes considerably worse than the mean field models.}
\label{fig:entropy} 
\end{figure}

\subsection{Adaptivity of Mean Field Neural Networks}
Recall that one motivation to study the mean field regime (instead of the kernel regime) is the presence of \textit{feature learning}. We illustrate this behavior in a simple student-teacher setup, where the target function is a single-index model with tanh activation. We set $n=500,d=50$, and optimize a two-layer tanh network ($M=1000$), either in the mean field regime using PDA, or in the kernel regime using SGD. For both methods we choose $\lambda_1=10^{-3}$, and for PDA we choose $\lambda_2=10^{-4}$. 

In Figure~\ref{fig:alignment} we plot the the evolution of the cosine similarity between the target vector $w^*$ and the top-5 singular vectors (PC1-5) of the weight matrix during training. In Figure~\ref{fig:alignment}(a) we observe that the mean field model trained with PDA ``adapts'' to the low-dimensional structure of the target function; in particular, the leading singular vector (bright yellow) aligns with the target direction. In contrast, we do not observe such alignment on the network in the kernel regime (Figure~\ref{fig:alignment}(b)), because the parameters do not travel away from the initialization. 
This comparison demonstrates the benefit of the mean field parameterization. 
 
\begin{figure}[htb!]
\centering
\begin{minipage}[t]{0.4\linewidth}
\centering
{\includegraphics[width=0.9\textwidth]{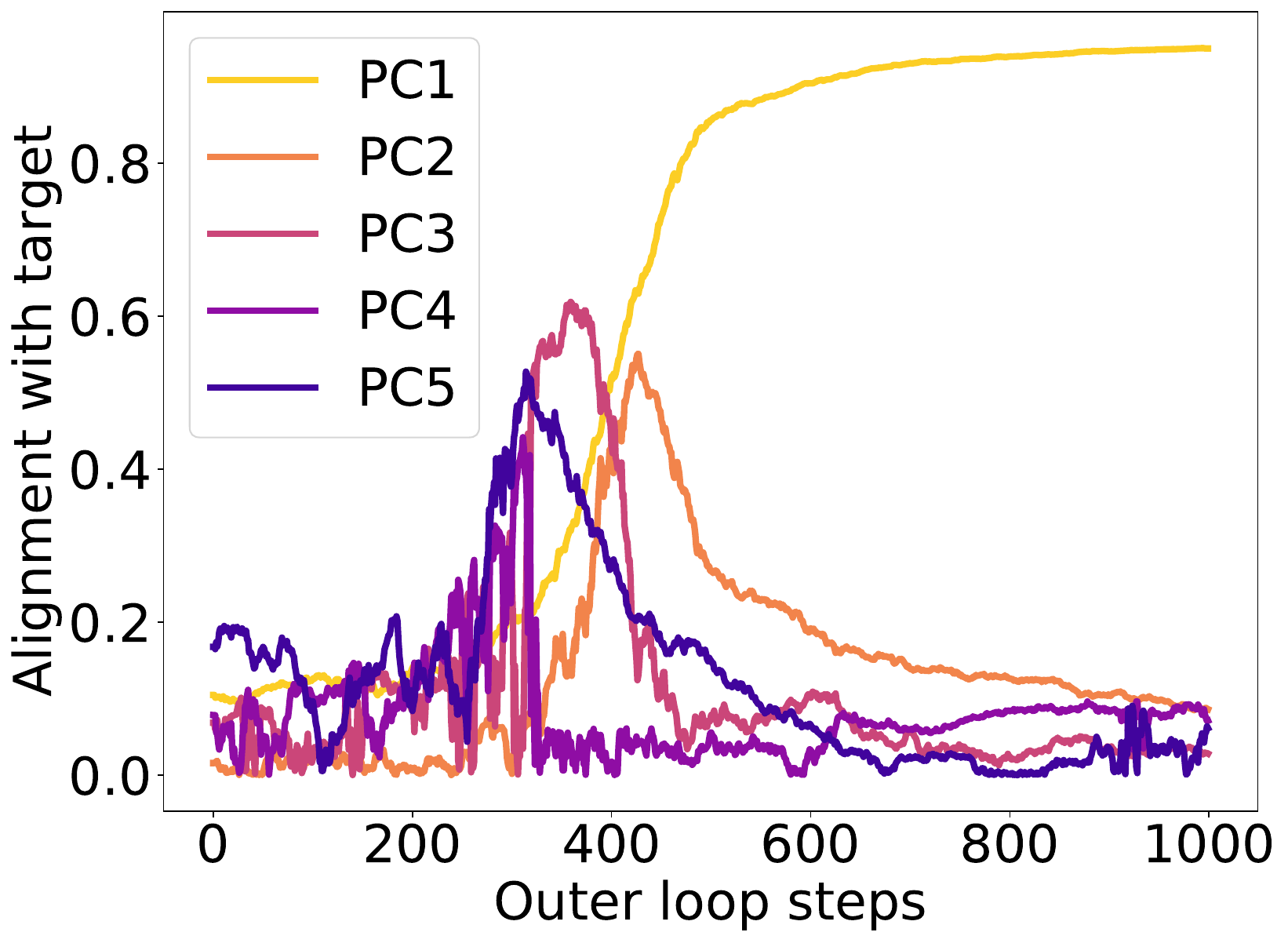}} \\ \vspace{-0.mm}
\small (a) Parameter Alignment (PDA).
\end{minipage}
\begin{minipage}[t]{0.4\linewidth}
\centering
{\includegraphics[width=0.945\textwidth]{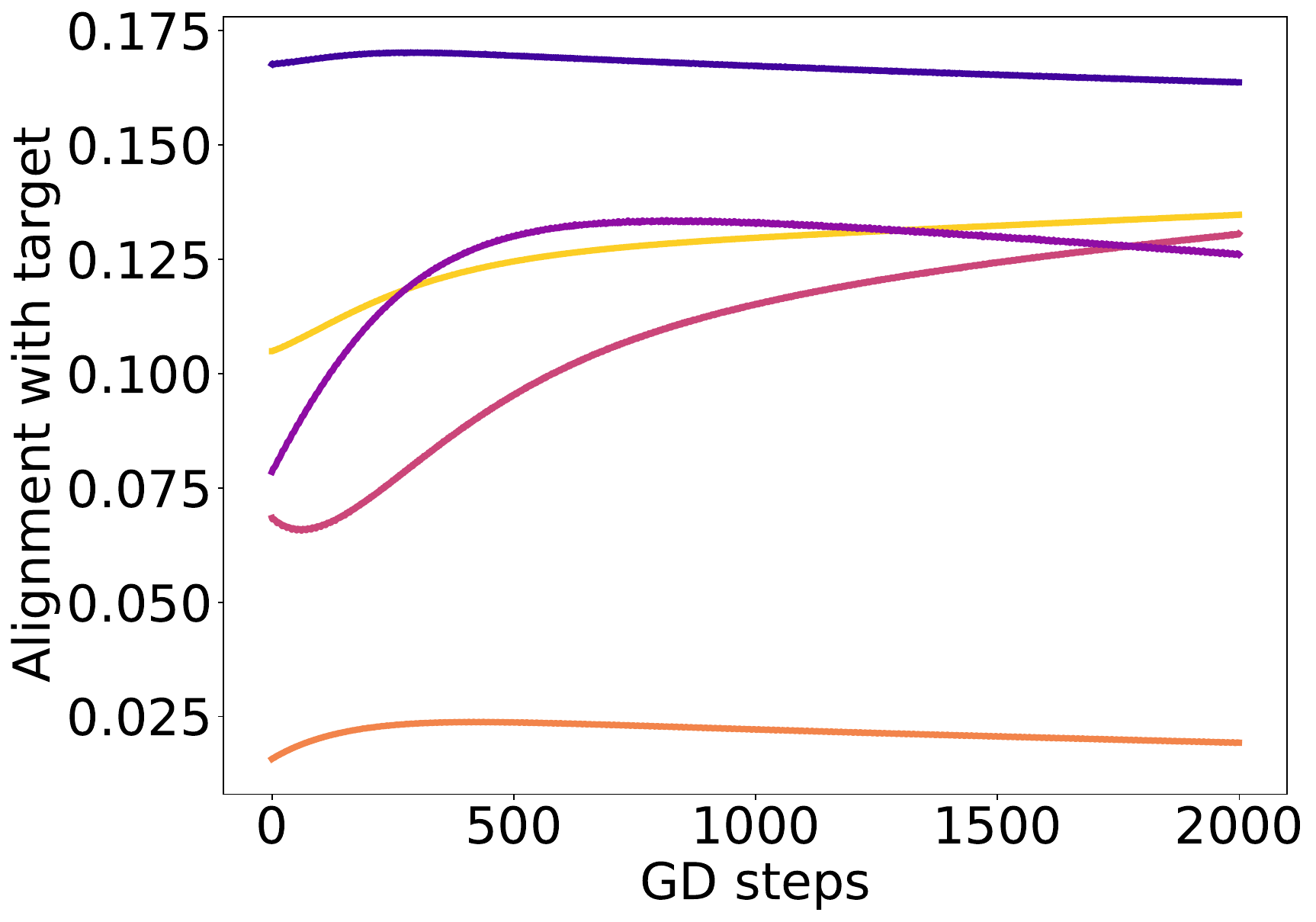}} \\ \vspace{-0.mm}
\small (b) Parameter Alignment (NTK). 
\end{minipage} 
\caption{\small Cosine similarity between the target vector $w^*$ and the top-5 singular vectors (PC1-5) of the weight matrix during training. The learned parameters ``align'' with the target function under the mean field parameterization (a), but not the NTK parameterization (b). } 
\label{fig:alignment} 
\vspace{2.5mm}
\end{figure} 

\bigskip

\section{Additional Related Work}
\paragraph{Particle inference algorithms.}
Bayesian inference is another example distribution optimization, in which the objective is to minimize an entropic regularized linear functional.
In addition to the Langevin algorithm, several interacting particle methods have been developed for this purpose, such as particle mirror descent (PMD) \citep{dai2016provable}, Stein variational gradient descent (SVGD) \citep{liu2016stein}, and ensemble Kalman sampler \citep{garbuno2020interacting}, and the corresponding mean field limits have been analyzed in \citet{lu2019scaling,ding2019ensemble}.
We remark that naive gradient-based method on the probability space often involves computing the probability of particles for the entropy term (e.g., kernel density estimation in PMD), which presents significant difficulty in constructing particle inference algorithms.
In contrast, our proposed algorithm avoids this computational challenge due to its algorithmic structure. 

\paragraph{Optimization of probability distributions.}
Parallel to our work, several recent papers also proposed optimization methods over space of probability measures by adapting finite-dimensional convex optimization theory. 
\cite{ying2020mirror}, \cite{kent2021frank} and \cite{chizat2021convergence} extend the Mirror descent method, Frank-Wolfe method, and (accelerated) Bregman proximal gradient method to the optimization of probability measures, respectively.
In addition, \cite{hsieh2019finding} developed an {entropic mirror descent} algorithm for generative adversarial networks, and \cite{chu2019probability} analyzed {probability functional descent} in the context of variational inference and reinforcement learning.  

\paragraph{The kernel regime and beyond.}
The neural tangent kernel model \citep{jacot2018neural} describes the learning dynamics of neural network under appropriate scaling.
Such description builds upon the linearization of the learning dynamics around its initialization, and (quantitative) global convergence guarantees of gradient-based methods for neural networks can be shown for regression problems \citep{du2018gradient,allen2019convergence,zou2020gradient,nitanda2020optimal} as well as classification problems \citep{cao2019generalization,nitanda2019gradient,ji2019polylogarithmic}.

However, due to the linearization, the NTK model cannot explain the presence of ``feature learning'' in neural networks (i.e.~parameters are able to travel and adapt to the structure of the learning problem).
In fact, various works have shown that deep learning is more powerful than kernel methods in terms of approximation and estimation error \citep{suzuki2018adaptivity,ghorbani2019linearized,suzuki2019deep,schmidt2020nonparametric,ghorbani2020neural,imaizumi2020advantage}, and in certain settings, neural networks optimized with gradient-based methods can outperform the NTK model (or more generally any kernel methods) in terms of generalization error or excess risk \citep{allen2019can,ghorbani2019limitations,yehudai2019power,bai2019beyond,allen2020backward,li2020learning,suzuki2020generalization,daniely2020learning}.

\end{document}